\documentclass{article}

\PassOptionsToPackage{numbers, compress}{natbib}

\usepackage[final]{neurips2025}

\usepackage[utf8]{inputenc} 
\usepackage[T1]{fontenc}    
\usepackage{hyperref}       
\usepackage{url}            
\usepackage{booktabs}       
\usepackage{amsfonts}       
\usepackage{nicefrac}       
\usepackage{microtype}      
\usepackage{amsmath,amssymb,amsthm}
\usepackage{xcolor}
\usepackage{algorithm}
\usepackage{algpseudocode}
\usepackage[noabbrev,capitalise,nameinlink]{cleveref}
\usepackage{enumerate,subfigure,multirow,epsfig,psfrag,hhline,color,srcltx,subeqnarray}

\newtheorem{assumption}{Assumption}

\newtheorem{lemma}{Lemma}

\newtheorem{theorem}{Theorem}

\newtheorem{definition}{Definition}

\DeclareMathOperator*{\argmax}{arg\,max}

\allowdisplaybreaks

\title{Deep Q-Learning with Gradient Target Tracking}

\author{%
  Bumgeun Park\thanks{Equal contribution.} \\
  Electrical Engineering, KAIST\\
  \texttt{j4t123@kaist.ac.kr}
    \And
  Taeho Lee \footnotemark[1]\\
  Electrical Engineering, KAIST\\
  \texttt{eho0228@kaist.ac.kr}
  \And
Donghwan Lee \\
  Electrical Engineering, KAIST\\
  \texttt{donghwan@kaist.ac.kr}
}

\begin{document}
\setlength{\abovedisplayskip}{3pt}
\setlength{\belowdisplayskip}{3pt}

\maketitle

\begin{abstract}
This paper introduces Q-learning with gradient target tracking, a novel reinforcement learning framework that provides a learned continuous target update mechanism as an alternative to the conventional hard update paradigm. In the standard deep Q-network (DQN), the target network is a copy of the online network's weights, held fixed for a number of iterations before being periodically replaced via a hard update. While this stabilizes training by providing consistent targets, it introduces a new challenge: the hard update period must be carefully tuned to achieve optimal performance. To address this issue, we propose two gradient-based target update methods: DQN with asymmetric gradient target tracking (AGT2-DQN) and DQN with symmetric gradient target tracking (SGT2-DQN). These methods replace the conventional hard target updates with continuous and structured updates using gradient descent, which effectively eliminates the need for manual tuning. We provide a theoretical analysis proving the convergence of these methods in tabular settings. Additionally, empirical evaluations demonstrate their advantages over standard DQN baselines, which suggest that gradient-based target updates can serve as an effective alternative to conventional target update mechanisms in Q-learning.
\end{abstract}

\section{Introduction}
Recently, reinforcement learning (RL)~\cite{sutton1998reinforcement} has seen remarkable success in solving sequential decision-making problems~\cite{mnih2015human,lillicrap2015continuous,silver2016mastering,schulman2017proximal,fujimoto2018addressing,haarnoja2018soft,van2016deep,wang2016dueling,anschel2017averaged}. In particular, deep Q-network (DQN) algorithm~\cite{mnih2015human} is a widely used approach for value-based RL, and has demonstrated human-level performance in various complex tasks, such as playing Atari games. However, standard DQN suffers from tuning challenges due to the periodic hard update of the target network. Specifically, the target Q-network is updated only periodically via a hard update, where every few thousand iterations, the target network’s weights are replaced by those of the online network. This strategy stabilizes learning but introduces a new challenge: the hard update period must be carefully tuned to achieve optimal performance. An alternative to hard updates is incremental updates, where the target network is updated gradually at each step. Lillicrap et al.~\cite{lillicrap2015continuous} popularized this approach in deep RL, using a soft update (or Polyak averaging), where the target parameters move toward the online parameters with a small learning rate. While this method mitigates instability, it still requires careful tuning of the averaging weight.

To address these challenges, we introduce a novel gradient-based target update framework for Q-learning. Our approach replaces the standard hard target update with gradient-based target learning approach, which continuously updates the target Q-function using gradient descent. In particular, we develop DQN with asymmetric gradient target tracking (AGT2-DQN) and DQN with symmetric gradient target tracking (SGT2-DQN). AGT2-DQN replaces the hard update with a continuous, gradient-based update for the target network. Instead of copying weights infrequently, the target network's parameters are treated as learnable, and they are adjusted at every step using gradient descent on a mean squared loss between the target Q-network and online Q-network. SGT2-DQN also uses continuous gradient-based target updates, but with a symmetric coupling between the online and target networks. Both networks are updated with loss terms that drive them toward each other's predictions. Morever, SGT2-DQN adds an extra regularization term to the loss functions of both networks that penalizes the difference between the online Q-value and target Q-value for the same state-action. This means the online and target networks influence each other's updates, which enforces a more balanced interaction (each network ``knows'' about the other's objective).

Unlike standard DQN, where the target network is updated only periodically, our methods ensure continuous adaptation of the target function through gradient-based learning. The advantages of our approach are supported by experiments on benchmark reinforcement learning tasks, which show that AGT2-DQN and SGT2-DQN achieve learning performance comparable to standard DQN without significant tuning steps.

Finally, we provide theoretical analysis, which shows that both the asymmetric and symmetric algorithms converge to the optimal Q-values in the tabular setting. In fact, for the tabular versions of AGT2-DQN/SGT2-DQN, called Q-learning with asymmetric gradient tracking (AGT2-QL) and Q-learning with symmetric gradient tracking (SGT2-QL), we prove that under standard assumptions, the online and target Q estimates will converge to $Q^*$ (the true optimal Q-function) with probability one.

\section{Related works}

In the original DQN~\cite{mnih2015human}, the target network is a copy of the online network's weights that is held fixed for a number of iterations and then replaced (``hard updated'') periodically. Researchers have looked into several alternatives to this hard update mechanism: soft (Polyak) updates have been introduced in~\cite{lillicrap2015continuous}. Instead of copying the network every few thousand steps, an incremental update can be done every step that updates the target parameters with a small learning rate towards the online parameters. In~\cite{lee2019target,lee2020unified}, the effect of the soft updates has been theoretically investigated for tabular TD-learning and tabular Q-learning respectively.

DeepMellow~\cite{kim2019deepmellow} removes the target network entirely by modifying the Bellman update and achieves faster learning without the target network. They replace the conventional $\max$ operator with a Mellowmax operator, a smooth approximation of $\max$ that prevents over-optimistic value estimates. Similarly, in~\cite{piche2021beyond}, the authors propose adding a functional regularization (FR) therm to the DQN loss and removing the target network, which is calle FR-DQN. In their FR-DQN algorithm, the Q-network is trained with an extra loss term that measures the difference between the current Q-values and those of a prior network (which is a periodically updated snapshot of the online network). This acts similarly to a target network, but instead of using the prior network's values as fixed targets, it uses them as a reference to restrain the current network's updates.

In summary, several strategies have been explored as alternatives to hard target updates. All these methods share the goal of providing stability to Q-learning without the downsides of an abruptly updated target. The new gradient target tracking approaches proposed in this paper contribute to this landscape by offering a learned, continuous target update mechanism as yet another alternative to the hard update paradigm.

\section{Preliminaries}

\subsection{Markov decision problem}
We consider the infinite-horizon discounted Markov decision problem and Markov decision process, where the agent sequentially takes actions to maximize cumulative discounted rewards. In a Markov decision process with the state-space ${\cal S}:=\{ 1,2,\ldots ,|{\cal S}|\}$ and action-space ${\cal A}:= \{1,2,\ldots,|{\cal A}|\}$, the decision maker selects an action $a \in {\cal A}$ at the current state $s\in {\cal S}$, then the state
transits to the next state $s'\in {\cal S}$ with probability $P(s'|s,a)$, and the transition incurs a
reward $r(s,a,s') \in {\mathbb R}$, where $P(s'|s,a)$ is the state transition probability from the current state
$s\in {\cal S}$ to the next state $s' \in {\cal S}$ under action $a \in {\cal A}$, and $r(s,a,s')$ is the reward function. Moreover, $|{\cal S}|$ and  $|{\cal A}|$ denotes cardinalities of $\cal S$ and $\cal A$, respectively. For convenience, we consider a deterministic reward function and simply write $r(s_k,a_k ,s_{k + 1}) =:r_{k+1},k \in \{ 0,1,\ldots \}$. A policy $\pi$ maps a state $s \in {\cal S}$ to an action distribution $\pi: {\cal S} \to \Delta_{\cal A}$, where $\Delta_{\cal A}$ denotes the set of all probability distributions over the action space ${\cal A}$. The objective of the Markov decision problem is to find an optimal policy, $\pi^*$, such that the cumulative discounted rewards over infinite time horizons is maximized, i.e.,
\begin{align*}
\pi^*:= \argmax_{\pi\in \Theta} {\mathbb E}\left[\left.\sum_{k=0}^\infty {\gamma^k r_{k+1}}\right|\pi\right],
\end{align*}
where $\gamma \in [0,1)$ is the discount factor, $\Theta$ is the set of all policies, $(s_0,a_0,s_1,a_1,\ldots)$ is a state-action trajectory generated by the Markov chain under policy $\pi$, and ${\mathbb E}[\cdot|\pi]$ is an expectation conditioned on the policy $\pi$. Moreover, Q-function under policy $\pi$ is defined as
\begin{align*}
Q^{\pi}(s,a)={\mathbb E}\left[ \left. \sum_{k=0}^\infty {\gamma^k r_{k+1}} \right|s_0=s,a_0=a,\pi \right],\quad (s,a)\in {\cal S} \times {\cal A},
\end{align*}
and the optimal Q-function is defined as $Q^*(s,a)=Q^{\pi^*}(s,a)$ for all $(s,a)\in {\cal S} \times {\cal A}$. Once $Q^*$ is known, then an optimal policy can be retrieved by the greedy policy $\pi^*(s)=\argmax_{a\in {\cal A}}Q^*(s,a)$. Throughout, we assume that the Markov decision process is ergodic so that the stationary state distribution exists.

\section{Deep Q-learning with gradient target tracking}

Before introducing the main results, we first briefly review the standard DQN in~\cite{mnih2015human}. DQN considers the two optimal Q-function estimates, $Q_\theta$ called the online Q-network, and $Q_{\theta'}$ called the target Q-network. Here, $\theta$ and $\theta'$ are called the online and target parameters, respectively.
To update these parameters, the following loss function is considered:
\[L(\theta ;B): = \frac{1}{2}\frac{1}{{|B|}}\sum\limits_{(s,a,r,s') \in B} {{{(y - {Q_\theta }(s,a))}^2}} \]
where $B$ is the mini-batch uniformly sampled from the replay buffer $D$, $|B|$ denotes the size of the mini-batch, and $y$, fixed as a constant, is called the target, which is defined as
\[y = r + {\bf 1}(s')\gamma {\max _{a \in A}}{Q_{\theta '}}(s',a)\]
Here, ${\bf 1}(s')$ is an indicator function defined as
\[{\bf{1}}(s): = \left\{ {\begin{array}{*{20}{c}}
0&{{\rm{if}}\,\,s = {\rm{terminal}}\,\,{\rm{state}}}\\
1&{{\rm{else}}}
\end{array}} \right.\]
The online Q-network $Q_\theta$ is then updated through the gradient descent step
\[\theta  \leftarrow \theta  - \alpha \nabla L(\theta ;B)\]
where the target parameter $\theta'$ is fixed as a constant.
Then, the target Q-network $Q_{\theta'}$ is updated periodically by the online parameter with period $C>0$
\[\theta'  \leftarrow \theta\]
\cref{fig:DQN-C} shows the learning curves of DQN for different $C\in \{1,10,100,200,500\}$ in Cartpole environment (Open AI gym). The results show that $C=10$ gives the best learning performance among the other choices of $C$.
\begin{figure}[h!]
\centering
\includegraphics[width=6cm,height=5cm]{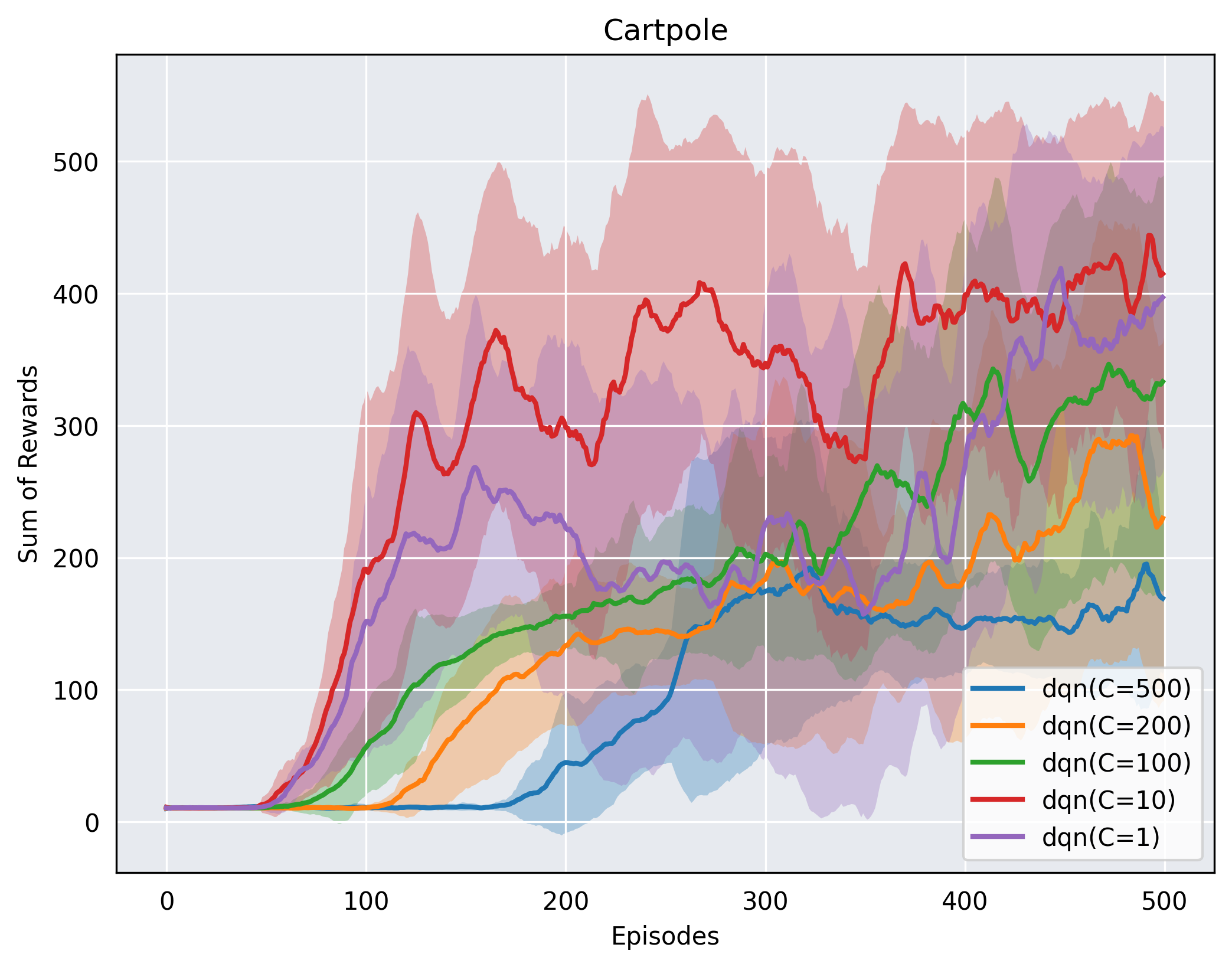}
\caption{Comparison of reward curves of DQN with different $C$ in Cartpole environment. }\label{fig:DQN-C}
\end{figure}
As the results indicate, achieving better learning performance requires careful selection of $C$, which demands extensive trials and time.

In this paper, we propose two gradient-based target learning methods for DQN: the first algorithm is called DQN with asymmetric gradient target tracking (AGT2-DQN), and the second is DQN with symmetric gradient target tracking (SGT2-DQN). In AGT2-DQN, we employ two distinct Q-networks, ${Q_{{\theta_1}}}$ and ${Q_{{\theta _2}}}$, referred to as the online and target Q-networks, respectively. The parameters $\theta_1$ and $\theta_2$ are called the online and target parameters, respectively. The corresponding loss functions are defined as
\[L_1(\theta _1;B): = \frac{1}{2}\frac{1}{|B|}\sum\limits_{(s,a,r,s') \in B} {{{(y_1 - Q_{\theta _1}(s,a))}^2}} \]
and
\[L_2(\theta _2;B): = \frac{\beta}{2}\frac{1}{{|B|}}\sum\limits_{(s,a,r,s') \in B} {{{({y_2} - {Q_{{\theta _2}}}(s,a))}^2}} \]
where $\beta>0$ is a constant weight, $y_1$ and $y_2$ are targets defined as
\[y_1 = r + {\bf 1}(s')\gamma {\max _{a \in {\cal A}}}{Q_{{\theta _2}}}(s',a),\quad {y_2} = {Q_{{\theta _1}}}(s,a) \]

Both $\theta_1$ and $\theta_2$ are updated through the gradient descent steps
\[{\theta _1} \leftarrow {\theta _1} - \alpha {\nabla _{{\theta _1}}}{L_1}({\theta _1};B),\quad {\theta _2} \leftarrow {\theta _2} - \alpha {\nabla _{{\theta _2}}}{L_2}({\theta _2};B),\]
where $\alpha >0$ is the step-size. The overall algorithm is summarized in Appendix~A.
Note that the online parameter update is identical to that in the standard DQN. The primary distinction arises from the target parameter update, which employs the following gradient descent step for the loss function $L_2$:
\[{\theta _2} \leftarrow {\theta _2} - \alpha {\nabla _{{\theta _2}}}{L_2}({\theta _2};B) = {\theta _2} + \alpha \frac{1}{{|B|}}\sum\limits_{(s,a,r,s') \in B} {({Q_{{\theta _1}}}(s,a) - {Q_{{\theta _2}}}(s,a)){\nabla _{{\theta _2}}}{Q_{{\theta _2}}}(s,a)} \]
The above update treats the target network in DQN as a learnable parameter updated via gradient descent, rather than updating it only periodically~\cite{mnih2015human} or by polyak averaging~\cite{lillicrap2015continuous}. This yields a form of continuous update where the target slowly learns to predict the online Q-values (or some combination thereof), instead of being replaced outright. 
\begin{figure}[h!]
\centering
\subfigure[AGT2-DQN]{\includegraphics[width=6cm,height=5cm]{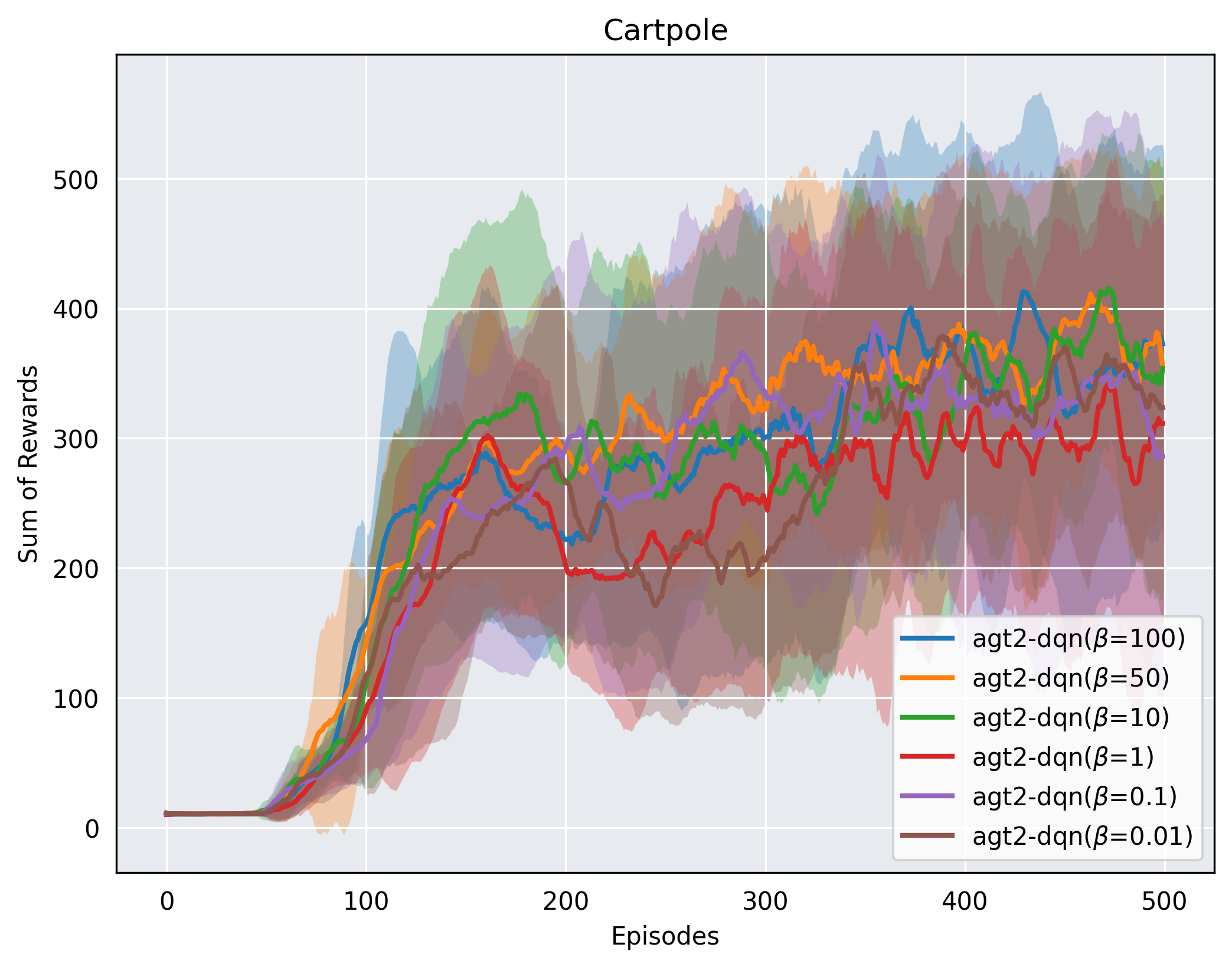}}
\subfigure[SGT2-DQN]{\includegraphics[width=6cm,height=5cm]{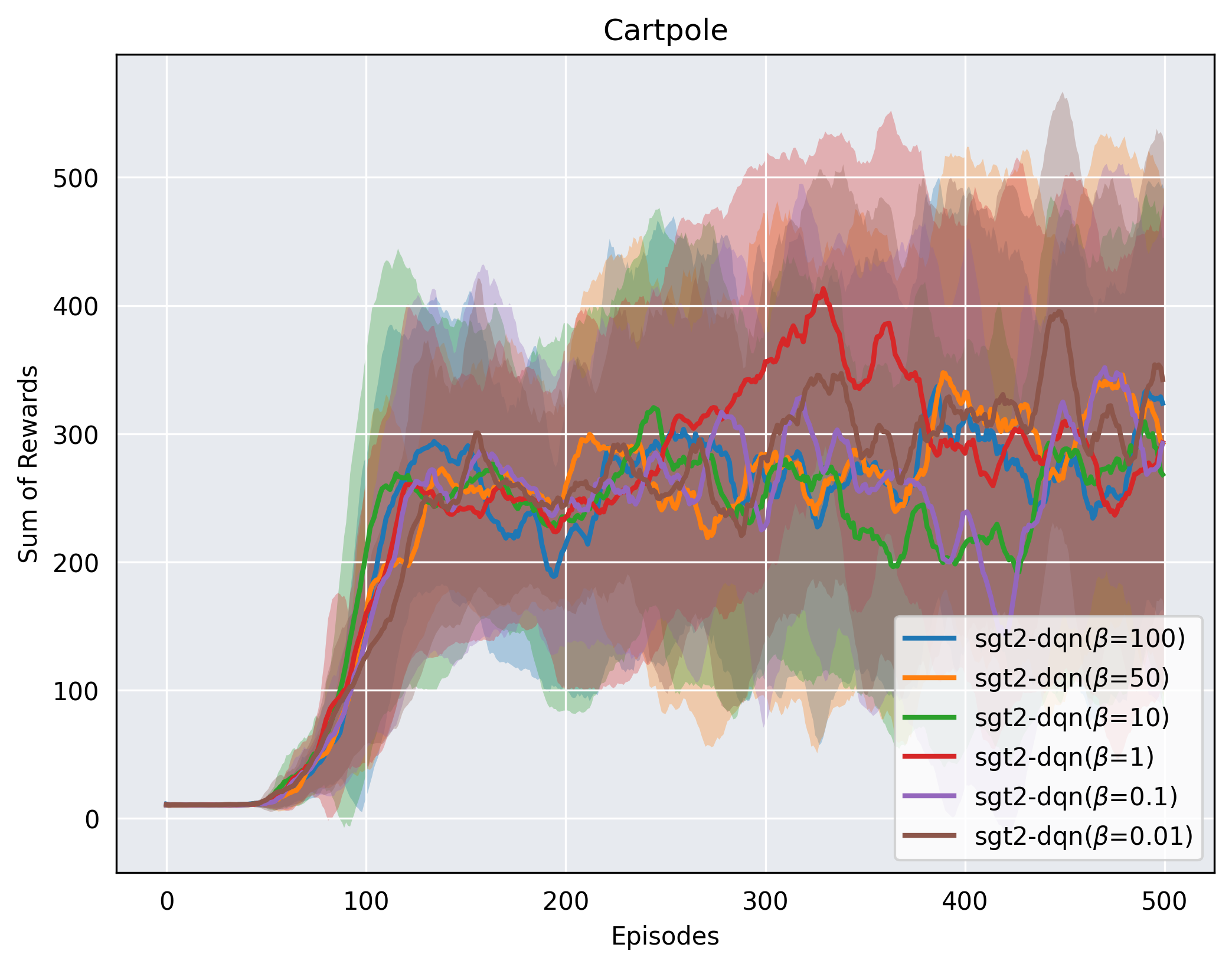}}
\caption{Comparison of reward curves of AGT2-DQN and SGT2-DQN with different $\beta$. Cartpole environment in OpenAI Gym is used here. }\label{fig:AGT2-DQN-beta}
\end{figure}

\begin{figure}[h!]
\centering
\includegraphics[width=6cm,height=5cm]{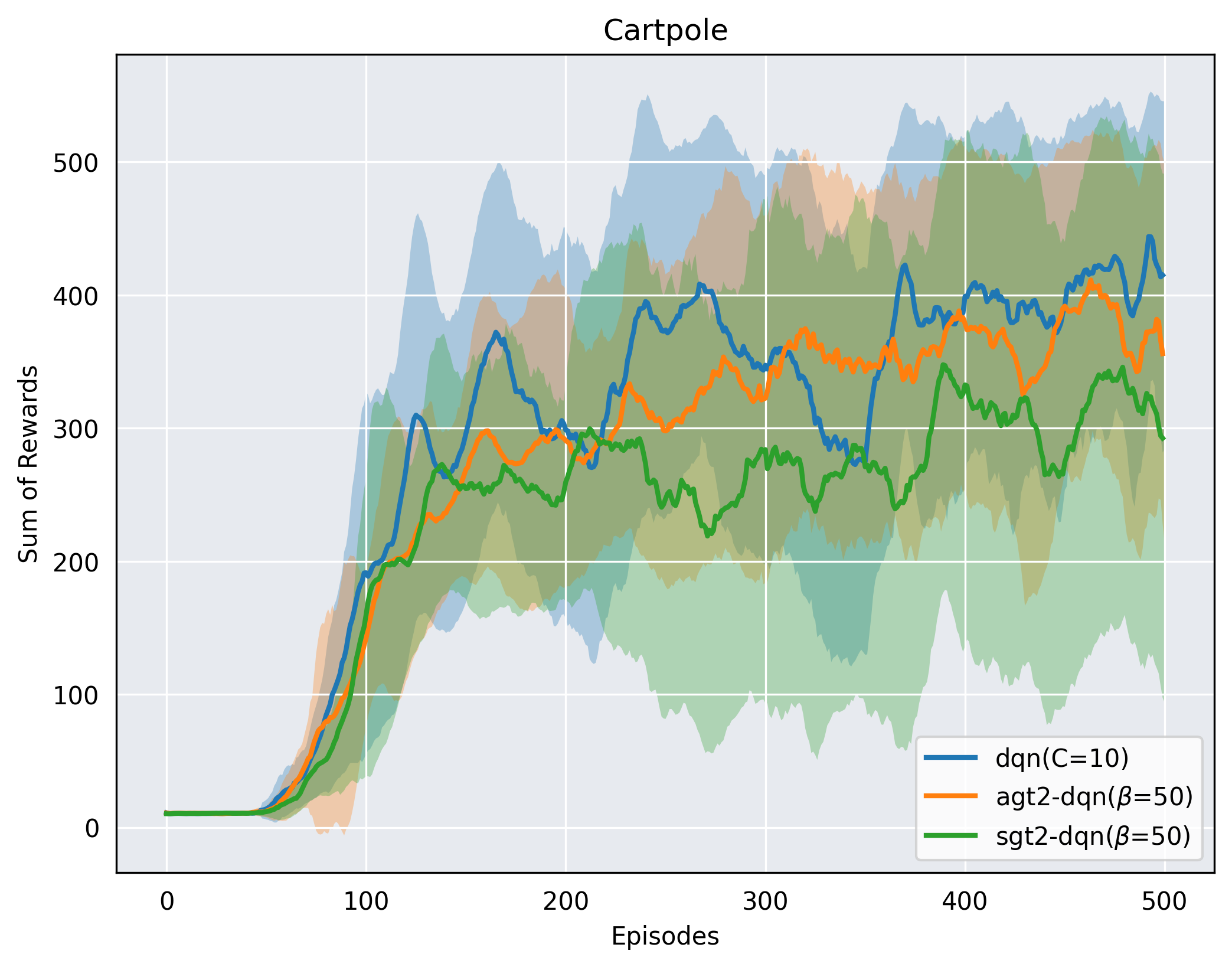}
\caption{Comparison of reward curves of DQN with $C=10$, AGT2-DQN and SGT2-DQN with $\beta = 50$. Cartpole environment in OpenAI Gym is used here. }\label{fig:DQN-AGT2-SGT2}
\end{figure}

As noted in the previous section, the updates for the online and target parameters are asymmetric. This asymmetry might lead to inefficiencies in finding the optimal solution since each update plays a distinct role and does not consider the influence of the other.

The algorithm can be conceptualized as a two-player game, where the first player (online parameter update) strives to solve the Bellman equation, while the second player (target parameter update) aims to track the first player's actions. In this setup, the utility functions of the two players are asymmetric and do not consider the objectives of each other. This mismatch can lead to instability and inefficiencies during the learning process. A potential alternative would be to design a game where each player’s utility function explicitly incorporates the other player's goal, leading to a more balanced interaction. In such a framework, the utility functions would be symmetric, which could potentially enhance the stability and efficiency of the algorithm.

Motivated by the previous discussions, we propose the second algorithm, referred to as SGT2-DQN. Similar to AGT2-DQN, this algorithm also employs two Q-networks ${Q_{{\theta_1}}}$ and ${Q_{{\theta _2}}}$, which are called the online and target Q-networks, respectively. However, SGT2-DQN introduces a symmetric approach to updating these networks, that aims to harmonize the interaction between them. In particular, SGT2-DQN employs the following loss functions:
\[{L_1}({\theta _1};B): = \frac{1}{2}\frac{1}{{|B|}}\sum\limits_{(s,a,r,s') \in B} {[{{(y_1 - Q_{{\theta _1}}(s,a))}^2} + \beta {{({Q_{{\theta _2}}}(s,a) - {Q_{{\theta _1}}}(s,a))}^2}]} \]
and
\[{L_2}({\theta _2};B): = \frac{1}{2}\frac{1}{{|B|}}\sum\limits_{(s,a,r,s') \in B} {[{{(y_2 - Q_{{\theta _2}}(s,a))}^2} + \beta {{({Q_{{\theta _1}}}(s,a) - {Q_{{\theta _2}}}(s,a))}^2}]}. \]
where $\beta >0 $ is a constant weight representing the regularization weight, and the targets $y_1$ and $y_2$ are defined as
\begin{align*}
y_1 = r + {\bf 1}(s')\gamma \max _{a \in {\cal A}}Q_{\theta _2}(s',a),\quad 
y_2 = r + {\bf 1}(s')\gamma {\max _{a \in {\cal A}}}Q_{\theta _1}(s',a).
\end{align*}

Next, both $\theta_1$ and $\theta_2$ are updated through the gradient descent steps
\[{\theta _1} \leftarrow {\theta _1} - \alpha {\nabla _{{\theta _1}}}{L_1}({\theta _1};B),\quad {\theta _2} \leftarrow {\theta _2} - \alpha {\nabla _{{\theta _2}}}{L_2}({\theta _2};B),\]
where $\alpha>0$ is the step-size. The overall algorithm is summarized in Appendix~B. In SGT2-DQN, both the online and target networks interact and adjust towards each other using gradient steps. Moreover, this algorithm echoes the idea of double Q-learning in a gradient setting. Double Q-learning maintains two estimates that learn from each other's predictions~\cite{hasselt2010double}.
SGT2-DQN similarly lets two function approximators co-evolve and correct one another. 

\cref{fig:AGT2-DQN-beta} illustrates the learning curves of AGT2-DQN and SGT2-DQN for different $\beta\in \{0.01,0.1,1,10,50,100\}$ in Cartpole environment. The results demonstrate that their learning efficiency is comparable to DQN. However, one can observe that they are less sensitive to the hyperparameter $\beta$.

\cref{fig:DQN-AGT2-SGT2} illustrates the learning curves of DQN with $C=10$, AGT2-DQN and SGT2-DQN with $\beta = 50$, where the hyperparameter for each method has been selected to achieve approximately the best performance among the tested grid points.

The results show that while DQN exhibits slightly better learning efficiency than the proposed methods in this environment. However, as mentioned before, the latter require less efforts for hyperparameter tuning in this case.

We conducted similar experiments across multiple environments, and the results varied depending on the environment. For instance, comparisons of learning curves are presented in~\cref{fig:comparison1}, where the hyperparameters were roughly tuned to achieve optimal learning performance for each method.

As shown in~\cref{fig:comparison1}(a), DQN with $C=50$, AGT2-DQN with $\beta = 0.1$, and SGT2-DQN with $\beta =1$ demonstrate comparable performance in the Acrobot environment. In contrast,~\cref{fig:comparison1}(b) illustrates that AGT2-DQN with $\beta = 0.1$ and SGT2-DQN with $\beta = 0.1$ exhibit better learning efficiency than DQN with $C=500$. A full comparison of results, including those with different hyperparameter settings, is provided in the Appendix~J. Moreover, Appendix~J also includes comparisons in other environments, which exhibit similar trends. 

Overall, DQN, AGT2-DQN, and SGT2-DQN exhibited comparable learning performance on average. In some environments, DQN outperformed the other two methods, while in others, AGT2-DQN and SGT2-DQN achieved better learning performance than DQN. However, we observed that AGT2-DQN and SGT2-DQN were slightly less sensitive to hyperparameters than DQN on average. Therefore, it is difficult to conclude that the proposed AGT2-DQN and SGT2-DQN outperform DQN. However, we argue that they offer interesting alternatives to the hard update paradigm in DQN.

\begin{figure}[h!]
\centering
\subfigure[Acrobot: DQN with $C=50$, AGT2-DQN with $\beta = 0.1$, and SGT2-DQN with $\beta = 1$.]{\includegraphics[width=6cm,height=5cm]{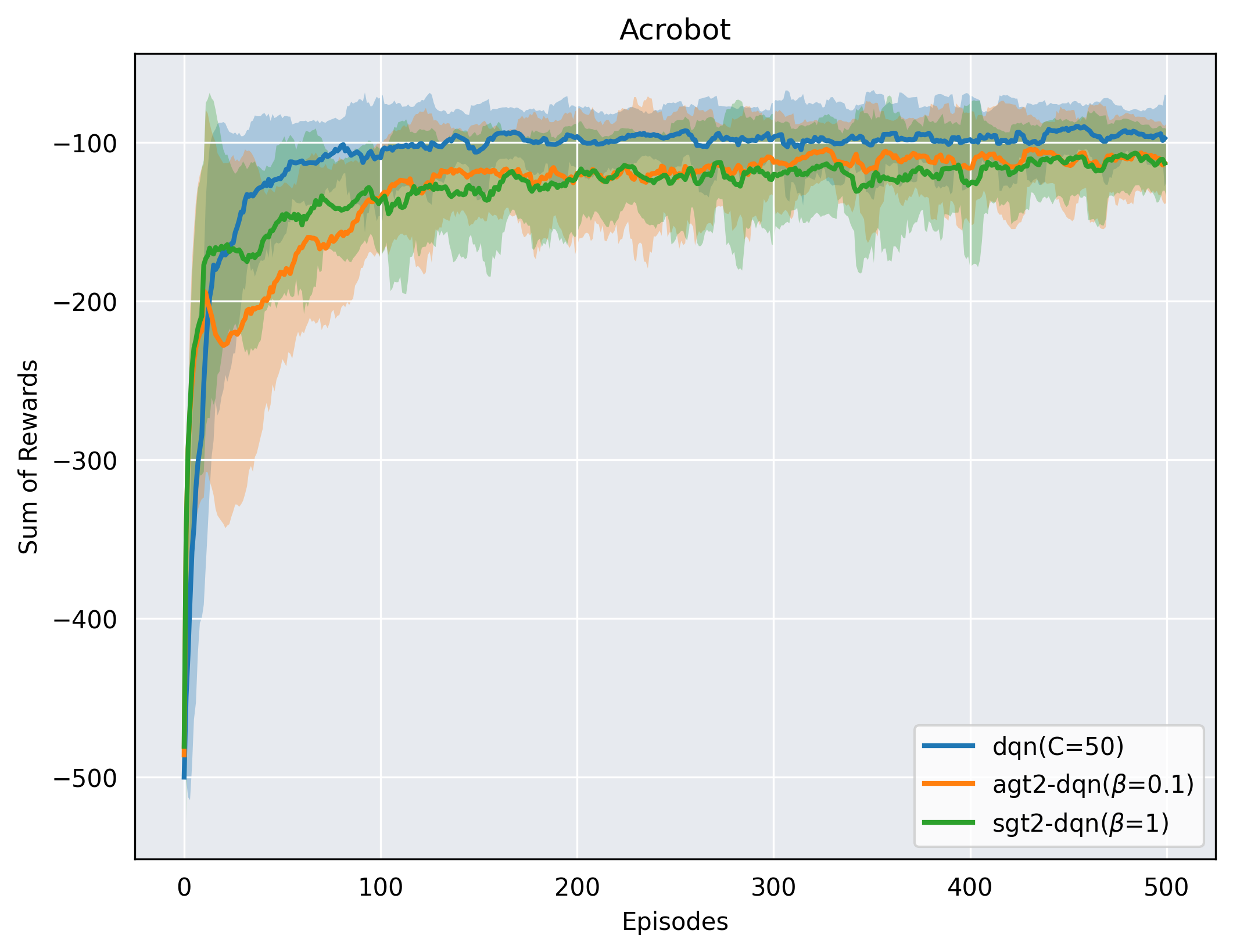}}
\subfigure[Pendulum: DQN with $C=500$, AGT2-DQN with $\beta = 0.1$, and SGT2-DQN with $\beta = 0.1$.]{\includegraphics[width=6cm,height=5cm]{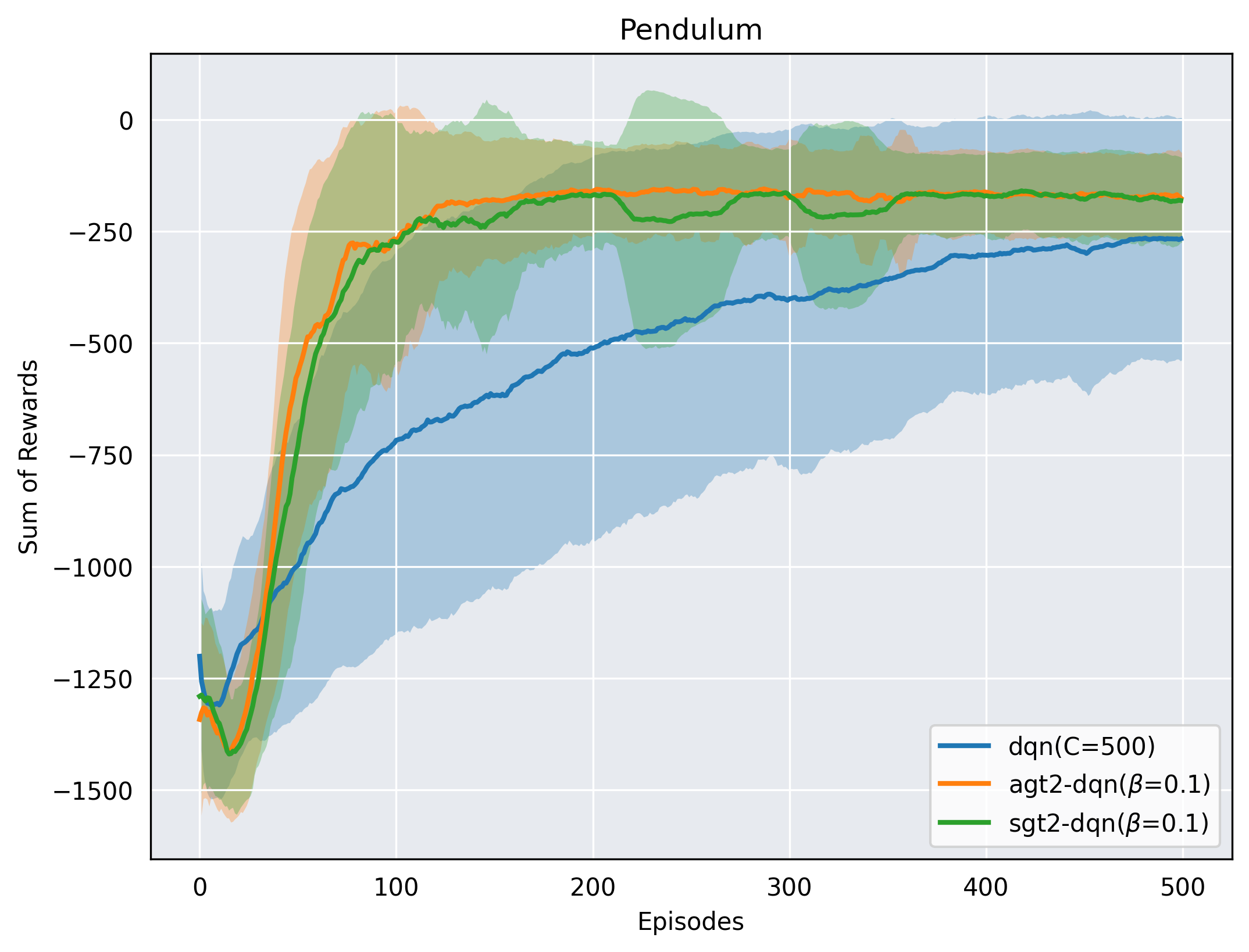}}
\caption{Comparison of simulated reward curves of AGT2-DQN,SGT2-DQN, DQN}\label{fig:comparison1}
\end{figure}

Beyond the tuning issue, our experience suggests that AGT2-DQN and SGT2-DQN achieve faster learning speeds than DQN during the early learning phase. However, they frequently encounter performance collapse after this phase. Similar to policy gradient methods, this collapse appears to result from overshooting in the gradient target update, a phenomenon that may not occur in standard DQN. We believe this issue can be mitigated by incorporating stabilization techniques, such as clipping the loss function, similar to the approach used in proximal policy optimization (PPO) in~\cite{schulman2017proximal}.

\section{Analysis of optimality}

The proposed AGT2-DQN and SGT2-DQN aim to minimize their respective loss functions through gradient descent steps. However, it is not currently clear how minimizing these errors in the loss functions ensures that the computed Q-estimates are closer to the optimal Q-function. Therefore, further investigation is required to establish a more definitive link between the reduction of loss function errors and the accuracy of the Q-estimates in approximating the optimal Q-function. This issue is the main goal of this section and will be critical for validating the effectiveness of the proposed methods.

To simplify the analysis, let us assume that by minimizing the loss functions of AGT2-DQN, we can approximately minimize the following expected loss functions:
\begin{align*}
{L_1}({\theta _1}) =& {{\mathbb E}_{(s,a) \sim U({\cal S} \times {\cal A}),s' \sim P( \cdot |s,a)}}\left[ {{{\left( {r(s,a,s') + \gamma {{\max }_{a \in {\cal A}}}{Q_{{\theta _2}}}(s',a) - {Q_{{\theta _1}}}(s,a)} \right)}^2}} \right]\\
{L_2}({\theta _2}) =& {{\mathbb E}_{(s,a) \sim U({\cal S} \times {\cal A}),s' \sim P( \cdot |s,a)}}\left[ {\frac{\beta }{2}{{({Q_{{\theta _1}}}(s,a) - {Q_{{\theta _2}}}(s,a))}^2}} \right]
\end{align*}
where $U({\cal S} \times {\cal A})$ means the uniform distribution over the set ${\cal S}\times {\cal A}$. Note that we employ the uniform distribution over ${\cal S}\times {\cal A}$ to simplify the analysis. In practice, this would correspond to the average distribution of ${\cal S}\times {\cal A}$ in the replay buffer.

Next, suppose that the loss functions are minimized with
\[{L_1}({\theta _1}) \le \varepsilon ,\quad {L_2}({\theta _2}) \le \varepsilon, \]
where $\varepsilon>0$ is a small number representing the maximum error of the two loss functions. Then, we can derive the following error bounds on the corresponding solution:
\begin{theorem}\label{thm:AGT2-DQN:optimality}
Suppose that the loss functions are minimized with
\[{L_1}({\theta _1}) \le \varepsilon ,\quad {L_2}({\theta _2}) \le \varepsilon. \]
For the expected loss functions of AGT2-DQN, we have
\begin{align*}
{\left\| Q_{\theta _1} - Q^* \right\|_\infty } \le& \frac{{\sqrt {\varepsilon |{\cal S}||{\cal A}|} }}{{1 - \gamma }} + \frac{\gamma }{{1 - \gamma }}\sqrt {\frac{2\varepsilon |{\cal S}||{\cal A}|}{\beta }} \\
{\left\| Q_{\theta_2} - Q^* \right\|_\infty } \le& \frac{{2\sqrt {\varepsilon |{\cal S}||{\cal A}|} }}{1 - \gamma} + \frac{\gamma }{1 - \gamma}\sqrt {\frac{2\varepsilon |{\cal S}||{\cal A}|}{\beta }}
\end{align*}
\end{theorem}
Its proof can be found in Appendix~H. We can observe that the error converges to $0$ as $\varepsilon \to 0$. Moreover, the errors also depend on $\beta>0$ and as $\beta \to \infty$ the error tends to be reduced.

Similarly, let us assume that by minimizing the loss functions of SGT2-DQN, we can approximately minimize the following expected loss function:
\begin{align*}
{L_1}({\theta _1}) =& {\mathbb E}_{(s,a) \sim U({\cal S} \times {\cal A}),s' \sim P( \cdot |s,a)}\left[ {{{\left( {r(s,a,s') + \gamma {\max_{a \in {\cal A}}}Q_{\theta _2}(s',a) - Q_{\theta _1}(s,a)} \right)}^2}} \right.\\
&\left. { + \frac{\beta }{2}{{(Q_{\theta _2}(s,a) - Q_{\theta _1}(s,a))}^2}} \right],\\
L_2(\theta _2) =& {\mathbb E}_{(s,a) \sim U({\cal S} \times {\cal A}), s' \sim P( \cdot |s,a)}\left[ {{{\left( r(s,a,s') + \gamma \max_{a \in {\cal A}}Q_{\theta_1}(s',a) - Q_{\theta_2}(s,a) \right)}^2}} \right.\\
&\left. { + \frac{\beta }{2}{{(Q_{\theta_1}(s,a) - Q_{\theta _2}(s,a))}^2}} \right]
\end{align*}
where $U({\cal S} \times {\cal A})$ means the uniform distribution over the set ${\cal S} \times {\cal A}$. Then, we can establish the following error bounds on the corresponding solution:
\begin{theorem}\label{thm:SGT2-DQN:optimality}
Suppose that the loss functions are minimized with
\[{L_1}({\theta _1}) \le \varepsilon ,\quad {L_2}({\theta _2}) \le \varepsilon \]
For the expected loss functions of SGT2-DQN, we have
\[{\left\| Q_{\theta _1} - Q^* \right\|_\infty },{\left\| Q_{\theta _2} - Q^* \right\|_\infty } \le \frac{{\sqrt {\varepsilon |{\cal S}||{\cal A}|} }}{{1 - \gamma }} + \frac{\gamma }{1 - \gamma }\sqrt {\frac{2\varepsilon |{\cal S}||{\cal A}|}{\beta }} \]
\end{theorem}
Its proof can be found in Appendix~I. As in the AGT2-DQN case, the error vanishes as $\varepsilon \to 0$.

\section{Convergence analysis for tabular case}

In this section, we provide theoretical convergence analysis of AGT2-DQN and SGT2-DQN in tabular cases for conceptual investigations, where the tabular versions of AGT2-DQN and SGT2-DQN are called Q-learning with asymmetric gradient tracking (AGT2-QL) and Q-learning with symmetric gradient tracking (SGT2-QL), respectively.
Let us first consider the following update of AGT2-QL:
\begin{align*}
Q^A_{k+1}(s_k,a_k)&=Q^A_k(s_k,a_k)+\alpha_k\left\{r_{k+1}+{\bf 1}(s_k') \gamma \max_{a \in {\cal A}} Q^B_k (s_{k}',a)-Q^A_k(s_k,a_k)\right\},\\
Q^B_{k+1}(s_k,a_k)&=Q^B_k(s_k,a_k)+\alpha_k \beta (Q^A_k(s_k,a_k)-Q^B_k(s_k,a_k)).
\end{align*}
where $\alpha_k >0$ is the learning rate, $\beta>0$ is a constant weight, and $s_{k}'$ can be interpreted in the following two ways: 1) in the Markovian observation model, $s_{k}' = s_{k+1}$, i.e., the next state and 2) in the i.i.d. observation model, $s_{k}'$ is the next state sampled at the current time, while it is not identical to $s_{k+1}$. In the second i.i.d. observation model, $s_{k}'$ is sampled independently from some distribution $d$. In this paper, we consider the i.i.d. observation model for our convergence analysis to simplify the overall analysis. However, we note that the analysis presented in this paper for the i.i.d. observation model can be extended to the Markovian observation model with some modifications~\cite{lim2024finite}. We also acknowledge that such more sophisticated analyses are not the main focus in this paper. The main goal of the convergence analysis in this section is to provide conceptual insights into the convergence of the proposed algorithms. We also note that simulation results are provided in the Appendix~G to demonstrate the validity and convergence of AGT2-QL. Empirically, the results show that AGT2-QL tends to converge faster for larger values of $\beta>0$.

We note that this algorithm is similar to the so-called averaging Q-learning developed in~\cite{lee2020unified}, which can be seen as a precursor. The averaging Q-learning in~\cite{lee2020unified} maintains two separate estimates, the target estimate $Q^B_k$
and the online estimate $Q^A_k$: the online estimate $Q^A_k$ is for approximating Q-function and updated through an online manner, whereas the target estimate $Q^B_k$ is for computing the target values and updated through taking Polyak’s averaging. Only the difference is the update rule for $Q^B_k$ which depends on the transition $(s_k,a_k)$.

AGT2-QL can be seen as a tabular version of AGT2-DQN due to the following insights: if we consider the two loss functions
\[{l_A}({Q^A}): = \frac{1}{2}{({y_1} - {Q^A}({s_k},{a_k}))^2},\quad {l_B}({Q^B}): = \frac{\beta }{2}{({y_2} - {Q^B}({s_k},{a_k}))^2}\]
with
\[{y_1} = {r_{k + 1}} + {\bf 1}(s_k') \gamma \max_{a \in {\cal A}}Q_k^B({s_{k + 1}},a),\quad {y_2} = Q_k^A({s_k},{a_k})\]
then each update in AGT2-QL can be interpreted as the gradient descent steps
\begin{align}
Q_{k + 1}^A({s_k},{a_k}) = Q_k^A({s_k},{a_k}) - \alpha_k \frac{{\partial {l_A}({Q^A})}}{{\partial {Q^A}({s_k},{a_k})}}\label{eq:gradient-step-A}
\end{align}
and
\begin{align}
Q_{k + 1}^B(s_k,a_k) = Q_k^B(s_k,a_k) - \alpha_k \frac{{\partial l_B(Q^B)}}{{\partial {Q^B}(s_k,a_k)}},\label{eq:gradient-step-B}
\end{align}
In other words, AGT2-QL can be seen as an tabular and online implementation of AGT2-DQN, whose convergence can be established as follows:
\begin{theorem}\label{thm:AGT2-QL-convergence}
Let us consider AGT2-QL and assume that the step-size satisfies
\begin{align}
0 \leq \alpha_k \leq 1,\quad \sum_{k=0}^\infty {\alpha_k}=\infty,\quad \sum_{k=0}^\infty{\alpha_k^2}<\infty.\label{eq:step-size-rule}
\end{align}
Assume that $(s_k,a_k,s_k')$ are i.i.d. samples, where $a_k$ is sampled from the fixed behavior policy $b$ and $s_k$ is sampled from a fixed state distribution $d$. Then for any $\beta> 0$, $Q^A_k\to Q^*$ and $Q^B_k\to Q^*$ with probability one.
\end{theorem}
The corresponding convergence analysis is given in Appendix~E. In AGT2-QL, the update for the target estimate $Q^B_k$ and the online estimate $Q^A_k$ are different, resulting in asymmetric updates for the two values. We can extend the idea in AGT2-QL by symmetrizing the update rule. In particular, by adding the two targets $y_1$ and $y_2$ and switching the roles of the target and online estimates, one can construct the following loss functions:
\[{l_A}({Q^A}): = \frac{1}{2}{({y_1} - {Q^A}({s_k},{a_k}))^2} + \frac{\beta }{2}{({Q^B}({s_k},{a_k}) - {Q^A}({s_k},{a_k}))^2}\]
and
\[{l_B}({Q^B}): = \frac{1}{2}{({y_2} - {Q^B}({s_k},{a_k}))^2} + \frac{\beta }{2}{({Q^A}({s_k},{a_k}) - {Q^B}({s_k},{a_k}))^2}\]
with the targets
\[y_1 = r_{k + 1} + {\bf 1}(s_k') \gamma {\max _{a \in {\cal A}}}Q_k^B(s_k',a),\quad {y_2} = r_{k + 1} + {\bf 1}(s_k') \gamma {\max _{a \in {\cal A}}}Q_k^A(s_k',a)\]
where $\beta >0$ is a constant representing the regularization weight.
Applying the gradient steps in~\eqref{eq:gradient-step-A} and ~\eqref{eq:gradient-step-B} results in the following SGT2-QL:
\begin{align*}
Q_{k + 1}^A({s_k},{a_k}) =& Q_k^A({s_k},{a_k})+ {\alpha _k}\{ r_{k + 1} + {\bf 1}(s_k') \gamma \max_{a \in {\cal A}}Q_k^B({s_k'},a) - Q_k^A(s_k,a_k)\\
  &+ \beta (Q_k^B({s_k},{a_k}) - Q_k^A({s_k},{a_k}))\}\\
Q_{k + 1}^B({s_k},a_k) =& Q_k^B(s_k,{a_k}) + {\alpha _k}\{ r_{k + 1} + {\bf 1}(s_k') \gamma \max_{a \in {\cal A}}Q_k^A(s_k',a) - Q_k^B(s_k,a_k)\\
& + \beta (Q_k^A({s_k},{a_k}) - Q_k^B({s_k},{a_k})) \}
\end{align*}
which can be seen as a tabular counterpart of SGT2-DQN.

We can observe that now the updates for the online and target estimates $Q^B_k$ and $Q^A_k$ are symmetric. We also note that this variant can be seen as a Q-learning counterpart of the double TD-learning developed in~\cite{lee2019target}.
In this paper, we establish its convergence as follows:
\begin{theorem}\label{thm:SGT2-QL-convergence}
Let us consider SGT2-QL and assume that the step-size satisfies~\eqref{eq:step-size-rule}.
Assume that $(s_k,a_k,s_k')$ are i.i.d. samples, where $a_k$ is sampled from the fixed behavior policy $b$ and $s_k$ is sampled from a fixed state distribution $d$.
Then for any $\beta>0$, $Q^A_k\to Q^*$ and $Q^B_k\to Q^*$ with probability one.
\end{theorem}
The proof is provided in the Appendix~F. We note that the convergence of SGT2-QL and its analysis have not yet been investigated in the literature. Additionally, empirical results demonstrating the convergence of this algorithm are included in the Appendix~G. These results show that SGT2-QL generally converges faster than AGT2-QL. Moreover, they indicate that the convergence speed of SGT2-QL is less sensitive to the weight $\beta>0$ compared to AGT2-QL.

\section{Conclusion}
This paper introduces gradient target tracking frameworks as alternatives to the hard target update in DQN. By replacing the periodic target updates with a continuous gradient-based tracking mechanism, the proposed approach mitigates tuning challenges.
Theoretical analysis establishes the convergence of the proposed methods in tabular settings. These findings suggest that the gradient-based target tracking is a promising alternative to conventional target update mechanisms in reinforcement learning.



\bibliographystyle{plainnat}
\bibliography{reference}

\newpage

\appendix

\section{Deep Q-learning with asymmetric gradient target tracking}
A detailed pseudo code of the proposed AGT2-DQN is given in~\cref{algo:appendix:AGT2-DQN}.

\begin{algorithm}[ht!]
\caption{Deep Q-learning with asymmetric gradient target tracking (AGT2-DQN)}
\begin{algorithmic}[1]
\State Initialize replay memory $D$ to capacity $|D|$
\State Randomly initialize the online parameter $\theta_1 \in {\mathbb R}^m$
\State Set $\theta_2 = \theta_1$
\State Set $k=0$

\For{Episode $i \in \{1,2,\ldots\}$}

\State Observe $s_0$

\For{$t \in\{0,1,\ldots, \tau-1\}$}

\State Take an action $a_t$ according to
\begin{align*}
a_t = \left\{ {\begin{array}{*{20}c}
   {\arg \max _{a \in {\cal A}} Q_{\theta_1} (s_t,a)\quad {\rm{with}}\,\,{\rm{probability}}\,\,1 - \varepsilon }  \\
   {a \sim {\rm{uniform}}({\cal A})\quad {\rm{with}}\,\,{\rm{probability}}\,\,\varepsilon }  \\
\end{array}} \right.
\end{align*}

\State Observe $r_{t+1},s_{t+1}$
\State Store the transition $(s_t ,a_t ,r_{t+1} ,s_{t + 1} )$ in $D$
\State Sample uniformly a random mini-batch $B$ of transitions $(s,a,r,s')$ from $D$

\State Set
\[{L_1}({\theta _1};B): = \frac{1}{2}\frac{1}{{|B|}}\sum\limits_{(s,a,r,s') \in B} {{{(y_1 - {Q_{\theta _1}}(s,a))}^2}} \]
and
\[{L_2}({\theta _2};B): = \frac{1}{2}\frac{1}{{|B|}}\sum\limits_{(s,a,r,s') \in B} {{{(y_2 - {Q_{\theta _2}}(s,a))}^2}} \]
where
\[y_1 = r + {\bf{1}}(s')\gamma \max_{a \in {\cal A}}{Q_{\theta _2}}(s',a) - {Q_{{\theta _1}}}(s,a)\]
and
\[{y_2} = {Q_{\theta_1}}(s,a)\]

\State Perform a gradient descent step
\[{\theta _1} \leftarrow {\theta _1} - \alpha {\nabla _{{\theta _1}}}{L_1}({\theta _1};B),\quad {\theta _2} \leftarrow {\theta _2} - \alpha {\nabla _{{\theta _2}}}{L_2}({\theta _2};B)\]

\State Set $k \leftarrow k+1$

\EndFor

\EndFor
\end{algorithmic}\label{algo:appendix:AGT2-DQN}
\end{algorithm}

\newpage
\section{Deep Q-learning with symmetric gradient target tracking}
A detailed pseudo code of the proposed SGT2-DQN is given in~\cref{algo:appendix:SGT2-DQN}.
\begin{algorithm}[ht!]
\caption{Deep Q-learning with symmetric gradient target tracking (SGT2-DQN)}
\begin{algorithmic}[1]
\State Initialize replay memory $D$ to capacity $|D|$
\State Randomly initialize the online parameter $\theta_1 \in {\mathbb R}^m$
\State Set $\theta_2 = \theta_1$
\State Set $k=0$

\For{Episode $i \in \{1,2,\ldots\}$}

\State Observe $s_0$

\For{$t \in\{0,1,\ldots, \tau-1\}$}

\State Take an action $a_t$ according to
\begin{align*}
a_t = \left\{ {\begin{array}{*{20}c}
   {\arg \max _{a \in {\cal A}} Q_{\theta_1} (s_t,a)\quad {\rm{with}}\,\,{\rm{probability}}\,\,1 - \varepsilon }  \\
   {a \sim {\rm{uniform}}({\cal A})\quad {\rm{with}}\,\,{\rm{probability}}\,\,\varepsilon }  \\
\end{array}} \right.
\end{align*}

\State Observe $r_{t+1},s_{t+1}$
\State Store the transition $(s_t ,a_t ,r_{t+1} ,s_{t + 1} )$ in $D$
\State Sample uniformly a random mini-batch $B$ of transitions $(s,a,r,s')$ from $D$

\State Set
\[{L_1}({\theta _1};B): = \frac{1}{2}\frac{1}{{|B|}}\sum\limits_{(s,a,r,s') \in B} {[{{(y_1 - {Q_{{\theta _1}}}(s,a))}^2} + \beta {{({Q_{{\theta _2}}}(s,a) - {Q_{{\theta _1}}}(s,a))}^2}]} \]
and
\[{L_2}({\theta _2};B): = \frac{1}{2}\frac{1}{{|B|}}\sum\limits_{(s,a,r,s') \in B} {[{{({y_2} - {Q_{{\theta _2}}}(s,a))}^2} + \beta {{(Q_{\theta _1}(s,a) - {Q_{{\theta _2}}}(s,a))}^2}]} \]
where
\[{y_1} = r + {\bf{1}}(s')\gamma \max_{a \in {\cal A}} Q_{\theta_2}(s',a) - {Q_{{\theta _1}}}(s,a)\]
and
\[{y_2} = r + {\bf{1}}(s')\gamma \max_{a \in {\cal A}}Q_{\theta_1}(s',a) - {Q_{{\theta _2}}}(s,a)\]

\State Perform a gradient descent step
\[{\theta _1} \leftarrow {\theta _1} - \alpha {\nabla _{{\theta _1}}}{L_1}({\theta _1};B),\quad {\theta _2} \leftarrow {\theta _2} - \alpha {\nabla _{{\theta _2}}}{L_2}({\theta _2};B)\]

\State Set $k \leftarrow k+1$

\EndFor

\EndFor
\end{algorithmic}\label{algo:appendix:SGT2-DQN}
\end{algorithm}

\newpage
\section{Convergence analysis with ODE model}

\subsection{Basics of nonlinear system theory}
Let us consider the nonlinear system
\begin{align}
\frac{d}{dt}x_t=f(x_t),\quad x_0=z,\quad t\in {\mathbb R}_+,\label{eq:nonlinear-system}
\end{align}
where $x_t\in {\mathbb R}^n$ is the state, $f:{\mathbb R}^n \to {\mathbb R}^n$ is a nonlinear mapping, and ${\mathbb R}_+$ denotes the set of nonnegative real numbers. For simplicity, we assume that the solution to~\eqref{eq:nonlinear-system} exists and is unique. In fact, this holds true so long as the mapping $f$ is globally Lipschitz continuous.
\begin{lemma}[{\cite[Theorem~3.2]{khalil2002nonlinear}}]\label{lemma:existence}
Let us consider the nonlinear system~\eqref{eq:nonlinear-system} and assume that $f$ is globally Lipschitz continuous, i.e.,
\begin{align}
&\|f(x)-f(y)\|\le L \|x-y\|,\quad \forall x,y \in {\mathbb R}^n,
\end{align}
for some $L>0$ and norm $\|\cdot\|$, then it has a unique solution $x(t)$ for all $t\geq 0$ and $x(0) \in {\mathbb R}^n$.
\end{lemma}

An important concept in dealing with the nonlinear system is the equilibrium point. A point $x=x^e$ in the state space is said to be an equilibrium point of~\eqref{eq:nonlinear-system} if it has the property that whenever the state of the system starts at $x^e$, it will remain at $x^e$~\citep{khalil2002nonlinear}. For~\eqref{eq:nonlinear-system}, the equilibrium points are the real roots of the equation $f(x)=0$. The equilibrium point $x^e$ is said to be globally asymptotically stable if for any initial state $x_0 \in {\mathbb R}^n$, $x_t \to x^e$ as $t \to \infty$. Now, we provide a vector comparison principle~\citep{walter1998ordinary,hirsch2006monotone,platzer2018vector} for multi-dimensional O.D.E., which will play a central role in the analysis below. We first introduce the quasi-monotone increasing function, which is a necessary prerequisite for the comparison principle.
\begin{definition}[Quasi-monotone function]\label{def:quasi-monotone}
A vector-valued function $f:{\mathbb R}^n \to {\mathbb R}^n$ with $f:=\begin{bmatrix} f_1 & f_2 & \cdots & f_n\\ \end{bmatrix}^T$ is said to be quasi-monotone increasing if $f_i (x) \le f_i (y)$ holds for all $i \in \{1,2,\ldots,n \}$ and $x,y \in {\mathbb R}^n$ such that $x_i=y_i$ and $x_j\le y_j$ for all $j\neq i$.
\end{definition}

An example of a quasi-monotone increasing function $f$ is $f(x)=Ax$ where $A$ is a Metzler matrix, which implies the off-diagonal elements of $A$ are nonnegative. Now, we introduce a vector comparison principle. For completeness, we provide a simple proof tailored to our interests.
\begin{lemma}[Vector comparison principle {\cite[Theorem~3.2]{hirsch2006monotone}}]\label{lemma:comparision-principle}
 Suppose that $\overline{f}$ and $\underline{f}$ are globally Lipschitz continuous. Let $v_t$ be a solution of the system
\begin{align*}
\frac{d}{dt}x_t=\overline{f}(x_t),\quad x_0\in {\mathbb R}^n,\forall \quad t  \geq 0,
\end{align*}
assume that $\overline{f}$ is quasi-monotone increasing, and let $v_t$ be a solution of the system
\begin{align}
\frac{d}{dt} v_t = \underline{f}(v_t), \quad v_0 < x_0, \quad \forall t \geq 0,\label{eq:lower-system}
\end{align}
where $\underline{f}(v) \le \overline{f}(v)$ holds for any $v \in {\mathbb R}^n$. Then, $v_t\leq x_t$  for all $t \geq 0$.
\end{lemma}
\begin{proof}
Instead of~\eqref{eq:lower-system}, first consider
\begin{align*}
\frac{d}{dt} v_{\varepsilon}(t) = \underline{f}(v_{\varepsilon}(t))-\varepsilon {\bf 1}_n,\quad v_{\varepsilon}(0) < x(0), \quad \forall t \geq 0
\end{align*}
where $\varepsilon>0$ is a sufficiently small real number and ${\bf 1}_n$ is a vector where all elements are ones, where we use a different notation for the time index for convenience. Suppose that the statement is not true, and let
\begin{align*}
t^* :=\inf \{t\ge 0:\exists i\,\,{\rm such\,\,that}\,\,v_{\varepsilon,i}(t)>x_{i}(t)\}<\infty,
\end{align*}
and let $i$ be such index. By the definition of $t^*$, we have that $v_{\varepsilon,i} (t^* ) = x_{i}(t^* )$ and $v_{\varepsilon ,j} (t^* ) \le x_j (t^* )$ for any $j\neq i$. Then, since $\overline{f}$ is quasi-monotone increasing, we have
\begin{align}
\overline{f}_i(v_\varepsilon (t^*))\le \overline{f}_i(x(t^*)).\label{eq:5}
\end{align}

On the other hand, by the definition of $t^*$, there exists a small $\delta >0$ such that
\begin{align*}
v_{\varepsilon,i} (t^*+\Delta t)>x_i (t^*+\Delta t)
\end{align*}
for all $0<\Delta t<\delta$. Dividing both sides by $\Delta t$ and taking the limit $\Delta t \to 0$, we have
\begin{align}
\dot v_{\varepsilon,i}(t^*)\ge\dot x_i(t^*)=\overline{f}_i(x(t^*)).\label{eq:2}
\end{align}

By the hypothesis, it holds that
\begin{align*}
\frac{d}{dt}v_\varepsilon(t)=\underline{f}(v_\varepsilon(t))-\varepsilon {\bf 1}_n < \underline{f}(v_\varepsilon(t))\le \overline{f}(v_\varepsilon(t))
\end{align*}
holds for all $t\geq 0$. The inequality implies $\dot v_{\varepsilon,i}(t)<\overline{f}_i (v_\varepsilon (t))$, which in combination with~\eqref{eq:2} leads to $\overline{f}_i (v_\varepsilon(t^*))>\overline{f}_i(x(t^*))$. However, it contradicts with~\eqref{eq:5}. Therefore, $v_\varepsilon(t)\le x(t)$ holds for all $t\geq 0$. Since the solution $v_\varepsilon(t)$ continuously depends on $\varepsilon >0$ \citep[Chap.~13]{walter1998ordinary}, taking the limit $\varepsilon \to 0$, we conclude $v_0(t)\le x(t)$ holds for all $t\geq 0$. This completes the proof.
\end{proof}

\subsection{Switching system theory}\label{sec:appendix:switching-system}

Let us consider the particular nonlinear system, the \emph{switched linear  system},
\begin{align}
&\frac{d}{dt} x_t=A_{\sigma_t} x_t,\quad x_0=z\in {\mathbb
R}^n,\quad t\in {\mathbb R}_+,\label{eq:switched-system}
\end{align}
where $x_t \in {\mathbb R}^n$ is the state,  $\sigma\in {\mathcal M}:=\{1,2,\ldots,M\}$ is called the mode,  $\sigma_t \in
{\mathcal M}$ is called the switching signal, and $\{A_\sigma,\sigma\in {\mathcal M}\}$ are called the subsystem matrices. The switching signal can be either arbitrary or controlled by the user under a certain switching policy. Especially, a state-feedback switching policy is denoted by $\sigma(x_t)$. To prove the global asymptotic stability of the switching system, we will use a fundamental algebraic stability condition of switching systems reported in~\cite{lin2009stability}.
More comprehensive surveys of stability of switching systems can be found in~\cite{lin2009stability} and~\cite{liberzon2003switching}.
\begin{lemma}[{\cite[Theorem~8]{lin2009stability}}]\label{lemma:fundamental-stability-lemma}
The origin of the linear switching system~\eqref{eq:switched-system} is the unique globally asymptotically stable equilibrium point under arbitrary switchings, $\sigma_t$, if and only if there exist a full column rank matrix , $L\in {\mathbb R}^{m\times n}$, $m\geq n$,
and a family of matrices, $\bar A_\sigma \in {\mathbb R}^{m\times n}, \sigma \in {\cal M}$, with
the so-called strictly negative row dominating diagonal condition, i.e., for each $\bar A_\sigma, \sigma \in {\cal M}$,
its elements satisfying
\begin{align*}
&[\bar A_\sigma]_{ii}+\sum_{j \in \{ 1,2, \ldots ,n\} \backslash \{ i\} } {|[\bar A_\sigma]_{ij}}|<0,\quad \forall i \in \{ 1,2, \ldots ,m\},
\end{align*}
where $[\cdot]_{ij}$ is the $(i,j)$-element of a matrix $(\cdot)$, such that the matrix relations
\begin{align*}
&L A_\sigma=\bar A_\sigma L,\quad\forall\sigma\in {\cal M},
\end{align*}
are satisfied.
\end{lemma}

A more general class of systems is the {\em affine switching system}
\begin{align}
&\frac{d}{dt} x_t=A_{\sigma_t} x_t + b_{\sigma_t},\quad x_0=z\in {\mathbb
R}^n,\quad t\in {\mathbb R}_+,\label{eq:affine-switched-system}
\end{align}
where $b_{\sigma_t} \in {\mathbb R}^n$ is the additional input vector. Due to the additional input $b_{\sigma_k}$, its stabilization becomes much more challenging.

\subsection{ODE-based stochastic approximation}\label{sec:ODE-stochastic-approximation}

Due to its generality, the convergence analyses of many RL algorithms rely on the ODE (ordinary differential equation) approach~\citep{bhatnagar2012stochastic,kushner2003stochastic}. It analyzes convergence of general stochastic recursions by examining stability of the associated ODE model based on the fact that the stochastic recursions with diminishing step-sizes approximate the corresponding ODEs in the limit. One of the most popular approach is based on the Borkar and Meyn theorem~\citep{borkar2000ode}. We now briefly introduce the Borkar and Meyn's ODE approach for analyzing convergence of the general stochastic recursions
\begin{align}
&\theta_{k+1}=\theta_k+\alpha_k (f(\theta_k)+\varepsilon_{k+1})\label{eq:general-stochastic-recursion}
\end{align}
where $f:{\mathbb R}^n \to {\mathbb R}^n$ is a nonlinear mapping. Basic technical assumptions are given below.
\begin{assumption}\label{assumption:1}
$\,$\begin{enumerate}
\item The mapping $f:{\mathbb R}^n  \to {\mathbb R}^n$ is
globally Lipschitz continuous and there exists a function
$f_\infty:{\mathbb R}^n\to {\mathbb R}^n$ such that
\begin{align*}
&\lim_{c\to \infty}\frac{f(c x)}{c}= f_\infty(x),\quad \forall x \in {\mathbb R}^n.
\end{align*}

\item The origin in ${\mathbb R}^n$ is an asymptotically stable
equilibrium for the ODE $\dot x_t=f_\infty (x_t)$.

\item There exists a unique globally asymptotically stable equilibrium
$\theta^e\in {\mathbb R}^n$ for the ODE $\dot x_t=f(x_t)$, i.e., $x_t\to\theta^e$ as $t\to\infty$.

\item The sequence $\{\varepsilon_k,{\cal G}_k,k\ge 1\}$ with ${\cal G}_k=\sigma(\theta_i,\varepsilon_i,i\le k)$
is a Martingale difference sequence. In addition, there exists a constant $C_0<\infty $ such that for any initial $\theta_0\in
{\mathbb R}^n$, we have ${\mathbb E}[\|\varepsilon_{k+1} \|^2 |{\cal G}_k]\le C_0(1+\|\theta_k\|^2),\forall k \ge 0$.

\item The step-sizes satisfy
$\alpha_k>0, \sum_{k=0}^\infty {\alpha_k}=\infty, \sum_{k=0}^\infty{\alpha_k^2}<\infty$.
\end{enumerate}
\end{assumption}

\begin{lemma}[{\cite[Borkar and Meyn theorem]{borkar2000ode}}]\label{lemma:Borkar}
Suppose that~\cref{assumption:1} holds. For any initial $\theta_0\in
{\mathbb R}^n$, $\sup_{k\ge 0} \|\theta_k\|<\infty$
with probability one. In addition, $\theta_k\to\theta^e$ as
$k\to\infty$ with probability one.
\end{lemma}
The Borkar and Meyn theorem states that under~\cref{assumption:1}, the stochastic process $(\theta_k)_{k=0}^\infty$ generated by~\eqref{eq:general-stochastic-recursion} is bounded and converges to $\theta^e$ almost surely.

\section{Assumptions and definitions}\label{sec:appendix:assumptions}

In this paper, we focus on the following setting: $\{(s_k,a_k,s_k')\}_{k=0}^{\infty}$ are i.i.d. samples under the behavior policy $\beta$, where the time-invariant behavior policy is the policy by which the RL agent actually behaves to collect experiences. Note that the notation $s_k'$ implies the next state sampled at the time step $k$, which is used instead of $s_{k+1}$ in order to distinguish $s_k'$ from $s_{k+1}$. In this paper, the notation $s_{k+1}$ indicate the current state at the iteration step $k+1$, while it does not depend on $s_k$. For simplicity, we assume that the state at each time is sampled from the stationary state distribution $p$, and in this case, the state-action distribution at each time is identically given by
\begin{align*}
d(s,a) = p (s)b (a|s),\quad (s,a) \in {\cal S} \times {\cal A}.
\end{align*}
In this paper, we assume that the behavior policy $b$ is time-invariant, and this scenario excludes the common method of using the $\varepsilon$-greedy behavior policy with $\varepsilon > 0$ because the $\varepsilon$-greedy behavior policy depends on the current Q-iterate, and hence is time-varying. Moreover, the proposed analysis cannot be easily extended to the analysis of Q-learning with the $\varepsilon$-greedy behavior policy. However, we note that this assumption is common in the ODE approaches for Q-learning and TD-learning~\citep{sutton1988learning}. This assumption can be relaxed by considering a time-varying distribution. However, this direction will not be addressed in this paper to simplify the presentation of the proofs.

Throughout, we make the following assumption for convenience.
\begin{assumption}\label{assumption:positive-distribution}
$d(s,a)> 0$ holds for all $(s,a)\in {\cal S} \times {\cal A}$.
\end{assumption}
This assumption ensures that every state-action pair is visited infinitely often for sufficient exploration. It is applied when the state-action occupation frequency is given and has also been considered in~\cite{li2021sample} and~\cite{chen2024lyapunov}.
The work in~\cite{beck2012error} introduces an alternative exploration condition, known as the cover time condition, which states that within a certain time period, every state-action pair is expected to be visited at least once. Slightly different variations of the cover time condition have been used in~\cite{even2003learning} and~\cite{li2021sample} for convergence rate analysis.

Throughout the paper, we will use the following matrix notations for compact dynamical system representations:
\begin{align*}
P:=& \begin{bmatrix}
   P_1\\
   \vdots\\
   P_{|{\cal A}|}\\
\end{bmatrix},\; R:= \begin{bmatrix}
   R_1 \\
   \vdots \\
   R_{|{\cal A}|} \\
\end{bmatrix},
\; Q:= \begin{bmatrix}
   Q(\cdot,1)\\
  \vdots\\
   Q(\cdot,|{\cal A}|)\\
\end{bmatrix},\\
D_a:=& \begin{bmatrix}
   d(1,a) & & \\
   & \ddots & \\
   & & d(|{\cal S}|,a)\\
\end{bmatrix},
\; D:= \begin{bmatrix}
   D_1 & & \\
    & \ddots  & \\
    & & D_{|{\cal A}|} \\
\end{bmatrix},
\end{align*}
where $P_a = P(\cdot,a,\cdot)\in {\mathbb R}^{|{\cal S}| \times |{\cal S}|}$, $Q(\cdot,a)\in {\mathbb R}^{|{\cal S}|},a\in {\cal A}$ and $R_a(s):={\mathbb E}[r(s,a,s')|s,a]$.
Note that $P\in{\mathbb R}^{|{\cal S}\times {\cal A}| \times |{\cal S}|  }$, $R \in {\mathbb R}^{|{\cal S}\times {\cal A}|}$, $Q\in {\mathbb R}^{|{\cal S}\times {\cal A}|}$, and $D\in {\mathbb R}^{|{\cal S}\times {\cal A}| \times |{\cal S}\times {\cal A}|}$.
In this notation, Q-function is encoded as a single vector $Q \in {\mathbb R}^{|{\cal S}\times {\cal A}|}$, which enumerates $Q(s,a)$ for all $s \in {\cal S}$ and $a \in {\cal A}$. In particular, the single value $Q(s,a)$ can be written as
\begin{align*}
Q(s,a) = (e_a  \otimes e_s )^T Q,
\end{align*}
where $e_s \in {\mathbb R}^{|{\cal S}|}$ and $e_a \in {\mathbb R}^{|{\cal A}|}$ are $s$-th basis vector (all components are $0$ except for the $s$-th component which is $1$) and $a$-th basis vector, respectively. Note also that under~\cref{assumption:positive-distribution}, $D$ is a nonsingular diagonal matrix with strictly positive diagonal elements.

For any stochastic policy, $\pi:{\cal S}\to \Delta_{\cal A}$, where $\Delta_{\cal A}$ is the set of all probability distributions over ${\cal A}$, we define the corresponding action transition matrix as
\begin{align}
\Pi^\pi:=\begin{bmatrix}
   \pi(1)^T \otimes e_1^T\\
   \pi(2)^T \otimes e_2^T\\
    \vdots\\
   \pi(|{\cal S}|)^T \otimes e_{|{\cal S}|}^T \\
\end{bmatrix}\in {\mathbb R}^{|{\cal S}| \times |{\cal S}\times {\cal A}|},\label{eq:swtiching-matrix}
\end{align}
where $e_s \in {\mathbb R}^{|{\cal S}|}$.
Then, it is well known that
$
P\Pi^\pi \in {\mathbb R}^{|{\cal S}\times {\cal A}| \times |{\cal S}\times {\cal A}|}
$
is the transition probability matrix of the state-action pair under policy $\pi$.
If we consider a deterministic policy, $\pi:{\cal S}\to {\cal A}$, the stochastic policy can be replaced with the corresponding one-hot encoding vector
$
\vec{\pi}(s):=e_{\pi(s)}\in \Delta_{\cal A},
$
where $e_a \in {\mathbb R}^{|{\cal A}|}$, and the corresponding action transition matrix is identical to~\eqref{eq:swtiching-matrix} with $\pi$ replaced with $\vec{\pi}$. For any given $Q \in {\mathbb R}^{|{\cal S}\times {\cal A}|}$, denote the greedy policy w.r.t. $Q$ as $\pi_Q(s):=\argmax_{a\in {\cal A}} Q(s,a)\in {\cal A}$.
We will use the following shorthand frequently:
$\Pi_Q:=\Pi^{\pi_Q}.$

\section{Convergence of AGT2-QL}

In this section, we study convergence of AGT2-QL. The full algorithm is described in~\cref{algo:AGT2-QL}.
\begin{algorithm}[h!]
\caption{AGT2-QL}
  \begin{algorithmic}[1]
    \State Initialize $Q_0^A$ and $Q_0^B$ randomly.
    \For{iteration $k=0,1,\ldots$}
    	\State Sample $(s,a)$
        \State Sample $s'\sim P(\cdot|s,a)$ and $r(s,a,s')$
        \State Update $Q^A_{k+1}(s,a)=Q^A_k(s,a)+\alpha_k \{r(s,a,s')+\gamma\max_{a\in {\cal A}} Q^B_k(s',a)-Q^A_k(s,a)\}$
        \State Update $Q^B_{k+1}(s,a)=Q^B_k(s,a)+\alpha_k\beta (Q^A_k(s,a)-Q^B_k(s,a))$
    \EndFor
  \end{algorithmic}\label{algo:AGT2-QL}
\end{algorithm}
We now analyze the convergence of AGT2-QL using the same switching system approach~\cite{lee2020unified}.

In particular, the main theoretical tool is the Borkar and Meyn theorem~\citep{borkar2000ode}. To apply this framework, it is crucial to study the corresponding ODE model and its asymptotic stability. Therefore, in the next subsection, we first derive the ODE model for AGT2-QL.

\subsection{Original system}
Using the notation introduced in~\cref{sec:appendix:assumptions}, the update of can be rewritten as
\begin{align*}
Q_{k + 1}^A =& Q_k^A + {\alpha _k}\left\{ {({e_a} \otimes {e_s}){{({e_a} \otimes {e_s})}^T}R} \right.\\
&\left. { + \gamma (e_a \otimes e_s){{({e_{s'}})}^T}{\max_{a \in {\cal A}}}Q_k^B( \cdot ,a) - ({e_a} \otimes {e_s}){{({e_a} \otimes {e_s})}^T}Q_k^A} \right\}\\
Q_{k + 1}^B =& Q_k^B + {\alpha _k}\beta\left\{ {({e_a} \otimes {e_s}){{({e_a} \otimes {e_s})}^T}Q_k^A - ({e_a} \otimes {e_s}){{({e_a} \otimes {e_s})}^T}Q_k^B} \right\}
\end{align*}
where $e_s \in {\mathbb R}^{|{\cal S}|}$ and $e_a \in {\mathbb R}^{|{\cal A}|}$ are $s$-th basis vector (all components are $0$ except for the $s$-th component which is $1$) and $a$-th basis vector, respectively.
The above update can be further expressed as
\begin{align*}
Q_{k + 1}^A =& Q_k^A + {\alpha _k}\{ DR + \gamma DP{\Pi _{Q_k^B}}Q_k^B - DQ_k^A + \varepsilon _{k + 1}^A\} \\
Q_{k + 1}^B =& Q_k^B + {\alpha _k}\{ \beta DQ_k^A - \beta DQ_k^B + \varepsilon _{k + 1}^B\}
\end{align*}
where
\begin{align*}
\varepsilon _{k + 1}^A =& ({e_a} \otimes {e_s}){({e_a} \otimes {e_s})^T}R + \gamma ({e_a} \otimes {e_s}){({e_{s'}})^T}{\Pi _{Q^B_k}}Q_k^B\\
& - ({e_a} \otimes {e_s}){({e_a} \otimes {e_s})^T}Q_k^A - (DR + \gamma DP{\Pi _{Q_k^B}}Q_k^B - DQ_k^A)\\
\varepsilon _{k + 1}^B =& ({e_a} \otimes {e_s}){({e_a} \otimes {e_s})^T}\beta Q_k^A - ({e_a} \otimes {e_s}){({e_a} \otimes {e_s})^T}\beta Q_k^B - \beta (DQ_k^A - DQ_k^B)
\end{align*}

As discussed in~\cref{sec:ODE-stochastic-approximation}, the convergence of AGT3-QL can be analyzed by evaluating the stability of the corresponding continuous-time ODE
\begin{align}
\frac{d}{{dt}}\left[ {\begin{array}{*{20}{c}}
{Q_t^A}\\
{Q_t^B}
\end{array}} \right] = \left[ {\begin{array}{*{20}{c}}
{ - D}&{\gamma DP{\Pi _{Q_t^B}}}\\
{\beta D}&{ - \beta D}
\end{array}} \right]\left[ {\begin{array}{*{20}{c}}
{Q_t^A}\\
{Q_t^B}
\end{array}} \right] + \left[ {\begin{array}{*{20}{c}}
{DR}\\
0
\end{array}} \right],\quad \left[ {\begin{array}{*{20}{c}}
{Q_0^A}\\
{Q_0^B}
\end{array}} \right] \in {R^{2|S||A|}},,\label{eq:appendix:AGT2-QL:original-system1}
\end{align}

Using the Bellman equation $(\gamma DP\Pi_{Q^*}-D)Q^*+DR=0$, the above expressions can be rewritten by
\begin{align}
\frac{d}{{dt}}\left[ {\begin{array}{*{20}{c}}
{Q_t^A - {Q^*}}\\
{Q_t^B - {Q^*}}
\end{array}} \right] =& \left[ {\begin{array}{*{20}{c}}
{ - D}&{\gamma DP{\Pi _{Q_t^B}}}\\
{\beta D}&{ - \beta D}
\end{array}} \right]\left[ {\begin{array}{*{20}{c}}
{Q_t^A - {Q^*}}\\
{Q_t^B - {Q^*}}
\end{array}} \right] + \left[ {\begin{array}{*{20}{c}}
{\gamma DP({\Pi _{Q_t^B}} - {\Pi _{{Q^*}}}){Q^*}}\\
0
\end{array}} \right],\nonumber\\
\left[ {\begin{array}{*{20}{c}}
{Q_0^A - {Q^*}}\\
{Q_0^B - {Q^*}}
\end{array}} \right] =& z \in {\mathbb R}^{2|{\cal S}||{\cal A}|}.\label{eq:appendix:AGT2-QL:original-system2}
\end{align}

The above system is a linear switching system discussed in~\cref{sec:appendix:switching-system}. More precisely, let $\Theta$ be the set of all deterministic policies, let us define a one-to-one mapping $\varphi :\Theta  \to \{ 1,2, \ldots ,|\Theta \times \Theta |\}$ from two deterministic policies $(\pi_A,\pi_B) \in \Theta \times \Theta$ to an integer in $\{ 1,2, \ldots ,|\Theta \times \Theta|\}$, and define
\begin{align*}
A_i = \left[ {\begin{array}{*{20}{c}}
{ - D}&{\gamma DP{\Pi ^{{\pi_B}}}}\\
{\beta D}&{ - \beta D}
\end{array}} \right] \in {\mathbb R}^{2|{\cal S} \times {\cal A}| \times 2|{\cal S} \times {\cal A}|},\quad b_i = \left[ {\begin{array}{*{20}{c}}
\gamma DP(\Pi^{\pi_B} - \Pi ^{\pi ^*})Q^*\\
0
\end{array}} \right] \in {\mathbb R}^{2|{\cal S} \times {\cal A}|}
\end{align*}
for all $i = \varphi (\pi_A,\pi_B )$ and $(\pi_A,\pi_B)  \in \Theta\times \Theta$.
Then, the above ODE can be written by the affine switching system
\begin{align*}
\frac{d}{dt}x_t= A_{\sigma(x_t)}x_t + b_{\sigma(x_t)},\quad x_0=z\in {\mathbb R}^{|{\cal S}||{\cal A}|},
\end{align*}
where
\begin{align*}
{x_t}: = \left[ {\begin{array}{*{20}{c}}
{Q_t^A - {Q^*}}\\
{Q_t^B - {Q^*}}
\end{array}} \right]
\end{align*}
is the state, $\sigma: {\mathbb R}^{|{\cal S}||{\cal A}|}\to \{1,2,\ldots ,|\Theta\times \Theta|\}$ is a state-feedback switching policy defined by $\sigma(x_t):=\psi(\pi_{Q^A_t},\pi_{Q^B_t})$, and $\pi_{Q}(s)=\argmax_{a\in {\cal A}}Q(s,a)$.

Until now, we have derived an ODE model of AGT2-QL. The next goal is to prove its asymptotic stability which is essential step to apply the Borkar and Meyn theorem~\citep{borkar2000ode}. 
Notably, proving the global asymptotic stability of the above switching system without the affine term is relatively straightforward using~\cref{lemma:fundamental-stability-lemma}. However, existing theories do not support switching systems that include affine terms.

To address this issue, we construct two comparison systems by leveraging the special structure of the switching system and the greedy policy, and we establish their global asymptotic stability. Building on the vector comparison principle introduced in~\cref{lemma:comparision-principle}, we then prove the asymptotic stability of the desired affine switching system. We note that this approach was first proposed in~\cite{lee2020unified}, and we follow its framework with modifications tailored to the proposed algorithms.

To proceed, defining the vector functions
\begin{align*}
f(x_1,x_2):=&\begin{bmatrix}
   f_1(x_1,x_2)\\
   f_2(x_1,x_2)\\
\end{bmatrix}:=\begin{bmatrix}
   -D & \gamma DP\Pi_{x_2+ Q^*}\\
   \beta D & -\beta D\\
\end{bmatrix} \begin{bmatrix}
   x_1\\
   x_2\\
\end{bmatrix}+\begin{bmatrix}
   \gamma DP(\Pi_{x_2+Q^*}-\Pi_{Q^*})Q^*\\
   0\\
\end{bmatrix},
\end{align*}
\eqref{eq:appendix:AGT2-QL:original-system2} can be written by the systems
\begin{align}
\frac{d}{dt}\begin{bmatrix}
   x_{t,1}\\
   x_{t,2}\\
\end{bmatrix}=\begin{bmatrix}
   f_1(x_{t,1},x_{t,2})\\
   f_2(x_{t,1},x_{t,2})\\
\end{bmatrix},\quad z_0 = \begin{bmatrix}
   Q_0^A-Q^*\\
   Q_0^B-Q^*\\
\end{bmatrix},\quad \forall t \geq 0\label{eq:appendix:4}
\end{align}
with
\[\left[ {\begin{array}{*{20}{c}}
{{x_{t,1}}}\\
{{x_{t,2}}}
\end{array}} \right] = \left[ {\begin{array}{*{20}{c}}
{Q_t^A - {Q^*}}\\
{Q_t^B - {Q^*}}
\end{array}} \right]\]
In the following, we present several lemmas to support the main analysis results.
\begin{lemma}\label{thm:appendix:AGT2-QL:property-f}
We have
\begin{align*}
f({x_1},{x_2}) =\left[ {\begin{array}{*{20}{c}}
{ - D{x_1} + \gamma DP{\Pi _{{x_2} + {Q^*}}}({x_2} + {Q^*}) - \gamma DP{\Pi _{{Q^*}}}{Q^*}}\\
{\beta D{x_1} - \beta D{x_2}}
\end{array}} \right]
\end{align*}
\end{lemma}
\begin{proof}
It can be proved through
\begin{align*}
f({x_1},{x_2}) =& \left[ {\begin{array}{*{20}{c}}
{ - D}&{\gamma DP{\Pi _{{x_2} + {Q^*}}}}\\
{\beta D}&{ - \beta D}
\end{array}} \right]\left[ {\begin{array}{*{20}{c}}
{{x_1}}\\
{{x_2}}
\end{array}} \right] + \left[ {\begin{array}{*{20}{c}}
{\gamma DP({\Pi _{{x_2} + {Q^*}}} - {\Pi _{{Q^*}}}){Q^*}}\\
0
\end{array}} \right]\\
=& \left[ {\begin{array}{*{20}{c}}
{ - D{x_1} + \gamma DP{\Pi _{{x_2} + {Q^*}}}{x_2} + \gamma DP({\Pi _{{x_2} + {Q^*}}} - {\Pi _{{Q^*}}}){Q^*}}\\
{\beta D{x_1} - \beta D{x_2}}
\end{array}} \right]\\
=& \left[ {\begin{array}{*{20}{c}}
{ - D{x_1} + \gamma DP{\Pi _{{x_2} + {Q^*}}}({x_2} + {Q^*}) - \gamma DP{\Pi _{{Q^*}}}{Q^*}}\\
{\beta D{x_1} - \beta D{x_2}}
\end{array}} \right],
\end{align*}
which completes the proof.
\end{proof}

\begin{lemma}\label{thm:appendix:AGT2-QL:quasi-monotone:f}
$f$ is quasi-monotone increasing.
\end{lemma}
\begin{proof}
We will check the condition of the quasi-monotone increasing function for $f_1$ and $f_2$, separately. Assume that $\Delta x_1\in {\mathbb R}^{|{\cal S}|||{\cal A}|}$ and $\Delta x_2\in {\mathbb R}^{|{\cal S}|||{\cal A}|}$ are nonnegative vectors, and an $i$the element of $\Delta x_1$ is zero.
For $f_1$, using~\cref{thm:appendix:AGT2-QL:property-f}, we have
\begin{align*}
&e_i^T{f_1}({x_1} + \Delta {x_1},{x_2} + \Delta {x_2})\\
 =&  - e_i^TD({x_1} + \Delta {x_1}) + \gamma e_i^TDP{\Pi _{{x_2} + \Delta {x_2} + {Q^*}}}({x_2} + \Delta {x_2} + {Q^*}) - \gamma e_i^TDP{\Pi _{{Q^*}}}{Q^*}\\
=&  - e_i^TD{x_1} + \gamma e_i^TDP{\Pi _{{x_2} + \Delta {x_2} + {Q^*}}}({x_2} + \Delta {x_2} + {Q^*}) - \gamma e_i^TDP{\Pi _{{Q^*}}}{Q^*}\\
\ge&  - e_i^TD{x_1} + \gamma e_i^TDP{\Pi _{{x_2} + {Q^*}}}({x_2} + {Q^*}) - \gamma e_i^TDP{\Pi _{{Q^*}}}{Q^*}\\
=& e_i^T{f_1}({x_1},{x_2}),
\end{align*}
where the second line is due to $-e_i^T D \Delta x_1 = 0$. Similarly, assuming that $\Delta x_1\in {\mathbb R}^{|{\cal S}|||{\cal A}|}$ and $\Delta x_2\in {\mathbb R}^{|{\cal S}|||{\cal A}|}$ are nonnegative vectors, and an $i$the element of $\Delta x_2$ is zero, we get
\begin{align*}
e_i^T f_2(x_1+\Delta x_1,x_2+\Delta x_2)=&\beta e_i^T D (x_1+\Delta x_1)-\beta e_i^T D (x_2+\Delta x_2)\\
=&\beta e_i^T D (x_1+\Delta x_1)-\beta e_i^T D x_2\\
\ge&\beta e_i^T D x_1- \beta e_i^T D x_2\\
=& e_i^T f_2(x_1,x_2),
\end{align*}
where the second line is due to $e_i^T D \Delta x_2=0$. Therefore, $f$ is quasi-monotone increasing.
\end{proof}

\begin{lemma}\label{thm:appendix:AGT2-QL:Lipschits:f}
$f$ is globally Lipshcitz continuous.
\end{lemma}
\begin{proof}
From~\cref{thm:appendix:AGT2-QL:property-f}, one gets
\begin{align*}
f({x_1},{x_2}) = \left[ {\begin{array}{*{20}{c}}
{ - D}&0\\
{\beta D}&{ - \beta D}
\end{array}} \right]\left[ {\begin{array}{*{20}{c}}
{{x_1}}\\
{{x_2}}
\end{array}} \right] + \left[ {\begin{array}{*{20}{c}}
{\gamma DP{\Pi _{{x_2} + {Q^*}}}({x_2} + {Q^*})}\\
0
\end{array}} \right] + \left[ {\begin{array}{*{20}{c}}
{ - \gamma DP{\Pi _{{Q^*}}}{Q^*}}\\
0
\end{array}} \right]
\end{align*}
Therefore, we have the inequalities
\begin{align*}
&{\left\| {f({x_1},{x_2}) - f({y_1},{y_2})} \right\|_\infty }\\
\le& {\left\| {\left[ {\begin{array}{*{20}{c}}
{ - D}&0\\
{\beta D}&{ - \beta D}
\end{array}} \right]\left[ {\begin{array}{*{20}{c}}
{{x_1}}\\
{{x_2}}
\end{array}} \right] - \left[ {\begin{array}{*{20}{c}}
{ - D}&0\\
{\beta D}&{ - \beta D}
\end{array}} \right]\left[ {\begin{array}{*{20}{c}}
{{y_1}}\\
{{y_2}}
\end{array}} \right]} \right\|_\infty }\\
& + {\left\| {\gamma DP{\Pi _{{x_2} + {Q^*}}}({x_2} + {Q^*}) - \gamma DP{\Pi _{{y_2} + {Q^*}}}({y_2} + {Q^*})} \right\|_\infty }\\
\le& {\left\| {\left[ {\begin{array}{*{20}{c}}
{ - D}&0\\
{\beta D}&{ - \beta D}
\end{array}} \right]} \right\|_\infty }{\left\| {\left[ {\begin{array}{*{20}{c}}
{{x_1}}\\
{{x_2}}
\end{array}} \right] - \left[ {\begin{array}{*{20}{c}}
{{y_1}}\\
{{y_2}}
\end{array}} \right]} \right\|_\infty } + {\left\| {\gamma DP} \right\|_\infty }{\left\| {{\Pi _{{x_2} + {Q^*}}}({x_2} + {Q^*}) - {\Pi _{{y_2} + {Q^*}}}({y_2} + {Q^*})} \right\|_\infty }\\
\le& {\left\| {\left[ {\begin{array}{*{20}{c}}
{ - D}&0\\
{\beta D}&{ - \beta D}
\end{array}} \right]} \right\|_\infty }{\left\| {\left[ {\begin{array}{*{20}{c}}
{{x_1}}\\
{{x_2}}
\end{array}} \right] - \left[ {\begin{array}{*{20}{c}}
{{y_1}}\\
{{y_2}}
\end{array}} \right]} \right\|_\infty } + {\left\| {\gamma DP} \right\|_\infty }{\left\| {\left[ {\begin{array}{*{20}{c}}
{{x_1}}\\
{{x_2}}
\end{array}} \right] - \left[ {\begin{array}{*{20}{c}}
{{y_1}}\\
{{y_2}}
\end{array}} \right]} \right\|_\infty }
\end{align*}
indicating that $f$ is globally Lipschitz continuous with respect to the $\|\cdot\|_\infty$ norm. This completes the proof.
\end{proof}

Our main goal here is to establish the asymptotic stability of the system given in~\eqref{eq:appendix:AGT2-QL:original-system2}. To this end, we will apply the tools introduced in~\cref{sec:appendix:switching-system}. In particular, we will derive the upper and lower comparison systems as discussed before.

\subsection{Upper comparison system}
We consider the following system which is called the upper comparison system:
\begin{align}
&\frac{d}{dt}\begin{bmatrix}
   Q_t^{A,u}-Q^*\\
   Q_t^{B,u}-Q^*\\
\end{bmatrix}=\begin{bmatrix}
   - D & \gamma DP\Pi_{Q_t^{B,u}-Q^*}\\
   \beta D & -\beta D\\
\end{bmatrix} \begin{bmatrix}
   Q_t^{A,u}-Q^*\\
   Q_t^{B,u}-Q^*\\
\end{bmatrix},\;\begin{bmatrix}
   Q_0^{A,u}-Q^*\\
   Q_0^{B,u}-Q^*\\
\end{bmatrix}> \begin{bmatrix}
   Q_0^A-Q^*\\
   Q_0^B-Q^*\\
\end{bmatrix}\in {\mathbb R}^{2|{\cal S}||{\cal A}|},\label{eq:appendix:AGT2-QL:upper-system}
\end{align}
where `$>$' in the above equation implies the element-wise inequality.

Defining the vector functions
\begin{align*}
h(x_1,x_2):=&\begin{bmatrix}
   h_1(x_1,x_2)\\
   h_2(x_1,x_2)\\
\end{bmatrix}:=\begin{bmatrix}
   -D & \gamma DP\Pi_{x_2}\\
   \beta D & - \beta D\\
\end{bmatrix} \begin{bmatrix}
   x_1\\
   x_2\\
\end{bmatrix}
\end{align*}
the upper comparison system can be written by the system
\begin{align*}
\frac{d}{dt}\begin{bmatrix}
   x_{t,1}\\
   x_{t,2}\\
\end{bmatrix}=\begin{bmatrix}
   h_1(x_{t,1},x_{t,2})\\
   h_2(x_{t,1},x_{t,2})\\
\end{bmatrix},\quad x_0 > \begin{bmatrix}
   Q_0^A-Q^*\\
   Q_0^B-Q^*\\
\end{bmatrix},\quad \forall t \geq 0
\end{align*}
with
\[\left[ {\begin{array}{*{20}{c}}
{{x_{t,1}}}\\
{{x_{t,2}}}
\end{array}} \right] = \left[ {\begin{array}{*{20}{c}}
{Q_t^{A,u} - {Q^*}}\\
{Q_t^{B,u} - {Q^*}}
\end{array}} \right].\]

In the following lemmas, we establish key properties of the upper comparison system.
\begin{lemma}\label{thm:appendix:AGT2-QL:quasi-monotone:h}
$h$ is quasi-monotone increasing.
\end{lemma}
\begin{proof}
We will check the condition of the quasi-monotone increasing function for $h_1$ and $h_2$, separately. Assume that $\Delta x_1\in {\mathbb R}^{|{\cal S}|||{\cal A}|}$ and $\Delta x_2\in {\mathbb R}^{|{\cal S}|||{\cal A}|}$ are nonnegative vectors, and an $i$the element of $\Delta x_1$ is zero. For $h_1$, we have
\begin{align*}
e_i^T h_1(x_1+\Delta x_1,x_2+\Delta x_2)=&-e_i^T D(x_1+\Delta x_1)+\gamma e_i^T DP\Pi_{x_2+\Delta x_2}(x_2+\Delta x_2)\\
=&-e_i^T D x_1+\gamma e_i^T DP\Pi_{x_2+\Delta x_2}(x_2+\Delta x_2)\\
\ge& -e_i^T Dy_1+\gamma e_i^T DP\Pi_{x_2} x_2\\
=& e_i^T h_1(x_1,x_2),
\end{align*}
where the second line is due to $-e_i^T D \Delta x_1 = 0$. Similarly, assuming that $\Delta x_1\in {\mathbb R}^{|{\cal S}|||{\cal A}|}$ and $\Delta x_2\in {\mathbb R}^{|{\cal S}|||{\cal A}|}$ are nonnegative vectors, and an $i$the element of $\Delta x_2$ is zero, we get
\begin{align*}
e_i^T h_2(x_1+\Delta x_1,x_2+\Delta x_2)=&\beta e_i^T D (x_1+\Delta x_1)-\beta e_i^T D (x_2+\Delta x_2)\\
=&\beta e_i^T D (x_1+\Delta x_1)-\beta e_i^T D x_2\\
\ge&\beta e_i^T D x_1- \beta e_i^T D x_2\\
=& e_i^T h_2(x_1,x_2),
\end{align*}
where the second line is due to $e_i^T D \Delta x_2=0$. Therefore, $h$ is quasi-monotone increasing.
\end{proof}
\begin{lemma}\label{thm:appendix:AGT2-QL:Lipschits:h}
$h$ is globally Lipshcitz continuous.
\end{lemma}
\begin{proof}
We complete the proof through the inequalities
\begin{align*}
&{\left\| {h({x_1},{x_2}) - h({y_1},{y_2})} \right\|_\infty }\\
 =& {\left\| {\left[ {\begin{array}{*{20}{c}}
{ - D}&{\gamma DP{\Pi _{{x_2}}}}\\
{\beta D}&{ - \beta D}
\end{array}} \right]\left[ {\begin{array}{*{20}{c}}
{{x_1}}\\
{{x_2}}
\end{array}} \right] - \left[ {\begin{array}{*{20}{c}}
{ - D}&{\gamma DP{\Pi _{{y_2}}}}\\
{\beta D}&{ - \beta D}
\end{array}} \right]\left[ {\begin{array}{*{20}{c}}
{{y_1}}\\
{{y_2}}
\end{array}} \right]} \right\|_\infty }\\
\le& {\left\| {\left[ {\begin{array}{*{20}{c}}
{ - D}&0\\
{\beta D}&{ - \beta D}
\end{array}} \right]\left[ {\begin{array}{*{20}{c}}
{{x_1}}\\
{{x_2}}
\end{array}} \right] - \left[ {\begin{array}{*{20}{c}}
{ - D}&0\\
{\beta D}&{ - \beta D}
\end{array}} \right]\left[ {\begin{array}{*{20}{c}}
{{y_1}}\\
{{y_2}}
\end{array}} \right]} \right\|_\infty } + {\left\| {\gamma DP{\Pi _{{x_2}}}{x_1} - \gamma DP{\Pi _{{y_2}}}{y_2}} \right\|_\infty }\\
\le& {\left\| {\left[ {\begin{array}{*{20}{c}}
{ - D}&0\\
{\beta D}&{ - \beta D}
\end{array}} \right]} \right\|_\infty }{\left\| {\left[ {\begin{array}{*{20}{c}}
{{x_1}}\\
{{x_2}}
\end{array}} \right] - \left[ {\begin{array}{*{20}{c}}
{{y_1}}\\
{{y_2}}
\end{array}} \right]} \right\|_\infty } + {\left\| {\gamma DP} \right\|_\infty }{\left\| {{x_1} - {y_2}} \right\|_\infty }\\
\le& {\left\| {\left[ {\begin{array}{*{20}{c}}
{ - D}&0\\
{\beta D}&{ - \beta D}
\end{array}} \right]} \right\|_\infty }{\left\| {\left[ {\begin{array}{*{20}{c}}
{{x_1}}\\
{{x_2}}
\end{array}} \right] - \left[ {\begin{array}{*{20}{c}}
{{y_1}}\\
{{y_2}}
\end{array}} \right]} \right\|_\infty } + {\left\| {\gamma DP} \right\|_\infty }{\left\| {\left[ {\begin{array}{*{20}{c}}
{{x_1}}\\
{{x_2}}
\end{array}} \right] - \left[ {\begin{array}{*{20}{c}}
{{y_1}}\\
{{y_2}}
\end{array}} \right]} \right\|_\infty }
\end{align*}
indicating that $h$ is globally Lipschitz continuous with respect to the $\|\cdot\|_\infty$ norm. This completes the proof.
\end{proof}

\begin{lemma}\label{thm:appendix:AGT2-QL:hf}
$h(x_1,x_2)\geq f(x_1,x_2)$ for all $(x_1,x_2)\in {\mathbb R}^{|{\cal S}||{\cal A}|}\times {\mathbb R}^{|{\cal S}||{\cal A}|}$, where `$\geq$' denotes the element-wise inequality.
\end{lemma}
\begin{proof}
Using $\gamma DP(\Pi_{x_2}-\Pi_{Q^*})Q^*\leq 0$ for all $x_2 \in {\mathbb R}^{|{\cal S}||{\cal A}|}$, we obtain
\begin{align*}
f({x_1},{x_2}) =& \left[ {\begin{array}{*{20}{c}}
{ - D}&{\gamma DP{\Pi _{{x_2} + {Q^*}}}}\\
{\beta D}&{ - \beta D}
\end{array}} \right]\left[ {\begin{array}{*{20}{c}}
{{x_1}}\\
{{x_2}}
\end{array}} \right] + \left[ {\begin{array}{*{20}{c}}
{\gamma DP({\Pi _{{x_2} + {Q^*}}} - {\Pi _{{Q^*}}}){Q^*}}\\
0
\end{array}} \right]\\
\le& \left[ {\begin{array}{*{20}{c}}
{ - D}&{\gamma DP{\Pi _{{x_2} + {Q^*}}}}\\
{\beta D}&{ - \beta D}
\end{array}} \right]\left[ {\begin{array}{*{20}{c}}
{{x_1}}\\
{{x_2}}
\end{array}} \right]\\
\le& \left[ {\begin{array}{*{20}{c}}
{ - D}&{\gamma DP{\Pi _{{x_2}}}}\\
{\beta D}&{ - \beta D}
\end{array}} \right]\left[ {\begin{array}{*{20}{c}}
{{x_1}}\\
{{x_2}}
\end{array}} \right]\\
 =& h({x_1},{x_2})
\end{align*}
for all $(x_1,x_2)\in {\mathbb R}^{|{\cal S}||{\cal A}|}\times {\mathbb R}^{|{\cal S}||{\cal A}|}$. This completes the proof.
\end{proof}

Based on the previous lemmas, we are now ready to prove that the solution of the upper comparison system indeed upper-bounds the solution of the original system, which is the reason why~\eqref{eq:appendix:AGT2-QL:upper-system} is named as such.
\begin{lemma}\label{thm:appendix:AGT2-QL:upper-original}
We have
\begin{align*}
\left[ {\begin{array}{*{20}{c}}
{Q_t^{A,u} - {Q^*}}\\
{Q_t^{B,u} - {Q^*}}
\end{array}} \right] \ge \left[ {\begin{array}{*{20}{c}}
{Q_t^A - {Q^*}}\\
{Q_t^B - {Q^*}}
\end{array}} \right],\quad \forall t \ge 0,
\end{align*}
where `$\geq$' denotes the element-wise inequality
\end{lemma}
\begin{proof}
The desired conclusion is obtained by~\cref{lemma:comparision-principle} with~\cref{thm:appendix:AGT2-QL:quasi-monotone:h,thm:appendix:AGT2-QL:hf,thm:appendix:AGT2-QL:Lipschits:f,thm:appendix:AGT2-QL:Lipschits:h}.
\end{proof}

Next, we prove that the upper comparison system~\eqref{eq:appendix:AGT2-QL:upper-system} is globally asymptotically stable at the origin. Notably, the proof is simpler than that of the original system, as the upper comparison system does not include the affine term.
\begin{lemma}\label{thm:appendix:AGT2-QL:stability:upper}
For any $\beta>0$, the origin is the unique globally asymptotically stable equilibrium point of the upper comparison system~\eqref{eq:appendix:AGT2-QL:upper-system}.
\end{lemma}
\begin{proof}
For notational convenience, we define $\Pi_\sigma$, $\sigma\in {\cal M}$ as $\Pi_{Q_t^B}$ such that $\sigma=\psi(\pi_{Q_t^B})$. Then, for the upper comparison switching system, we apply~\cref{lemma:fundamental-stability-lemma} with
\begin{align*}
A_{\sigma}=\begin{bmatrix}
   -D & \gamma D P\Pi_{\sigma}\\
   \beta D & -\beta D\\
\end{bmatrix}
\end{align*}
and
\begin{align*}
L= \begin{bmatrix}
   I & 0\\
   0 & \gamma^{1/2} I\\
\end{bmatrix},
\end{align*}
which satisfies
\[L{A_\sigma } = \left[ {\begin{array}{*{20}{c}}
I&0\\
0&{{\gamma ^{1/2}}I}
\end{array}} \right]\left[ {\begin{array}{*{20}{c}}
{ - D}&{\gamma DP{\Pi _\sigma }}\\
{\beta D}&{ - \beta D}
\end{array}} \right] = \left[ {\begin{array}{*{20}{c}}
{ - D}&{\gamma DP{\Pi _\sigma }}\\
{{\gamma ^{1/2}}\beta D}&{ - {\gamma ^{1/2}}\beta D}
\end{array}} \right] = {{\bar A}_\sigma }\left[ {\begin{array}{*{20}{c}}
I&0\\
0&{{\gamma ^{1/2}}I}
\end{array}} \right]\]
with
\begin{align*}
\bar A_\sigma=\begin{bmatrix}
   -D & \gamma^{1/2} D P\Pi_\sigma\\
   \gamma^{1/2} \beta D & -\beta D\\
\end{bmatrix}.
\end{align*}

To check the strictly negative row dominating diagonal condition, for $i \in\{1,2,\ldots,|{\cal S}||{\cal A}|\}$, we have
\begin{align*}
[\bar A_\sigma]_{ii}+\sum_{j\in \{1,2,\ldots ,n\} \backslash \{ i\}}{|[\bar A_\sigma]_{ij}|}=& [-D]_{ii}+\gamma^{1/2}[D ]_{ii} \sum_{j \in \{1,2,\ldots,n\}\backslash \{i\}} {|[P\Pi_\sigma]_{ij}|}\\
\le& [-D]_{ii}+\gamma^{1/2}[D]_{ii}\\
\le& (-1+\gamma^{1/2})[D]_{ii}
< 0.
\end{align*}

For $i\in \{|{\cal S}||{\cal A}|+1,|{\cal S}||{\cal A}|+2,\ldots,2|{\cal S}||{\cal A}|\}$, we also have
\begin{align*}
{[{{\bar A}_\sigma }]_{ii}} + {\sum _{j \in \{ 1,2, \ldots ,n\} \backslash \{ i\} }}|{[{{\bar A}_\sigma }]_{ij}}| \le  - {[D]_{ii}}\beta  + {[D]_{ii}}\beta {\gamma ^{1/2}} = {[D]_{ii}}\beta ( - 1 + {\gamma ^{1/2}}) < 0
\end{align*}
for any $\beta>0$. Therefore, the strictly negative row dominating diagonal condition is satisfied. By~\cref{lemma:fundamental-stability-lemma}, the origin of the upper comparison system is globally asymptotically stable.
\end{proof}

\subsection{Lower comparison system}
Let us consider the so-called lower comparison system
\begin{align}
\frac{d}{dt}\begin{bmatrix}
   Q_t^{A,l}-Q^*\\
   Q_t^{B,l}-Q^*\\
\end{bmatrix}=\begin{bmatrix}
   -D & \gamma DP\Pi_{\pi_{Q^*}}\\
   \beta D & -\beta D\\
\end{bmatrix}\begin{bmatrix}
   Q_t^{A,l}-Q^*\\
   Q_t^{B,l}-Q^*\\
\end{bmatrix},\quad \begin{bmatrix}
   Q_0^{A,l}-Q^*\\
   Q_0^{B,l}-Q^*\\
\end{bmatrix}< \begin{bmatrix}
   Q_0^A-Q^*\\
   Q_0^B-Q^*\\
\end{bmatrix}\in {\mathbb R}^{2|{\cal S}||{\cal A}|},\label{eq:appendix:AGT2-QL:lower-system}
\end{align}
where `$<$' in the above equation implies the element-wise inequality. Here, we note that the lower comparison system is simply a linear system. Therefore, its analysis is much simpler than the upper and original system. Defining the vector functions
\begin{align*}
g(x_1,x_2):=&\begin{bmatrix}
   g_1(x_1,x_2)\\
   g_2(x_1,x_2)\\
\end{bmatrix}=\begin{bmatrix}
   -D & \gamma DP\Pi_{\pi_{Q^*}}\\
   \beta D & -\beta D\\
\end{bmatrix} \begin{bmatrix}
   x_1\\
   x_2\\
\end{bmatrix},
\end{align*}
the lower comparison system can be written by
\begin{align*}
\frac{d}{dt}\begin{bmatrix}
   x_{t,1}\\
   x_{t,2}\\
\end{bmatrix}=\begin{bmatrix}
   g_1(x_{t,1},x_{t,2})\\
   g_2(x_{t,1},x_{t,2})\\
\end{bmatrix},\quad x_0 < \begin{bmatrix}
   Q_0^A-Q^*\\
   Q_0^B-Q^*\\
\end{bmatrix},
\end{align*}
for all $t \geq 0$.

In the sequel, we present key properties of the lower comparison system.
\begin{lemma}\label{thm:appendix:AGT2-QL:Lipschits:g}
$g$ is globally Lipschitz continuous.
\end{lemma}
\begin{proof}
It is straightforward from the linearity. In particular, we complete the proof through the inequalities
\begin{align*}
{\left\| {g({x_1},{x_2}) - g({y_1},{y_2})} \right\|_\infty } =& {\left\| {\left[ {\begin{array}{*{20}{c}}
{ - D}&{\gamma DP \Pi _{Q^*}}\\
{\beta D}&{ - \beta D}
\end{array}} \right]\left[ {\begin{array}{*{20}{c}}
{{x_1}}\\
{{x_2}}
\end{array}} \right] - \left[ {\begin{array}{*{20}{c}}
{ - D}&{\gamma DP{\Pi _{Q^*}}}\\
{\beta D}&{ - \beta D}
\end{array}} \right]\left[ {\begin{array}{*{20}{c}}
{{y_1}}\\
{{y_2}}
\end{array}} \right]} \right\|_\infty }\\
\le& {\left\| {\left[ {\begin{array}{*{20}{c}}
{ - D}&{\gamma DP{\Pi _{Q^*}}}\\
{\beta D}&{ - \beta D}
\end{array}} \right]} \right\|_\infty }{\left\| {\left[ {\begin{array}{*{20}{c}}
{{x_1}}\\
{{x_2}}
\end{array}} \right] - \left[ {\begin{array}{*{20}{c}}
{{y_1}}\\
{{y_2}}
\end{array}} \right]} \right\|_\infty }
\end{align*}
indicating that $g$ is globally Lipschitz continuous with respect to the $\|\cdot\|_\infty$ norm. This completes the proof.
\end{proof}

\begin{lemma}\label{thm:appendix:AGT2-QL:fg}
$f(x_1,x_2)\geq g(x_1,x_2)$ for all $(x_1,x_2)\in {\mathbb R}^{|{\cal S}||{\cal A}|}\times {\mathbb R}^{|{\cal S}||{\cal A}|}$, where `$\geq$' denotes the element-wise inequality.
\end{lemma}
\begin{proof}
We obtain
\begin{align*}
f({x_1},{x_2}) =& \left[ {\begin{array}{*{20}{c}}
{ - D}&{\gamma DP{\Pi _{{x_2} + {Q^*}}}}\\
{\beta D}&{ - \beta D}
\end{array}} \right]\left[ {\begin{array}{*{20}{c}}
{{x_1}}\\
{{x_2}}
\end{array}} \right] + \left[ {\begin{array}{*{20}{c}}
{\gamma DP({\Pi _{{x_2} + {Q^*}}} - {\Pi _{{Q^*}}}){Q^*}}\\
0
\end{array}} \right]\\
=& \left[ {\begin{array}{*{20}{c}}
{ - D{x_1} + \gamma DP{\Pi _{{x_2} + {Q^*}}}{x_2} + \gamma DP({\Pi _{{x_2} + {Q^*}}} - {\Pi _{{Q^*}}}){Q^*}}\\
{\beta D{x_1} - \beta D{x_2}}
\end{array}} \right]\\
=& \left[ {\begin{array}{*{20}{c}}
{ - D{x_1} + \gamma DP{\Pi _{{x_2} + {Q^*}}}({x_2} + {Q^*}) - \gamma DP{\Pi _{{Q^*}}}{Q^*}}\\
{\beta D{x_1} - \beta D{x_2}}
\end{array}} \right]\\
\ge& \left[ {\begin{array}{*{20}{c}}
{ - D{x_1} + \gamma DP{\Pi _{{Q^*}}}({x_2} + {Q^*}) - \gamma DP{\Pi _{{Q^*}}}{Q^*}}\\
{\beta D{x_1} - \beta D{x_2}}
\end{array}} \right]\\
=& \left[ {\begin{array}{*{20}{c}}
{ - D{x_1} + \gamma DP{\Pi _{{Q^*}}}{x_2}}\\
{\beta D{x_1} - \beta D{x_2}}
\end{array}} \right]\\
=& \left[ {\begin{array}{*{20}{c}}
{ - D}&{\gamma DP{\Pi _{{Q^*}}}}\\
{\beta D}&{ - \beta D}
\end{array}} \right]\left[ {\begin{array}{*{20}{c}}
{{x_1}}\\
{{x_2}}
\end{array}} \right]\\
=& g({x_1},{x_2})
\end{align*}
for all $(x_1,x_2)\in {\mathbb R}^{|{\cal S}||{\cal A}|}\times {\mathbb R}^{|{\cal S}||{\cal A}|}$. This completes the proof.
\end{proof}

Similar to the upper comparison system, we prove that the solution of the lower comparison system indeed provides a lower bound for the solution of the original system.
\begin{lemma}\label{thm:appendix:AGT2-QL:original-lower}
We have
\begin{align*}
\begin{bmatrix}
   Q_t^A-Q^*\\
   Q_t^B-Q^*\\
\end{bmatrix} \ge \begin{bmatrix}
 Q_t^{A,l}- Q^* \\
 Q_t^{B,l}-Q^* \\
\end{bmatrix},\quad \forall t \geq 0,
\end{align*}
where `$\geq$' denotes the element-wise inequality.
\end{lemma}
\begin{proof}
The desired conclusion is obtained by~\cref{lemma:comparision-principle}with~\cref{thm:appendix:AGT2-QL:quasi-monotone:f,thm:appendix:AGT2-QL:fg,thm:appendix:AGT2-QL:Lipschits:f,thm:appendix:AGT2-QL:Lipschits:g}.
\end{proof}

Moreover, the next lemma proves that the lower comparison system is also globally asymptotically stable at the origin.
\begin{lemma}\label{thm:appendix:AGT2-QL:stability:lower}
For any $\beta>0$, the origin is the unique globally asymptotically stable equilibrium point of the lower comparison system~\eqref{eq:appendix:AGT2-QL:lower-system}.
\end{lemma}
\begin{proof}
The proof follows the same procedure as that of the upper comparison system. Therefore, the detailed proof is omitted here to avoid repetition.
\end{proof}

So far, we have established several key properties of the upper and lower comparison systems, including their global asymptotic stability. In the next subsection, we prove the global asymptotic stability of the original system based on these results. 

\subsection{Stability of the original system}
We establish the global asymptotic stability of~\eqref{eq:appendix:AGT2-QL:original-system2}.
\begin{theorem}\label{thm:appendix:AGT2-QL:stability:f}
For any $\beta>0$, the origin is the unique globally asymptotically stable equilibrium point of the original system~\eqref{eq:appendix:AGT2-QL:original-system2}.
Equivalently, $\left[ {\begin{array}{*{20}{c}}
{{Q^*}}\\
{{Q^*}}
\end{array}} \right]$ is the unique globally asymptotically stable equilibrium point of the original system~\eqref{eq:appendix:AGT2-QL:original-system1}.
\end{theorem}
\begin{proof}
By~\cref{thm:appendix:AGT2-QL:original-lower,thm:appendix:AGT2-QL:upper-original}, we have
\[\left[ {\begin{array}{*{20}{c}}
{Q_t^{A,u} - {Q^*}}\\
{Q_t^{B,u} - {Q^*}}
\end{array}} \right] \ge \left[ {\begin{array}{*{20}{c}}
{Q_t^A - {Q^*}}\\
{Q_t^B - {Q^*}}
\end{array}} \right] \ge \left[ {\begin{array}{*{20}{c}}
{Q_t^{A,l} - {Q^*}}\\
{Q_t^{B,l} - {Q^*}}
\end{array}} \right],\quad \forall t \ge 0\]
Moreover, by~\cref{thm:appendix:AGT2-QL:stability:lower} and~\cref{thm:appendix:AGT2-QL:stability:upper}, we have
\[\left[ {\begin{array}{*{20}{c}}
{Q_t^{A,u} - {Q^*}}\\
{Q_t^{B,u} - {Q^*}}
\end{array}} \right] \to 0,\quad \left[ {\begin{array}{*{20}{c}}
{Q_t^{A,l} - {Q^*}}\\
{Q_t^{B,l} - {Q^*}}
\end{array}} \right] \to 0\]
as $t\to\infty$. Therefore, the state of the original system also asymptotically converges to the origin. This completes the proof.
\end{proof}

\subsection{Numerical example}
In this subsection, we present a simple example to illustrate the validity of the properties of the upper and lower comparison systems established in the previous sections. To this end, let us consider an MDP with ${\cal S}=\{1,2\}$, ${\cal A}=\{1,2\}$, $\gamma = 0.9$,
\begin{align*}
&P_1=\begin{bmatrix}
   0.2 & 0.8\\
   0.3 & 0.7 \\
\end{bmatrix},\quad P_2=\begin{bmatrix}
   0.5 & 0.5 \\
   0.7 & 0.3\\
\end{bmatrix},
\end{align*}
a behavior policy $\beta$ such that
\begin{align*}
&{\mathbb P}[a = 1|s = 1] = 0.2,\quad {\mathbb P}[a = 2|s = 1] = 0.8,\\
&{\mathbb P}[a = 1|s = 2] = 0.7,\quad {\mathbb P}[a = 2|s = 2] = 0.3,
\end{align*}
and
\begin{align*}
R_1  = \begin{bmatrix}
   3  \\
   1  \\
\end{bmatrix},\quad R_2  = \begin{bmatrix}
   2  \\
   1  \\
\end{bmatrix}
\end{align*}

Simulated trajectories of the ODE model of AGT2-QL including the upper and lower comparison systems are depicted in~\cref{fig:appendix:AGT2-QL:1} for $Q^A_t$ part and~\cref{fig:appendix:AGT2-QL:2} for $Q^B_t$ part. The results empirically prove that the ODE model associated with AGT2-QL is asymptotically stable. Moreover, they illustrate that the solutions of the upper and lower comparison systems bound the solution of the original system, as established by the theory.
\begin{figure}[h!]
\centering\includegraphics[width=14cm,height=10cm]{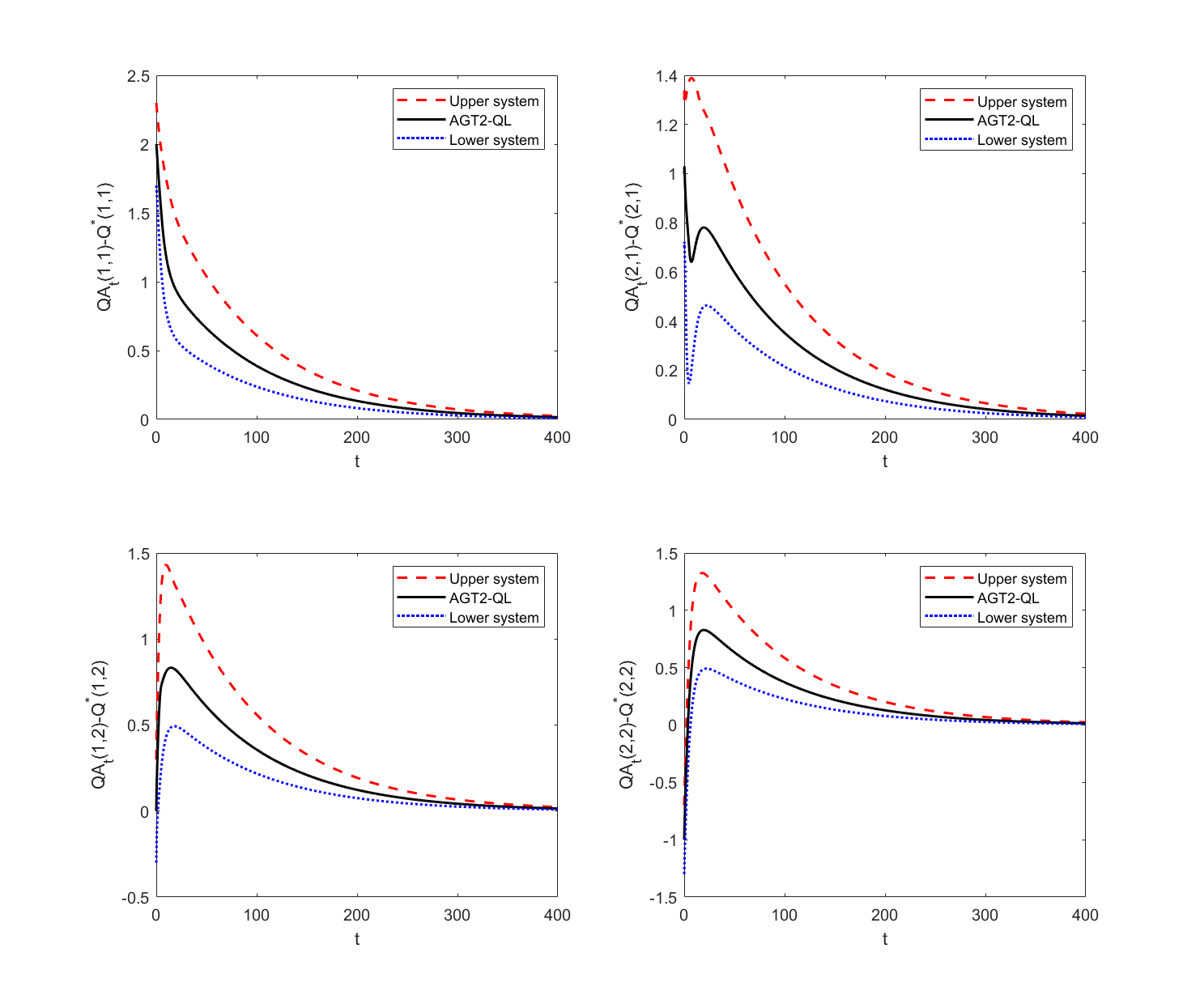}
\caption{Trajectories of the O.D.E. model of AGT2-QL and the corresponding upper and lower comparison systems ($Q^A_t$ part). 
The solution of the ODE model (black line) is upper and lower bounded by the upper and lower comparison systems, respectively (red and blue lines, respectively). This result provides numerical evidence that the bounding rules hold.}\label{fig:appendix:AGT2-QL:1}
\end{figure}
\begin{figure}[h!]
\centering\includegraphics[width=14cm,height=10cm]{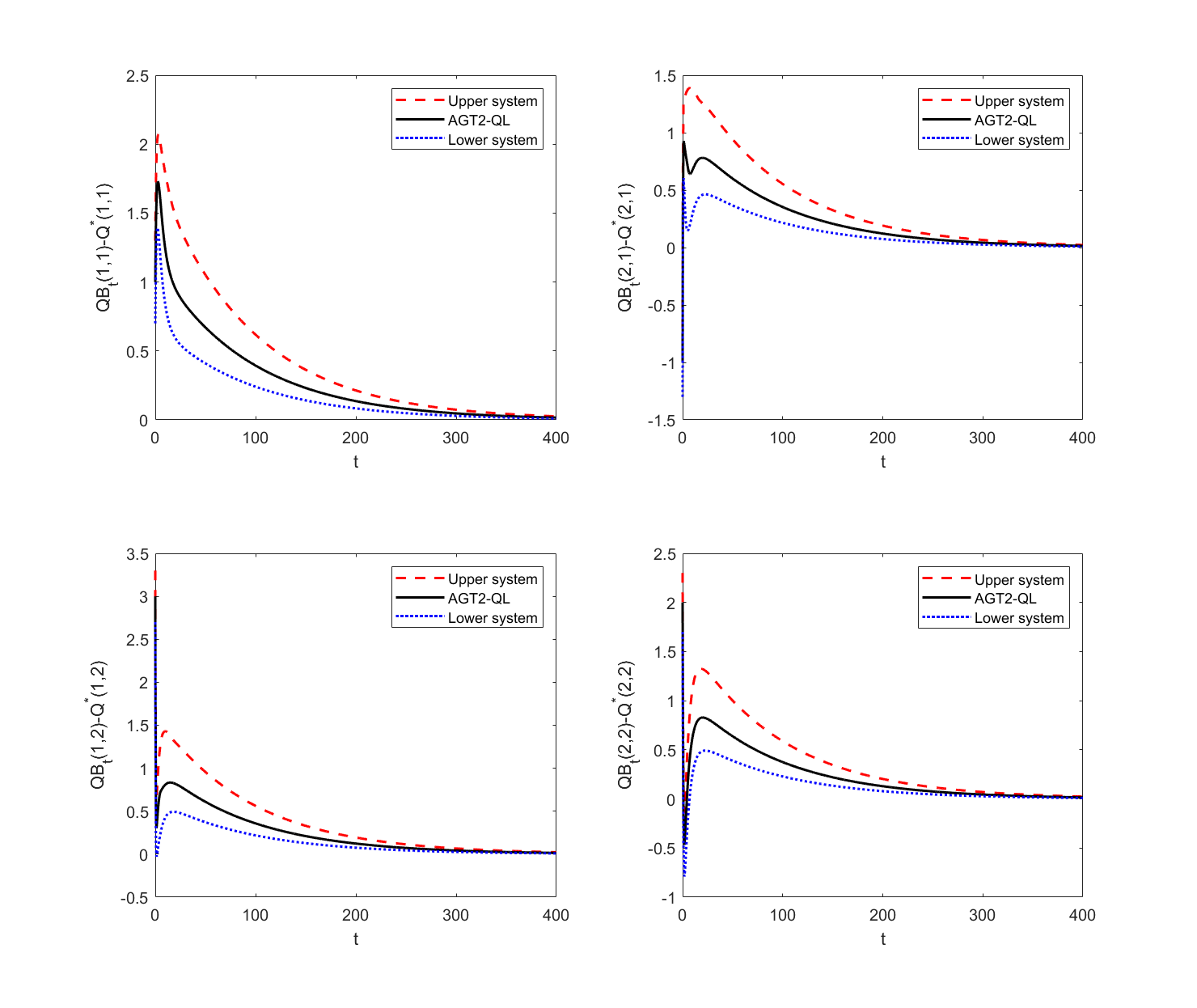}
\caption{Trajectories of the original O.D.E. model of AGT2-QL and the corresponding upper and lower comparison systems ($Q^B_t$ part). The solution of the ODE model (black line) is upper and lower bounded by the upper and lower comparison systems, respectively (red and blue lines, respectively). This result provides numerical evidence that the bounding rules hold.}\label{fig:appendix:AGT2-QL:2}
\end{figure}

\subsection{Convergence of AGT2-QL}
We can now apply the Borkar and Meyn theorem,~\cref{lemma:Borkar}, to prove~\cref{thm:AGT2-QL-convergence}.
First of all, note that the system in~\eqref{eq:appendix:4} corresponds to the ODE model in~\cref{assumption:1}. The proof is completed by examining all the statements in~\cref{assumption:1}:

\subsubsection{Step 1:}
AGT2-QL can be expressed as the stochastic recursion in~\eqref{eq:general-stochastic-recursion} with
\begin{align*}
f\left( {\left[ {\begin{array}{*{20}{c}}
{{\theta _1}}\\
{{\theta _2}}
\end{array}} \right]} \right): = \left[ {\begin{array}{*{20}{c}}
{ - D}&{\gamma DP{\Pi _{{\theta _2}}}}\\
{\beta D}&{ - \beta D}
\end{array}} \right]\left[ {\begin{array}{*{20}{c}}
{{\theta _1}}\\
{{\theta _2}}
\end{array}} \right] + \left[ {\begin{array}{*{20}{c}}
{DR}\\
0
\end{array}} \right].
\end{align*}
Moreover, $f$ is globally Lipschitz continuous according to~\cref{thm:appendix:AGT2-QL:Lipschits:f}.

To prove the first statement of~\cref{assumption:1}, we note that
\begin{align*}
{f_\infty }\left( {\left[ {\begin{array}{*{20}{c}}
{{\theta _1}}\\
{{\theta _2}}
\end{array}} \right]} \right) =& \mathop {\lim }\limits_{c \to \infty } \frac{{f\left( {\left[ {\begin{array}{*{20}{c}}
{c{\theta _1}}\\
{c{\theta _2}}
\end{array}} \right]} \right)}}{c}\\
=& \mathop {\lim }\limits_{c \to \infty } \frac{{\left[ {\begin{array}{*{20}{c}}
{ - D}&{\gamma DP{\Pi _{c{\theta _2}}}}\\
{\beta D}&{ - \beta D}
\end{array}} \right]\left[ {\begin{array}{*{20}{c}}
{c{\theta _1}}\\
{c{\theta _2}}
\end{array}} \right] + \left[ {\begin{array}{*{20}{c}}
{DR}\\
0
\end{array}} \right]}}{c}\\
=& \mathop {\lim }\limits_{c \to \infty } \left[ {\begin{array}{*{20}{c}}
{ - D}&{\gamma DP{\Pi _{c{\theta _2}}}}\\
{\beta D}&{ - \beta D}
\end{array}} \right]\left[ {\begin{array}{*{20}{c}}
{{\theta _1}}\\
{{\theta _2}}
\end{array}} \right]\\
=& \left[ {\begin{array}{*{20}{c}}
{ - D}&{\gamma DP{\Pi _{{\theta _2}}}}\\
{\beta D}&{ - \beta D}
\end{array}} \right]\left[ {\begin{array}{*{20}{c}}
{{\theta _1}}\\
{{\theta _2}}
\end{array}} \right]
\end{align*}
where the last equality is due to the homogeneity of the policy, $\pi _{c\theta_2 }(s) = \argmax _{a \in {\cal A}}c\theta_2 (s,a) = \arg {\max _{a \in A}}\theta_2 (s,a)$, where $\theta_2 (s,a)$ denotes the entry in the parameter vector $\theta_2$ corresponding to the state-action pair $(s,a)\in {\cal S}\times {\cal A}$.

\subsubsection{Step 2:}
Let us consider the system
\begin{align*}
\frac{d}{{dt}}\left[ {\begin{array}{*{20}{c}}
{{\theta _{1,t}}}\\
{{\theta _{2,t}}}
\end{array}} \right] = {f_\infty }\left( {\left[ {\begin{array}{*{20}{c}}
{{\theta _{1,t}}}\\
{{\theta _{2,t}}}
\end{array}} \right]} \right) = \left[ {\begin{array}{*{20}{c}}
{ - D}&{\gamma DP{\Pi _{{\theta _{2,t}}}}}\\
{\beta D}&{ - \beta D}
\end{array}} \right]\left[ {\begin{array}{*{20}{c}}
{{\theta _{1,t}}}\\
{{\theta _{2,t}}}
\end{array}} \right].
\end{align*}
Its global asymptotic stability around the origin can be easily proved following similar lines as in the proof of the upper comparison system in~\cref{thm:appendix:AGT2-QL:stability:upper}.

\subsubsection{Step 3:}
According to~\cref{thm:appendix:AGT2-QL:stability:f}, the ODE model of AGT2-QL is globally asymptotically stable around the optimal solution $\left[ {\begin{array}{*{20}{c}}
{{\theta _1}}\\
{{\theta _2}}
\end{array}} \right] = \left[ {\begin{array}{*{20}{c}}
{{Q^*}}\\
{{Q^*}}
\end{array}} \right]$.

\subsubsection{Step 4:}
Next, we prove the remaining parts. Recall that AGT2-QL is expressed as
\begin{align*}
{{\bar Q}_{k + 1}} = {{\bar Q}_k} + {\alpha _k}\{ f({{\bar Q}_k}) + {\varepsilon _{k + 1}}\}
\end{align*}
where ${{\bar Q}_k}: = \left[ {\begin{array}{*{20}{c}}
{Q_k^A}\\
{Q_k^B}
\end{array}} \right]$
and
\begin{align*}
{\varepsilon _{k + 1}}:=& \left[ {\begin{array}{*{20}{c}}
{ - ({e_s} \otimes {e_a}){{({e_s} \otimes {e_a})}^T}}&{  \gamma ({e_s} \otimes {e_a})e_{s'}^T{\Pi _{{Q^B}}}}\\
{\beta ({e_s} \otimes {e_a}){{({e_s} \otimes {e_a})}^T}}&{ - \beta ({e_s} \otimes {e_a}){{({e_s} \otimes {e_a})}^T}}
\end{array}} \right]\left[ {\begin{array}{*{20}{c}}
{{Q^A}}\\
{{Q^B}}
\end{array}} \right] + \left[ {\begin{array}{*{20}{c}}
{({e_s} \otimes {e_a})r(s,a,r')}\\
0
\end{array}} \right]\\
&- \left[ {\begin{array}{*{20}{c}}
{ - D}&{\gamma DP{\Pi _{{Q^B}}}}\\
{\beta D}&{ - \beta D}
\end{array}} \right]\left[ {\begin{array}{*{20}{c}}
{{Q^A}}\\
{{Q^B}}
\end{array}} \right] - \left[ {\begin{array}{*{20}{c}}
R\\
0
\end{array}} \right]
\end{align*}
and
\begin{align*}
f\left( {\left[ {\begin{array}{*{20}{c}}
{{Q^A}}\\
{{Q^B}}
\end{array}} \right]} \right): = \left[ {\begin{array}{*{20}{c}}
{ - D}&{\gamma DP{\Pi _{{Q^B}}}}\\
{\beta D}&{ - \beta D}
\end{array}} \right]\left[ {\begin{array}{*{20}{c}}
{{Q^A}}\\
{{Q^B}}
\end{array}} \right] + \left[ {\begin{array}{*{20}{c}}
{DR}\\
0
\end{array}} \right],
\end{align*}

To proceed, let us define the history
\[{\cal G}_k: = ({{\bar \varepsilon }_k},{{\bar \varepsilon }_{k - 1}}, \ldots ,{{\bar \varepsilon }_1},{{\bar Q}_k},{{\bar Q}_{k - 1}}, \ldots ,{{\bar Q}_0})\]
Moreover, let us define the corresponding process $(M_k)_{k=0}^\infty$ with $M_k:=\sum_{i=1}^k {\varepsilon_i}$. Then, we can prove that $(M_k)_{k=0}^\infty$ is Martingale. To do so, we can easily prove ${\mathbb E}[\varepsilon_{k+1}|{\cal G}_k]=0$ by
\begin{align*}
{\mathbb E}[\varepsilon _{k + 1}|{\cal G}_k] =& {\mathbb E}\left[ {\left. {\left[ {\begin{array}{*{20}{c}}
{ - (e_s \otimes e_a){(e_s \otimes e_a)^T}}&{\gamma  - (e_s \otimes e_a)e_{s'}^T \Pi_{Q^B}}\\
{\beta (e_s \otimes e_a){(e_s \otimes e_a)^T}}&{ - \beta (e_s \otimes e_a){{(e_s \otimes e_a)}^T}}
\end{array}} \right]\left[ {\begin{array}{*{20}{c}}
{{Q^A}}\\
{{Q^B}}
\end{array}} \right]} \right|{\cal G}_k} \right]\\
&- \left[ {\begin{array}{*{20}{c}}
{ - D}&{\gamma DP{\Pi _{{Q^B}}}}\\
{\beta D}&{ - \beta D}
\end{array}} \right]\left[ {\begin{array}{*{20}{c}}
{{Q^A}}\\
{{Q^B}}
\end{array}} \right]\\
=& 0
\end{align*}

Using this identity, we have
\begin{align*}
{\mathbb E}[M_{k+1}|{\cal G}_k]=& {\mathbb E}\left[ \left. \sum_{i=1}^{k+1}{\varepsilon_i} \right|{\cal G}_k\right]={\mathbb E}[\varepsilon_{k+1}|{\cal G}_k]+{\mathbb E}\left[ \left. \sum_{i=1}^k {\varepsilon_i} \right|{\cal G}_k \right]\\
=&{\mathbb E}\left[\left.\sum_{i=1}^k{\varepsilon_i} \right|{\cal G}_k \right]=\sum_{i=1}^k {\varepsilon_i}=M_k.
\end{align*}
Therefore, $(M_k)_{k=0}^\infty$ is a Martingale sequence, and $\varepsilon_{k+1} = M_{k+1}-M_k$ is a Martingale difference. Moreover, it can be easily proved that the fourth condition of~\cref{assumption:1} is satisfied by algebraic calculations. Therefore, the fourth condition is met.

\section{Convergence of SGT2-QL}

In this section, we study convergence of SGT2-QL. The full algorithm is described in~\cref{algo:SGT2-QL}.
\begin{algorithm}[h!]
\caption{SGT2-QL}
  \begin{algorithmic}[1]
    \State Initialize $Q_0^A$ and $Q_0^B$ randomly.
    \For{iteration $k=0,1,\ldots$}
    	\State Sample $(s,a)$
        \State Sample $s'\sim P(\cdot|s,a)$ and $r(s,a,s')$
        \State Update $Q^A_{k+1}(s,a)=Q^A_k(s,a)+\alpha_k \{r(s,a,s')+\gamma\max_{a\in {\cal A}} Q^B_k(s',a)-Q^A_k(s,a)\}+\alpha_k\beta (Q^B(s,a)_k-Q^A_k(s,a))$
        \State Update $Q^B_{k+1}(s,a)=Q^B_k(s,a)+\alpha_k \{r(s,a,s')+\gamma\max_{a\in {\cal A}} Q^A_k(s',a)-Q^B_k(s,a)\}+\alpha_k\beta (Q^A(s,a)_k-Q^B_k(s,a))$
    \EndFor
  \end{algorithmic}\label{algo:SGT2-QL}
\end{algorithm}
Similar to ATG2-GL, we analyze the convergence of SGT2-QL based on the same switching system approach~\cite{lee2020unified}.

\subsection{Original system}
As before, our first goal is to derive an ODE model of SGT2-QL to apply the Borkar and Meyn theorem~\cref{lemma:Borkar}. To begin with, the update rule of SGT2-QL in~\cref{algo:SGT2-QL} can be rewritten as follows:
\begin{align*}
Q_{k + 1}^A =& Q_k^A + {\alpha _k}\left\{ {({e_a} \otimes e_s){{(e_a \otimes e_s)}^T}R} \right.\\
&+ \gamma ({e_a} \otimes e_s){(e_{s'})^T}{\max _{a \in {\cal A}}}Q_k^B( \cdot ,a) - ({e_a} \otimes {e_s}){(e_a \otimes e_s)^T}Q_k^A\\
&\left. { + \beta ((e_a \otimes e_s){{(e_a \otimes e_s)}^T}Q_k^B - ({e_a} \otimes {e_s}){{({e_a} \otimes {e_s})}^T}Q_k^A)} \right\},\\
Q_{k + 1}^B =& Q_k^B + {\alpha _k}\left\{ {({e_a} \otimes {e_s}){{({e_a} \otimes {e_s})}^T}R} \right.\\
& + \gamma ({e_a} \otimes {e_s}){({e_{s'}})^T}{\max _{a \in {\cal A}}}Q_k^A( \cdot ,a) - ({e_a} \otimes {e_s}){({e_a} \otimes {e_s})^T}Q_k^B\\
&\left. { + \beta (({e_a} \otimes {e_s}){{({e_a} \otimes {e_s})}^T}Q_k^A - ({e_a} \otimes {e_s}){{({e_a} \otimes {e_s})}^T}Q_k^B)} \right\},
\end{align*}
where $e_s \in {\mathbb R}^{|{\cal S}|}$ and $e_a \in {\mathbb R}^{|{\cal A}|}$ are $s$-th basis vector (all components are $0$ except for the $s$-th component which is $1$) and $a$-th basis vector, respectively.
The above update can be further expressed as
\begin{align*}
Q_{k + 1}^A = Q_k^A + {\alpha _k}\{ DR + \gamma DP{\Pi _{Q_k^B}}Q_k^B - DQ_k^A + \beta (Q_k^B - Q_k^A) + \varepsilon _{k + 1}^A\}\\
Q_{k + 1}^B = Q_k^B + {\alpha _k}\{ DR + \gamma DP{\Pi _{Q_k^B}}Q_k^A - DQ_k^B + \beta (Q_k^A - Q_k^B) + \varepsilon _{k + 1}^B\}
\end{align*}
where
\begin{align*}
\varepsilon _{k + 1}^A =& ({e_a} \otimes {e_s}){({e_a} \otimes {e_s})^T}R + \gamma ({e_a} \otimes {e_s}){({e_{s'}})^T}{\Pi _{Q_k^B}}Q_k^B\\
& - ({e_a} \otimes {e_s}){({e_a} \otimes {e_s})^T}Q_k^A + ({e_a} \otimes {e_s}){({e_a} \otimes {e_s})^T}\beta (Q_k^B - Q_k^B)\\
&- (DR + \gamma DP{\Pi _{Q_k^B}}Q_k^B - DQ_k^A + \beta (Q_k^B - Q_k^a))\\
\varepsilon _{k + 1}^A =& ({e_a} \otimes {e_s}){({e_a} \otimes {e_s})^T}R + \gamma ({e_a} \otimes {e_s}){({e_{s'}})^T}{\Pi _{Q_k^B}}Q_k^B\\
& - ({e_a} \otimes {e_s}){({e_a} \otimes {e_s})^T}Q_k^A + ({e_a} \otimes {e_s}){({e_a} \otimes {e_s})^T}\beta (Q_k^B - Q_k^B)\\
& - (DR + \gamma DP{\Pi _{Q_k^B}}Q_k^B - DQ_k^A + \beta (Q_k^B - Q_k^a))
\end{align*}

As discussed in~\cref{sec:ODE-stochastic-approximation}, the convergence of SGT3-QL can be analyzed by evaluating the stability of the corresponding continuous-time ODE model
\begin{align}
\frac{d}{{dt}}\left[ {\begin{array}{*{20}{c}}
{Q_t^A}\\
{Q_t^B}
\end{array}} \right] = \left[ {\begin{array}{*{20}{c}}
{ - (1 + \beta )D}&{\gamma DP{\Pi _{Q_t^B}} + \beta D}\\
{\gamma DP{\Pi _{Q_t^A}} + \beta D}&{ - (1 + \beta )D}
\end{array}} \right]\left[ {\begin{array}{*{20}{c}}
{Q_t^A}\\
{Q_t^B}
\end{array}} \right] + \left[ {\begin{array}{*{20}{c}}
{DR}\\
{DR}
\end{array}} \right],\quad \left[ {\begin{array}{*{20}{c}}
{Q_0^A}\\
{Q_0^B}
\end{array}} \right] \in {\mathbb R}^{2|{\cal S}||{\cal A}|}.\label{eq:appendix:SGT2-QL:original-system1}
\end{align}

Using the Bellman equation $(\gamma DP\Pi_{Q^*}-D)Q^*+DR=0$, the above expressions can be rewritten by
\begin{align}
\frac{d}{{dt}}\left[ {\begin{array}{*{20}{c}}
{Q_t^A - {Q^*}}\\
{Q_t^B - {Q^*}}
\end{array}} \right] =& \left[ {\begin{array}{*{20}{c}}
{ - (1 + \beta )D}&{\gamma DP{\Pi _{Q_t^B}} + \beta D}\\
{\gamma DP{\Pi _{Q_t^A}} + \beta D}&{ - (1 + \beta )D}
\end{array}} \right]\left[ {\begin{array}{*{20}{c}}
{Q_t^A - Q^*}\\
{Q_t^B - Q^*}
\end{array}} \right]\nonumber\\
&+ \left[ {\begin{array}{*{20}{c}}
{\gamma DP({\Pi _{Q_t^B}} - {\Pi _{{Q^*}}}){Q^*}}\\
{\gamma DP({\Pi _{Q_t^A}} - {\Pi _{{Q^*}}}){Q^*}}
\end{array}} \right],\quad \left[ {\begin{array}{*{20}{c}}
{Q_0^A - {Q^*}}\\
{Q_0^B - {Q^*}}
\end{array}} \right] = z \in {\mathbb R}^{2|{\cal S}||{\cal A}|}.\label{eq:appendix:SGT2-QL:original-system2}
\end{align}

The above system is a linear switching system. More precisely, let $\Theta$ be the set of all deterministic policies, let us define a one-to-one mapping $\varphi :\Theta  \to \{ 1,2, \ldots ,| \Theta \times \Theta | \}$ from two deterministic policies $(\pi_A,\pi_B) \in \Theta \times \Theta$ to an integer in $\{ 1,2, \ldots ,|\Theta \times \Theta|\}$, and define
\begin{align*}
{A_i} =& \left[ {\begin{array}{*{20}{c}}
{ - (1 + \beta )D}&{\gamma DP{\Pi ^{{\pi _B}}} + \beta D}\\
{\gamma DP \Pi^{\pi _A} + \beta D}&{ - (1 + \beta )D}
\end{array}} \right] \in {\mathbb R}^{2|{\cal S} \times {\cal A}| \times 2|{\cal S} \times {\cal A}|},\\
b_i =& \left[ {\begin{array}{*{20}{c}}
{\gamma DP({\Pi ^{{\pi _B}}} - {\Pi ^{{\pi ^*}}}){Q^*}}\\
{\gamma DP({\Pi ^{{\pi _A}}} - {\Pi ^{{\pi ^*}}}){Q^*}}
\end{array}} \right] \in {\mathbb R}^{2|{\cal S} \times {\cal A}|}
\end{align*}
for all $i = \varphi (\pi_A,\pi_B )$ and $(\pi_A,\pi_B)  \in \Theta\times \Theta$.
Then, the above ODE can be written by the affine switching system
\begin{align*}
\frac{d}{dt}x_t= A_{\sigma(x_t)}x_t + b_{\sigma(x_t)},\quad x_0=z\in {\mathbb R}^{|{\cal S}||{\cal A}|},
\end{align*}
where
\begin{align*}
x_t: = \left[ {\begin{array}{*{20}{c}}
{Q_t^A - {Q^*}}\\
{Q_t^B - {Q^*}}
\end{array}} \right]
\end{align*}
is the state, $\sigma: {\mathbb R}^{|{\cal S}||{\cal A}|}\to \{1,2,\ldots ,|\Theta\times \Theta|\}$ is a state-feedback switching policy defined by $\sigma(x_t):=\psi(\pi_{Q^A_t},\pi_{Q^B_t})$, and $\pi_{Q}(s)=\argmax_{a\in {\cal A}}Q(s,a)$.

For convenience of analysis, let us define the following vector functions:
\begin{align*}
f(x_1,x_2): =& \left[ {\begin{array}{*{20}{c}}
{{f_1}({x_1},{x_2})}\\
{{f_2}({x_1},{x_2})}
\end{array}} \right]\\
:=& \left[ {\begin{array}{*{20}{c}}
{ - (1 + \beta )D}&{\gamma DP{\Pi _{{x_2} + {Q^*}}} + \beta D}\\
{\gamma DP{\Pi _{{x_1} + {Q^*}}} + \beta D}&{ - (1 + \beta )D}
\end{array}} \right]\left[ {\begin{array}{*{20}{c}}
{{x_1}}\\
{{x_2}}
\end{array}} \right] + \left[ {\begin{array}{*{20}{c}}
{\gamma DP({\Pi _{{x_2} + {Q^*}}} - {\Pi _{{Q^*}}}){Q^*}}\\
{\gamma DP({\Pi _{{x_1} + {Q^*}}} - {\Pi _{{Q^*}}}){Q^*}}
\end{array}} \right].
\end{align*}
Then,~\eqref{eq:appendix:SGT2-QL:original-system2} can be written by the system
\begin{align*}
\frac{d}{dt}\begin{bmatrix}
   x_{t,1}\\
   x_{t,2}\\
\end{bmatrix}=\begin{bmatrix}
   f_1(x_{t,1},x_{t,2})\\
   f_2(x_{t,1},x_{t,2})\\
\end{bmatrix},\quad z_0 = \begin{bmatrix}
   Q_0^A-Q^*\\
   Q_0^B-Q^*\\
\end{bmatrix},\quad \forall t \geq 0,
\end{align*}
where
\[\left[ {\begin{array}{*{20}{c}}
{{x_{t,1}}}\\
{{x_{t,2}}}
\end{array}} \right] = \left[ {\begin{array}{*{20}{c}}
{Q_t^A - {Q^*}}\\
{Q_t^B - {Q^*}}
\end{array}} \right]\]
is the state vector. Next, we present several useful lemmas that aid in the analysis of the ODE model.
\begin{lemma}\label{thm:appendix:SGT2-QL:property-f} The function $f$ can be represented by
\begin{align*}
f({x_1},{x_2}) = \left[ {\begin{array}{*{20}{c}}
{ - (1 + \beta )D{x_1} + \beta D{x_2} + \gamma DP{\Pi _{{x_2} + {Q^*}}}({x_2} + {Q^*}) - \gamma DP{\Pi _{{Q^*}}}{Q^*}}\\
{ - (1 + \beta )D{x_2} + \beta D{x_1} + \gamma DP{\Pi _{{x_1} + {Q^*}}}({x_1} + {Q^*}) - \gamma DP{\Pi _{{Q^*}}}{Q^*}}
\end{array}} \right].
\end{align*}
\end{lemma}
\begin{proof}
The proof follows from the following algebraic manipulations:
\begin{align*}
f({x_1},{x_2}) =& \left[ {\begin{array}{*{20}{c}}
{ - (1 + \beta )D}&{\gamma DP{\Pi _{{x_2} + {Q^*}}} + \beta D}\\
{\gamma DP{\Pi _{{x_1} + {Q^*}}} + \beta D}&{ - (1 + \beta )D}
\end{array}} \right]\left[ {\begin{array}{*{20}{c}}
{{x_1}}\\
{{x_2}}
\end{array}} \right] + \left[ {\begin{array}{*{20}{c}}
{\gamma DP({\Pi _{{x_2} + {Q^*}}} - {\Pi _{{Q^*}}}){Q^*}}\\
{\gamma DP({\Pi _{{x_1} + {Q^*}}} - {\Pi _{{Q^*}}}){Q^*}}
\end{array}} \right]\\
=& \left[ {\begin{array}{*{20}{c}}
{ - (1 + \beta )D{x_1} + \beta D{x_2} + \gamma DP{\Pi _{{x_2} + {Q^*}}}{x_2} + \gamma DP({\Pi _{{x_2} + {Q^*}}} - {\Pi _{{Q^*}}}){Q^*}}\\
{ - (1 + \beta )D{x_2} + \beta D{x_1} + \gamma DP{\Pi _{{x_1} + {Q^*}}}{x_1} + \gamma DP({\Pi _{{x_1} + {Q^*}}} - {\Pi _{{Q^*}}}){Q^*}}
\end{array}} \right]\\
=& \left[ {\begin{array}{*{20}{c}}
{ - (1 + \beta )D{x_1} + \beta D{x_2} + \gamma DP{\Pi _{{x_2} + {Q^*}}}({x_2} + {Q^*}) - \gamma DP{\Pi _{{Q^*}}}{Q^*}}\\
{ - (1 + \beta )D{x_2} + \beta D{x_1} + \gamma DP{\Pi _{{x_1} + {Q^*}}}({x_1} + {Q^*}) - \gamma DP{\Pi _{{Q^*}}}{Q^*}}
\end{array}} \right].
\end{align*}
This completes the proof.
\end{proof}

\begin{lemma}\label{thm:appendix:SGT2-QL:quasi-monotone:f}
$f$ is quasi-monotone increasing.
\end{lemma}
\begin{proof}
We will check the condition of the quasi-monotone increasing function for $f_1$ and $f_2$, separately. Assume that $\Delta x_1\in {\mathbb R}^{|{\cal S}|||{\cal A}|}$ and $\Delta x_2\in {\mathbb R}^{|{\cal S}|||{\cal A}|}$ are nonnegative vectors, and an $i$the element of $\Delta x_1$ is zero.
For $f_1$, using~\cref{thm:appendix:SGT2-QL:property-f}, we have
\begin{align*}
&e_i^T{f_1}({x_1} + \Delta {x_1},{x_2} + \Delta {x_2})\\
=&  - (1 + \beta )e_i^TD({x_1} + \Delta {x_1}) + \beta D({x_2} + \Delta {x_2})\\
& + \gamma e_i^TDP{\Pi _{{x_2} + \Delta {x_2} + {Q^*}}}({x_2} + \Delta {x_2} + {Q^*}) - \gamma e_i^TDP{\Pi _{{Q^*}}}{Q^*}\\
=&  - (1 + \beta )e_i^TD{x_1} + \beta D({x_2} + \Delta {x_2}) + \gamma e_i^TDP{\Pi _{{x_2} + \Delta {x_2} + {Q^*}}}({x_2} + \Delta {x_2} + {Q^*}) - \gamma e_i^TDP{\Pi _{{Q^*}}}{Q^*}\\
\ge&  - (1 + \beta )e_i^TD{x_1} + \beta D{x_2} + \gamma e_i^TDP{\Pi _{{x_2} + {Q^*}}}({x_2} + {Q^*}) - \gamma e_i^TDP{\Pi _{{Q^*}}}{Q^*}\\
=& e_i^T{f_1}({x_1},{x_2}),
\end{align*}
where the second line is due to $-(1+\beta)e_i^T D \Delta x_1 = 0$. Similarly, assuming that $\Delta x_1\in {\mathbb R}^{|{\cal S}|||{\cal A}|}$ and $\Delta x_2\in {\mathbb R}^{|{\cal S}|||{\cal A}|}$ are nonnegative vectors, and an $i$the element of $\Delta x_2$ is zero, we get
\begin{align*}
e_i^T{f_2}({x_1} + \Delta {x_1},{x_2} + \Delta {x_2}) \ge e_i^T{f_2}({x_1},{x_2}),
\end{align*}
by symmetry. Therefore, $h$ is quasi-monotone increasing.
\end{proof}

\begin{lemma}\label{thm:appendix:SGT2-QL:Lipschits:f}
$f$ is globally Lipshcitz continuous.
\end{lemma}
\begin{proof}
From~\cref{thm:appendix:SGT2-QL:property-f}, one has
\begin{align*}
f({x_1},{x_2}) =& \left[ {\begin{array}{*{20}{c}}
{ - (1 + \beta )D}&{\beta D}\\
{\beta D}&{ - (1 + \beta )D}
\end{array}} \right]\left[ {\begin{array}{*{20}{c}}
{{x_1}}\\
{{x_2}}
\end{array}} \right]\\
& + \left[ {\begin{array}{*{20}{c}}
{\gamma DP{\Pi _{{x_2} + {Q^*}}}({x_2} + {Q^*})}\\
{\gamma DP{\Pi _{{x_1} + {Q^*}}}({x_1} + {Q^*})}
\end{array}} \right] + \left[ {\begin{array}{*{20}{c}}
{ - \gamma DP{\Pi _{{Q^*}}}{Q^*}}\\
{ - \gamma DP{\Pi _{{Q^*}}}{Q^*}}
\end{array}} \right].
\end{align*}
sing this relation, we derive the following sequence of inequalities:
\begin{align*}
&{\left\| {f({x_1},{x_2}) - f({y_1},{y_2})} \right\|_\infty }\\
\le& {\left\| {\left[ {\begin{array}{*{20}{c}}
{ - (1 + \beta )D}&{\beta D}\\
{\beta D}&{ - (1 + \beta )D}
\end{array}} \right]\left[ {\begin{array}{*{20}{c}}
{{x_1}}\\
{{x_2}}
\end{array}} \right] - \left[ {\begin{array}{*{20}{c}}
{ - (1 + \beta )D}&{\beta D}\\
{\beta D}&{ - (1 + \beta )D}
\end{array}} \right]\left[ {\begin{array}{*{20}{c}}
{{y_1}}\\
{{y_2}}
\end{array}} \right]} \right\|_\infty }\\
&+ {\left\| {\gamma DP{\Pi _{{x_2} + {Q^*}}}({x_2} + {Q^*}) - \gamma DP{\Pi _{{y_2} + {Q^*}}}({y_2} + {Q^*})} \right\|_\infty } \\
&+ {\left\| {\gamma DP{\Pi _{{x_1} + {Q^*}}}({x_1} + {Q^*}) - \gamma DP{\Pi _{{y_1} + {Q^*}}}({y_1} + {Q^*})} \right\|_\infty }\\
\le& {\left\| {\left[ {\begin{array}{*{20}{c}}
{ - (1 + \beta )D}&{\beta D}\\
{\beta D}&{ - (1 + \beta )D}
\end{array}} \right]} \right\|_\infty }{\left\| {\left[ {\begin{array}{*{20}{c}}
{{x_1}}\\
{{x_2}}
\end{array}} \right] - \left[ {\begin{array}{*{20}{c}}
{{y_1}}\\
{{y_2}}
\end{array}} \right]} \right\|_\infty }\\
& + {\left\| {\gamma DP} \right\|_\infty }{\left\| {{\Pi _{{x_2} + {Q^*}}}({x_2} + {Q^*}) - {\Pi _{{y_2} + {Q^*}}}({y_2} + {Q^*})} \right\|_\infty }\\
& + {\left\| {\gamma DP} \right\|_\infty }{\left\| {{\Pi _{{x_1} + {Q^*}}}({x_1} + {Q^*}) - {\Pi _{{y_1} + {Q^*}}}({y_1} + {Q^*})} \right\|_\infty }\\
\le& {\left\| {\left[ {\begin{array}{*{20}{c}}
{ - (1 + \beta )D}&{\beta D}\\
{\beta D}&{ - (1 + \beta )D}
\end{array}} \right]} \right\|_\infty }{\left\| {\left[ {\begin{array}{*{20}{c}}
{{x_1}}\\
{{x_2}}
\end{array}} \right] - \left[ {\begin{array}{*{20}{c}}
{{y_1}}\\
{{y_2}}
\end{array}} \right]} \right\|_\infty } + {\left\| {\gamma DP} \right\|_\infty }{\left\| {{x_1} - {y_1}} \right\|_\infty }\\
& + {\left\| {\gamma DP} \right\|_\infty }{\left\| {{x_2} - {y_2}} \right\|_\infty }\\
\le& {\left\| {\left[ {\begin{array}{*{20}{c}}
{ - (1 + \beta )D}&{\beta D}\\
{\beta D}&{ - (1 + \beta )D}
\end{array}} \right]} \right\|_\infty }{\left\| {\left[ {\begin{array}{*{20}{c}}
{{x_1}}\\
{{x_2}}
\end{array}} \right] - \left[ {\begin{array}{*{20}{c}}
{{y_1}}\\
{{y_2}}
\end{array}} \right]} \right\|_\infty } + 2{\left\| {\gamma DP} \right\|_\infty }{\left\| {\left[ {\begin{array}{*{20}{c}}
{{x_1}}\\
{{x_2}}
\end{array}} \right] - \left[ {\begin{array}{*{20}{c}}
{{y_1}}\\
{{y_2}}
\end{array}} \right]} \right\|_\infty }
\end{align*}
indicating that  $f$ is globally Lipschitz continuous with respect to the $\|\cdot\|_\infty$ norm. This completes the proof.
\end{proof}

\subsection{Upper comparison system}

We consider the following system which is called the upper comparison system:
\begin{align}
\frac{d}{dt}\left[ {\begin{array}{*{20}{c}}
Q_t^{A,u} - Q^*\\
Q_t^{B,u} - Q^*
\end{array}} \right] =& \left[ {\begin{array}{*{20}{c}}
{ - (1 + \beta )D}&\gamma DP{\Pi _{Q_t^{B,u} - {Q^*}}} + \beta D\\
{\gamma DP{\Pi _{Q_t^{A,u} - {Q^*}}} + \beta D}&{ - (1 + \beta )D}
\end{array}} \right]\left[ {\begin{array}{*{20}{c}}
Q_t^{A,u} - Q^*\\
Q_t^{B,u} - Q^*
\end{array}} \right],\nonumber\\
&\left[ {\begin{array}{*{20}{c}}
Q_0^{A,u} - Q^*\\
Q_0^{B,u} - Q^*
\end{array}} \right] > \left[ {\begin{array}{*{20}{c}}
{Q_0^A - {Q^*}}\\
{Q_0^B - {Q^*}}
\end{array}} \right] \in {\mathbb R}^{2|{\cal S}||{\cal A}|}.\label{eq:appendix:SGT2-QL:upper-system}
\end{align}

Defining the following vector function:
\begin{align*}
h({x_1},{x_2}): = \left[ {\begin{array}{*{20}{c}}
{{h_1}({x_1},{x_2})}\\
{{h_2}({x_1},{x_2})}
\end{array}} \right]: = \left[ {\begin{array}{*{20}{c}}
{ - (1 + \beta )D}&{\gamma DP{\Pi _{{x_2}}} + \beta D}\\
{\gamma DP{\Pi _{{x_1}}} + \beta D}&{ - (1 + \beta )D}
\end{array}} \right]\left[ {\begin{array}{*{20}{c}}
{{x_1}}\\
{{x_2}}
\end{array}} \right]
\end{align*}
the upper comparison system can be written by the system
\begin{align*}
\frac{d}{dt}\begin{bmatrix}
   x_{t,1}\\
   x_{t,2}\\
\end{bmatrix}=\begin{bmatrix}
   h_1(x_{t,1},x_{t,2})\\
   h_2(x_{t,1},x_{t,2})\\
\end{bmatrix},\quad x_0 > \begin{bmatrix}
   Q_0^A-Q^*\\
   Q_0^B-Q^*\\
\end{bmatrix},\quad \forall t \geq 0
\end{align*}
where
\[\left[ {\begin{array}{*{20}{c}}
{{x_{t,1}}}\\
{{x_{t,2}}}
\end{array}} \right] = \left[ {\begin{array}{*{20}{c}}
{Q_t^{A,u} - {Q^*}}\\
{Q_t^{B,u} - {Q^*}}
\end{array}} \right]\]
is the state vector. Similar to the original ODE model, we derive several useful properties of the upper comparison system.
\begin{lemma}\label{thm:appendix:SGT2-QL:quasi-monotone:h}
$h$ is quasi-monotone increasing.
\end{lemma}
\begin{proof}
We will check the condition of the quasi-monotone increasing function for $h_1$ and $h_2$, separately. Assume that $\Delta x_1\in {\mathbb R}^{|{\cal S}|||{\cal A}|}$ and $\Delta x_2\in {\mathbb R}^{|{\cal S}|||{\cal A}|}$ are nonnegative vectors, and an $i$the element of $\Delta x_1$ is zero. For $h_1$, we have
\begin{align*}
&e_i^T h_1({x_1} + \Delta {x_1},{x_2} + \Delta {x_2})\\
=&  - (1 + \beta )e_i^TD({x_1} + \Delta {x_1}) + \beta D({x_2} + \Delta {x_2}) + \gamma e_i^TDP{\Pi _{{x_2} + \Delta {x_2}}}({x_2} + \Delta {x_2})\\
=&  - (1 + \beta )e_i^TD{x_1} + \beta D({x_2} + \Delta {x_2}) + \gamma e_i^TDP{\Pi _{{x_2} + \Delta {x_2}}}({x_2} + \Delta {x_2})\\
\ge & - (1 + \beta )e_i^TD{x_1} + \beta D({x_2}) + \gamma e_i^TDP{\Pi _{{x_2}}}{x_2}\\
=& e_i^T h_1({x_1},{x_2}),
\end{align*}
where the second line is due to $-e_i^T D \Delta x_1 = 0$. Similarly, assuming that $\Delta x_1\in {\mathbb R}^{|{\cal S}|||{\cal A}|}$ and $\Delta x_2\in {\mathbb R}^{|{\cal S}|||{\cal A}|}$ are nonnegative vectors, and an $i$the element of $\Delta x_2$ is zero, we get
\begin{align*}
e_i^T{h_2}({x_1} + \Delta {x_1},{x_2} + \Delta {x_2}) \ge e_i^T{h_2}({x_1},{x_2}),
\end{align*}
by symmetry. Therefore, $h$ is quasi-monotone increasing.
\end{proof}
\begin{lemma}\label{thm:appendix:SGT2-QL:Lipschits:h}
$h$ is globally Lipshcitz continuous.
\end{lemma}
\begin{proof}
We complete the proof through the inequalities
\begin{align*}
&{\left\| {h({x_1},{x_2}) - h({y_1},{y_2})} \right\|_\infty }\\
 =& {\left\| {\left[ {\begin{array}{*{20}{c}}
{ - D}&{\gamma DP{\Pi _{{x_2}}}}\\
{\beta D}&{ - \beta D}
\end{array}} \right]\left[ {\begin{array}{*{20}{c}}
{{x_1}}\\
{{x_2}}
\end{array}} \right] - \left[ {\begin{array}{*{20}{c}}
{ - D}&{\gamma DP{\Pi _{{y_2}}}}\\
{\beta D}&{ - \beta D}
\end{array}} \right]\left[ {\begin{array}{*{20}{c}}
{{y_1}}\\
{{y_2}}
\end{array}} \right]} \right\|_\infty }\\
\le& {\left\| {\left[ {\begin{array}{*{20}{c}}
{ - D}&0\\
{\beta D}&{ - \beta D}
\end{array}} \right]\left[ {\begin{array}{*{20}{c}}
{{x_1}}\\
{{x_2}}
\end{array}} \right] - \left[ {\begin{array}{*{20}{c}}
{ - D}&0\\
{\beta D}&{ - \beta D}
\end{array}} \right]\left[ {\begin{array}{*{20}{c}}
{{y_1}}\\
{{y_2}}
\end{array}} \right]} \right\|_\infty } + {\left\| {\gamma DP{\Pi _{{x_2}}}{x_1} - \gamma DP{\Pi _{{y_2}}}{y_2}} \right\|_\infty }\\
\le& {\left\| {\left[ {\begin{array}{*{20}{c}}
{ - D}&0\\
{\beta D}&{ - \beta D}
\end{array}} \right]} \right\|_\infty }{\left\| {\left[ {\begin{array}{*{20}{c}}
{{x_1}}\\
{{x_2}}
\end{array}} \right] - \left[ {\begin{array}{*{20}{c}}
{{y_1}}\\
{{y_2}}
\end{array}} \right]} \right\|_\infty } + {\left\| {\gamma DP} \right\|_\infty }{\left\| {{x_1} - {y_2}} \right\|_\infty }\\
\le& {\left\| {\left[ {\begin{array}{*{20}{c}}
{ - D}&0\\
{\beta D}&{ - \beta D}
\end{array}} \right]} \right\|_\infty }{\left\| {\left[ {\begin{array}{*{20}{c}}
{{x_1}}\\
{{x_2}}
\end{array}} \right] - \left[ {\begin{array}{*{20}{c}}
{{y_1}}\\
{{y_2}}
\end{array}} \right]} \right\|_\infty } + {\left\| {\gamma DP} \right\|_\infty }{\left\| {\left[ {\begin{array}{*{20}{c}}
{{x_1}}\\
{{x_2}}
\end{array}} \right] - \left[ {\begin{array}{*{20}{c}}
{{y_1}}\\
{{y_2}}
\end{array}} \right]} \right\|_\infty }
\end{align*}
indicating that $h$ is globally Lipschitz continuous with respect to the $\|\cdot\|_\infty$ norm. This completes the proof.
\end{proof}

\begin{lemma}\label{thm:appendix:SGT2-QL:hf}
$h(x_1,x_2)\geq f(x_1,x_2)$ for all $(x_1,x_2)\in {\mathbb R}^{|{\cal S}||{\cal A}|}\times {\mathbb R}^{|{\cal S}||{\cal A}|}$, where `$\geq$' denotes the element-wise inequality.
\end{lemma}
\begin{proof}
The proof follows from the inequalities
\begin{align*}
f({x_1},{x_2}) =& \left[ {\begin{array}{*{20}{c}}
{ - (1 + \beta )D}&{\gamma DP{\Pi _{{x_2} + {Q^*}}} + \beta D}\\
{\gamma DP{\Pi _{{x_1} + {Q^*}}} + \beta D}&{ - (1 + \beta )D}
\end{array}} \right]\left[ {\begin{array}{*{20}{c}}
{{x_1}}\\
{{x_2}}
\end{array}} \right] + \left[ {\begin{array}{*{20}{c}}
{\gamma DP({\Pi _{{x_2} + {Q^*}}} - {\Pi _{{Q^*}}}){Q^*}}\\
{\gamma DP({\Pi _{{x_1} + {Q^*}}} - {\Pi _{{Q^*}}}){Q^*}}
\end{array}} \right]\\
\le& \left[ {\begin{array}{*{20}{c}}
{ - (1 + \beta )D}&{\gamma DP{\Pi _{{x_2} + {Q^*}}} + \beta D}\\
{\gamma DP{\Pi _{{x_1} + {Q^*}}} + \beta D}&{ - (1 + \beta )D}
\end{array}} \right]\left[ {\begin{array}{*{20}{c}}
{{x_1}}\\
{{x_2}}
\end{array}} \right]\\
\le& \left[ {\begin{array}{*{20}{c}}
{ - (1 + \beta )D}&{\gamma DP{\Pi _{{x_2}}} + \beta D}\\
{\gamma DP{\Pi _{{x_1}}} + \beta D}&{ - (1 + \beta )D}
\end{array}} \right]\left[ {\begin{array}{*{20}{c}}
{{x_1}}\\
{{x_2}}
\end{array}} \right]\\
=& g({x_1},{x_2})
\end{align*}
for all $(x_1,x_2)\in {\mathbb R}^{|{\cal S}||{\cal A}|}\times {\mathbb R}^{|{\cal S}||{\cal A}|}$.
\end{proof}

Based on the previous lemmas, we are now ready to prove that the solution of the upper comparison system indeed upper-bounds the solution of the original system, which is the reason why~\eqref{eq:appendix:SGT2-QL:upper-system} is named as such.
\begin{lemma}\label{thm:appendix:SGT2-QL:upper-original}
We have
\begin{align*}
\left[ {\begin{array}{*{20}{c}}
{Q_t^{A,u} - {Q^*}}\\
{Q_t^{B,u} - {Q^*}}
\end{array}} \right] \ge \left[ {\begin{array}{*{20}{c}}
{Q_t^A - {Q^*}}\\
{Q_t^B - {Q^*}}
\end{array}} \right],\quad \forall t \ge 0,
\end{align*}
where `$\geq$' denotes the element-wise inequality.
\end{lemma}
\begin{proof}
The desired conclusion is obtained by~\cref{lemma:comparision-principle} with~\cref{thm:appendix:SGT2-QL:quasi-monotone:h,thm:appendix:SGT2-QL:hf,thm:appendix:SGT2-QL:Lipschits:f,thm:appendix:SGT2-QL:Lipschits:h}.
\end{proof}

Next, we prove that the upper comparison system~\eqref{eq:appendix:SGT2-QL:upper-system} is globally asymptotically stable at the origin. 
\begin{lemma}\label{thm:appendix:SGT2-QL:stability:upper}
For any $\beta>0$, the origin is the unique globally asymptotically stable equilibrium point of the upper comparison system~\eqref{eq:appendix:SGT2-QL:upper-system}.
\end{lemma}
\begin{proof}
For notational convenience, we define $\Pi_\sigma$, $\sigma\in {\cal M}$ as $\Pi_{Q_t^B}$ such that $\sigma=\psi(\pi_{Q_t^B})$. Then, for the upper comparison switching system, we apply~\cref{lemma:fundamental-stability-lemma} with
\begin{align*}
{A_{{\sigma _1},{\sigma _2}}} = \left[ {\begin{array}{*{20}{c}}
{ - (1 + \beta )D}&{\gamma DP{\Pi _{{\sigma _1}}} + \beta D}\\
{\gamma DP{\Pi _{{\sigma _2}}} + \beta D}&{ - (1 + \beta )D}
\end{array}} \right]
\end{align*}
and
\begin{align*}
L= \begin{bmatrix}
   I & 0\\
   0 & I\\
\end{bmatrix},
\end{align*}
which satisfies
\[L{A_{{\sigma _1},{\sigma _2}}} = {{\bar A}_{{\sigma _1},{\sigma _2}}}L\]
with
\[{{\bar A}_{{\sigma _1},{\sigma _2}}} = \left[ {\begin{array}{*{20}{c}}
{ - (1 + \beta )D}&{\gamma DP{\Pi _{{\sigma _1}}} + \beta D}\\
{\gamma DP{\Pi _{{\sigma _2}}} + \beta D}&{ - (1 + \beta )D}
\end{array}} \right]\]

To check the strictly negative row dominating diagonal condition, for $i \in\{1,2,\ldots,|{\cal S}||{\cal A}|\}$, we have
\begin{align*}
{[{{\bar A}_{{\sigma _1},{\sigma _2}}}]_{ii}} + {\sum _{j \in \{ 1,2, \ldots ,n\} \backslash \{ i\} }}|{[{{\bar A}_{{\sigma _1},{\sigma _2}}}]_{ij}}| =&  - (1 + \beta ){[D]_{ii}} + \gamma {[D]_{ii}}{\sum _{j \in \{ 1,2, \ldots ,n\} \backslash \{ i\} }}|{[P{\Pi _\sigma }]_{ij}}| + \beta {[D]_{ii}}\\
\le & - (1 + \beta ){[D]_{ii}} + (\gamma  + \beta ){[D]_{ii}}\\
=& (\gamma  - 1){[D]_{ii}}\\
<& 0
\end{align*}

For $i\in \{|{\cal S}||{\cal A}|+1,|{\cal S}||{\cal A}|+2,\ldots,2|{\cal S}||{\cal A}|\}$, by symmetry, we also have the same result
\[{[{{\bar A}_{{\sigma _1},{\sigma _2}}}]_{ii}} + {\sum _{j \in \{ 1,2, \ldots ,n\} \backslash \{ i\} }}|{[{{\bar A}_{{\sigma _1},{\sigma _2}}}]_{ij}}| < 0\]
for any $\beta>0$. Therefore, the strictly negative row dominating diagonal condition is satisfied. By~\cref{lemma:fundamental-stability-lemma}, the origin of the upper comparison system is globally asymptotically stable.
\end{proof}

\subsection{Lower comparison system}
Let us consider the so-called lower comparison system
\begin{align}
\frac{d}{{dt}}\left[ {\begin{array}{*{20}{c}}
{Q_t^{A,l} - {Q^*}}\\
{Q_t^{B,l} - {Q^*}}
\end{array}} \right] =& \left[ {\begin{array}{*{20}{c}}
{ - (1 + \beta )D}&{\gamma DP{\Pi _{{\pi _{{Q^*}}}}} + \beta D}\\
{\gamma DP{\Pi _{{\pi _{{Q^*}}}}} + \beta D}&{ - (1 + \beta )D}
\end{array}} \right]\left[ {\begin{array}{*{20}{c}}
{Q_t^{A,l} - {Q^*}}\\
{Q_t^{B,l} - {Q^*}}
\end{array}} \right],\nonumber\\
&\left[ {\begin{array}{*{20}{c}}
{Q_0^{A,l} - {Q^*}}\\
{Q_0^{B,l} - {Q^*}}
\end{array}} \right] < \left[ {\begin{array}{*{20}{c}}
{Q_0^A - {Q^*}}\\
{Q_0^B - {Q^*}}
\end{array}} \right] \in {\mathbb R}^{2|{\cal S}||{\cal A}|},\label{eq:appendix:SGT2-QL:lower-system}
\end{align}
We note that the lower comparison system is simply a linear system. Defining the vector function
\begin{align*}
g({x_1},{x_2}): = \left[ {\begin{array}{*{20}{c}}
{{g_1}({x_1},{x_2})}\\
{{g_2}({x_1},{x_2})}
\end{array}} \right] = \left[ {\begin{array}{*{20}{c}}
{ - (1 + \beta )D}&{\gamma DP{\Pi _{{\pi _{{Q^*}}}}} + \beta D}\\
{\gamma DP\Pi _{\pi _{Q^*}} + \beta D}&{ - (1 + \beta )D}
\end{array}} \right]\left[ {\begin{array}{*{20}{c}}
{{x_1}}\\
{{x_2}}
\end{array}} \right],
\end{align*}
the lower comparison system can be written by
\begin{align*}
\frac{d}{dt}\begin{bmatrix}
   x_{t,1}\\
   x_{t,2}\\
\end{bmatrix}=\begin{bmatrix}
   g_1(x_{t,1},x_{t,2})\\
   g_2(x_{t,1},x_{t,2})\\
\end{bmatrix},\quad x_0 < \begin{bmatrix}
   Q_0^A-Q^*\\
   Q_0^B-Q^*\\
\end{bmatrix},
\end{align*}
for all $t \geq 0$.
Below, we present several key properties of the lower comparison system.
\begin{lemma}\label{thm:appendix:SGT2-QL:Lipschits:g}
$g$ is globally Lipschitz continuous.
\end{lemma}
\begin{proof}
It is straightforward from the linearity. In particular, we complete the proof through the inequalities
\begin{align*}
{\left\| {g({x_1},{x_2}) - g({y_1},{y_2})} \right\|_\infty } \le {\left\| {\left[ {\begin{array}{*{20}{c}}
{ - (1 + \beta )D}&{\gamma DP{\Pi _{{\pi _{{Q^*}}}}} + \beta D}\\
{\gamma DP{\Pi _{{\pi _{{Q^*}}}}} + \beta D}&{ - (1 + \beta )D}
\end{array}} \right]} \right\|_\infty }{\left\| {\left[ {\begin{array}{*{20}{c}}
{{x_1}}\\
{{x_2}}
\end{array}} \right] - \left[ {\begin{array}{*{20}{c}}
{{y_1}}\\
{{y_2}}
\end{array}} \right]} \right\|_\infty }
\end{align*}
indicating that $g$ is globally Lipschitz continuous with respect to the $\|\cdot\|_\infty$ norm. This completes the proof.
\end{proof}
\begin{lemma}\label{thm:appendix:SGT2-QL:fg}
$f(x_1,x_2)\geq g(x_1,x_2)$ for all $(x_1,x_2)\in {\mathbb R}^{|{\cal S}||{\cal A}|}\times {\mathbb R}^{|{\cal S}||{\cal A}|}$, where `$\geq$' denotes the element-wise inequality.
\end{lemma}
\begin{proof}
The proof follows directly from the inequalities
\begin{align*}
f({x_1},{x_2}) =& \left[ {\begin{array}{*{20}{c}}
{ - (1 + \beta )D{x_1} + \beta D{x_2} + \gamma DP{\Pi _{{x_2} + {Q^*}}}({x_2} + {Q^*}) - \gamma DP{\Pi _{{Q^*}}}{Q^*}}\\
{ - (1 + \beta )D{x_2} + \beta D{x_1} + \gamma DP{\Pi _{{x_1} + {Q^*}}}({x_1} + {Q^*}) - \gamma DP{\Pi _{{Q^*}}}{Q^*}}
\end{array}} \right]\\
\ge& \left[ {\begin{array}{*{20}{c}}
{ - (1 + \beta )D{x_1} + \beta D{x_2} + \gamma DP{\Pi _{{Q^*}}}({x_2} + {Q^*}) - \gamma DP{\Pi _{{Q^*}}}{Q^*}}\\
{ - (1 + \beta )D{x_2} + \beta D{x_1} + \gamma DP{\Pi _{{Q^*}}}({x_1} + {Q^*}) - \gamma DP{\Pi _{{Q^*}}}{Q^*}}
\end{array}} \right]\\
=& \left[ {\begin{array}{*{20}{c}}
{ - (1 + \beta )D}&{\gamma DP{\Pi _{{\pi _{{Q^*}}}}} + \beta D}\\
{\gamma DP{\Pi _{{\pi _{{Q^*}}}}} + \beta D}&{ - (1 + \beta )D}
\end{array}} \right]\left[ {\begin{array}{*{20}{c}}
{{x_1}}\\
{{x_2}}
\end{array}} \right]\\
=& g({x_1},{x_2})
\end{align*}
for all $(x_1,x_2)\in {\mathbb R}^{|{\cal S}||{\cal A}|}\times {\mathbb R}^{|{\cal S}||{\cal A}|}$. This completes the proof.
\end{proof}

Similar to the upper comparison system, we prove that the solution of the lower comparison system indeed provides a lower bound for the solution of the original system.
\begin{lemma}\label{thm:appendix:SGT2-QL:original-lower}
We have
\begin{align*}
\begin{bmatrix}
   Q_t^A-Q^*\\
   Q_t^B-Q^*\\
\end{bmatrix} \ge \begin{bmatrix}
 Q_t^{A,l}- Q^* \\
 Q_t^{B,l}-Q^* \\
\end{bmatrix},\quad \forall t \geq 0,
\end{align*}
where `$\geq$' denotes the element-wise inequality.
\end{lemma}
\begin{proof}
The desired conclusion is obtained by~\cref{lemma:comparision-principle}with~\cref{thm:appendix:SGT2-QL:quasi-monotone:f,thm:appendix:SGT2-QL:fg,thm:appendix:SGT2-QL:Lipschits:f,thm:appendix:SGT2-QL:Lipschits:g}.
\end{proof}

Moreover, the next lemma proves that the lower comparison system is also globally asymptotically stable at the origin.
\begin{lemma}\label{thm:appendix:SGT2-QL:stability:lower}
For any $\beta>0$, the origin is the unique globally asymptotically stable equilibrium point of the lower comparison system~\eqref{eq:appendix:SGT2-QL:lower-system}.
\end{lemma}
\begin{proof}
For the proof, one can apply the same procedure as in the proof of the upper comparison system.
\end{proof}

So far, we have established several key properties of the upper and lower comparison systems, including their global asymptotic stability. In the next subsection, we prove the global asymptotic stability of the original system based on these results.

\subsection{Stability of the original system}
We establish the global asymptotic stability of~\eqref{eq:appendix:SGT2-QL:original-system2}.
\begin{theorem}\label{thm:stability:SGT2-QL:2}
For any $\beta>0$, the origin is the unique globally asymptotically stable equilibrium point of the original system~\eqref{eq:appendix:SGT2-QL:original-system2}.
Equivalently, $\left[ {\begin{array}{*{20}{c}}
{{Q^*}}\\
{{Q^*}}
\end{array}} \right]$ is the unique globally asymptotically stable equilibrium point of the original system~\eqref{eq:appendix:SGT2-QL:original-system1}.
\end{theorem}
\begin{proof}
By~\cref{thm:appendix:SGT2-QL:original-lower,thm:appendix:SGT2-QL:upper-original}, we have
\[\left[ {\begin{array}{*{20}{c}}
{Q_t^{A,u} - {Q^*}}\\
{Q_t^{B,u} - {Q^*}}
\end{array}} \right] \ge \left[ {\begin{array}{*{20}{c}}
{Q_t^A - {Q^*}}\\
{Q_t^B - {Q^*}}
\end{array}} \right] \ge \left[ {\begin{array}{*{20}{c}}
{Q_t^{A,l} - {Q^*}}\\
{Q_t^{B,l} - {Q^*}}
\end{array}} \right],\quad \forall t \ge 0.\]
Moreover, by~\cref{thm:appendix:SGT2-QL:stability:lower} and~\cref{thm:appendix:SGT2-QL:stability:upper}, we have
\[\left[ {\begin{array}{*{20}{c}}
{Q_t^{A,u} - {Q^*}}\\
{Q_t^{B,u} - {Q^*}}
\end{array}} \right] \to 0,\quad \left[ {\begin{array}{*{20}{c}}
{Q_t^{A,l} - {Q^*}}\\
{Q_t^{B,l} - {Q^*}}
\end{array}} \right] \to 0\]
as $t\to\infty$. Therefore, the state of the original system also asymptotically converges to the origin. This completes the proof.
\end{proof}

\subsection{Numerical example}
In this subsection, we provide a simple example to illustrate the validity of the properties of the upper and lower comparison systems established in the previous sections.
Let us consider the MDP previously considered for AGT2-QL with ${\cal S}=\{1,2\}$, ${\cal A}=\{1,2\}$, $\gamma = 0.9$,
\begin{align*}
&P_1=\begin{bmatrix}
   0.2 & 0.8\\
   0.3 & 0.7 \\
\end{bmatrix},\quad P_2=\begin{bmatrix}
   0.5 & 0.5 \\
   0.7 & 0.3\\
\end{bmatrix},
\end{align*}
a behavior policy $\beta$ such that
\begin{align*}
&{\mathbb P}[a = 1|s = 1] = 0.2,\quad {\mathbb P}[a = 2|s = 1] = 0.8,\\
&{\mathbb P}[a = 1|s = 2] = 0.7,\quad {\mathbb P}[a = 2|s = 2] = 0.3,
\end{align*}
and the expected reward vectors
\begin{align*}
R_1  = \begin{bmatrix}
   3  \\
   1  \\
\end{bmatrix},\quad R_2  = \begin{bmatrix}
   2  \\
   1  \\
\end{bmatrix}
\end{align*}

Solutions of the O.D.E. model of SGT2-QL and the upper and lower comparison systems are depicted in~\cref{fig:appendix:SGT2-QL:1} for $Q^A_t$ part and~\cref{fig:appendix:SGT2-QL:2} for $Q^B_t$ part.
\begin{figure}[h!]
\centering\includegraphics[width=14cm,height=10cm]{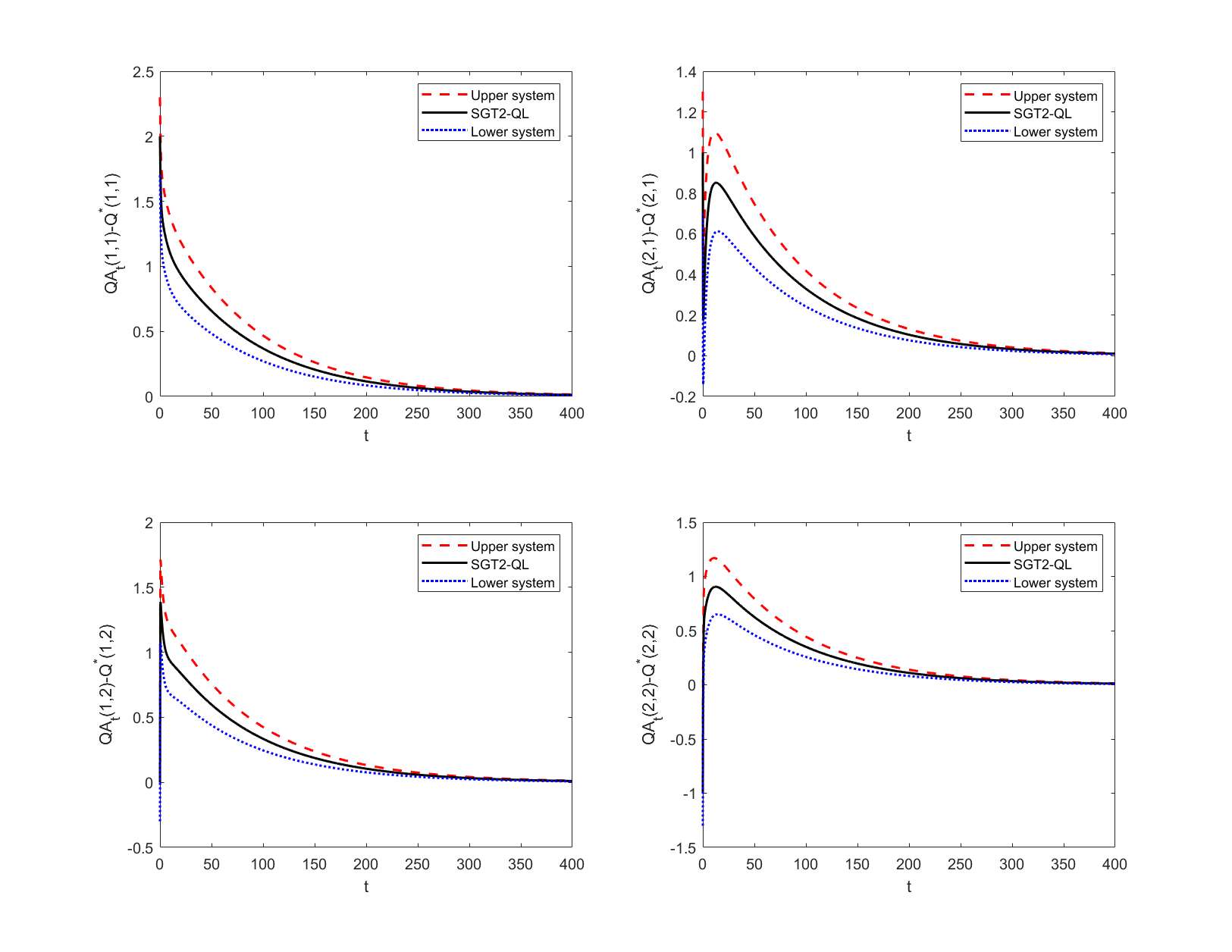}
\caption{Trajectories of the O.D.E. model of SGT2-QL and the corresponding upper and lower comparison systems ($Q^A_t$ part). he solution of the ODE model (black line) is upper and lower bounded by the upper and lower comparison systems, respectively (red and blue lines, respectively). This result provides numerical evidence that the bounding rules hold.}\label{fig:appendix:SGT2-QL:1}
\end{figure}
\begin{figure}[h!]
\centering\includegraphics[width=14cm,height=10cm]{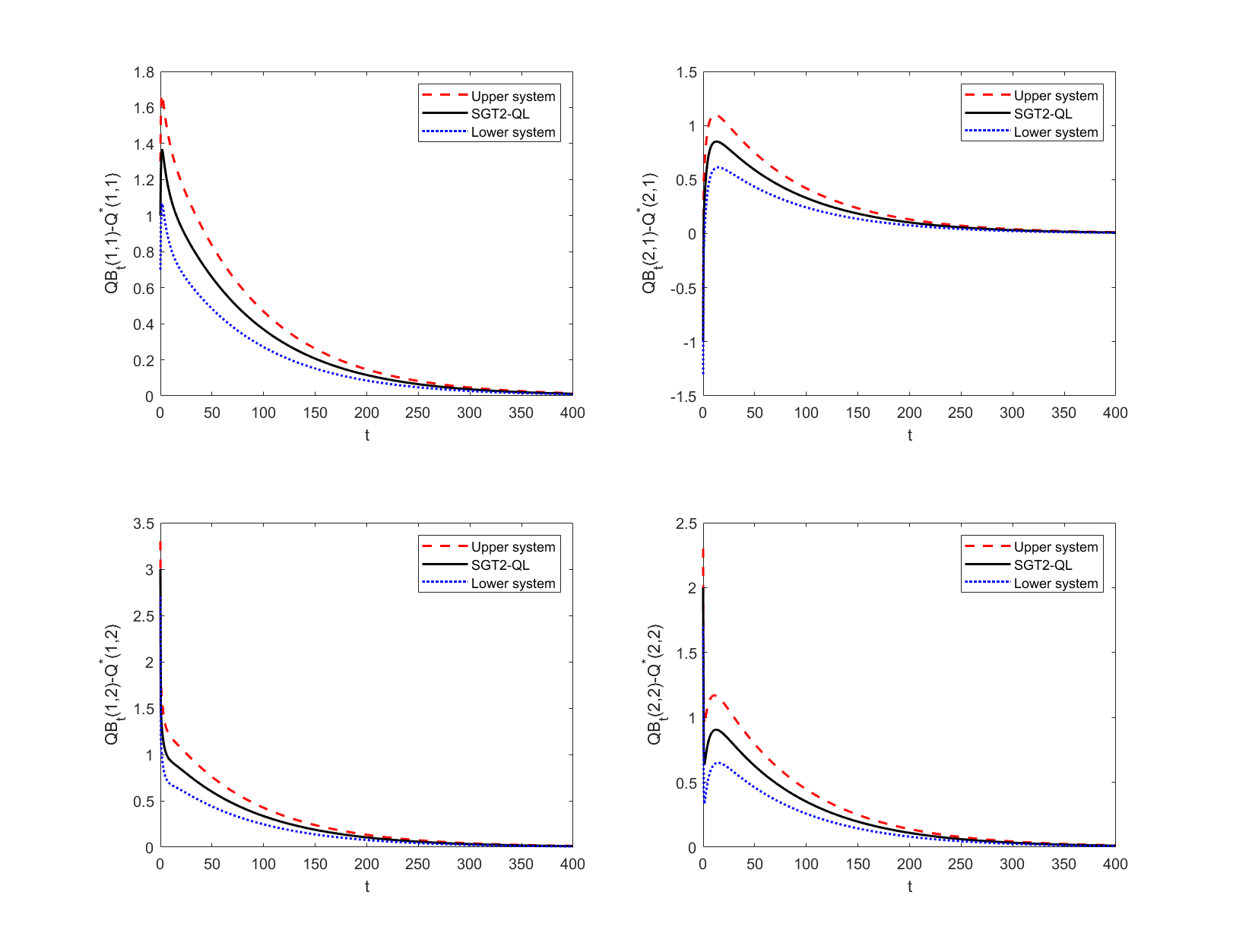}
\caption{Trajectories of the original O.D.E. model of SGT2-QL and the corresponding upper and lower comparison systems ($Q^B_t$ part). he solution of the ODE model (black line) is upper and lower bounded by the upper and lower comparison systems, respectively (red and blue lines, respectively). This result provides numerical evidence that the bounding rules hold.}\label{fig:appendix:SGT2-QL:2}
\end{figure}

As before, the simulation study empirically proves that the ODE model associated with SGT2-QL is asymptotically stable. Moreover, they illustrate that the solutions of the upper and lower comparison systems bound the solution of the original system, as established by the theory.

So far, we have examined key properties of the ODE model of SGT2-QL and proved its global asymptotic stability, which is essential for establishing the convergence of SGT2-QL. In the next subsection, we will analyze its convergence in detail.

\subsection{Convergence of SGT2-QL}
In this subsection, we apply the Borkar and Meyn theorem,~\cref{lemma:Borkar}, to prove the convergence of SGT2-QL in~\cref{thm:SGT2-QL-convergence}. First of all, note that the system in~\eqref{eq:appendix:4} corresponds to the ODE model in~\cref{assumption:1}. The proof is completed by examining all the statements in~\cref{assumption:1}:

\subsubsection{Step 1:}
SGT2-QL can be expressed as the stochastic recursion in~\eqref{eq:general-stochastic-recursion} with
\begin{align*}
f\left( {\left[ {\begin{array}{*{20}{c}}
{{\theta _1}}\\
{{\theta _2}}
\end{array}} \right]} \right): = \left[ {\begin{array}{*{20}{c}}
{ - (1 + \beta )D}&{\gamma DP{\Pi _{{\theta _2}}} + \beta D}\\
{\gamma DP{\Pi _{{\theta _1}}} + \beta D}&{ - (1 + \beta )D}
\end{array}} \right]\left[ {\begin{array}{*{20}{c}}
{{\theta _1}}\\
{{\theta _2}}
\end{array}} \right] + \left[ {\begin{array}{*{20}{c}}
{DR}\\
{DR}
\end{array}} \right].
\end{align*}
Moreover, $f$ is globally Lipschitz continuous according to~\cref{thm:appendix:SGT2-QL:Lipschits:f}.

To prove the first statement of~\cref{assumption:1}, we note that
\begin{align*}
{f_\infty }\left( {\left[ {\begin{array}{*{20}{c}}
{{\theta _1}}\\
{{\theta _2}}
\end{array}} \right]} \right) =& \mathop {\lim }\limits_{c \to \infty } \frac{{f\left( {\left[ {\begin{array}{*{20}{c}}
{c{\theta _1}}\\
{c{\theta _2}}
\end{array}} \right]} \right)}}{c}\\
=& \mathop {\lim }\limits_{c \to \infty } \frac{{\left[ {\begin{array}{*{20}{c}}
{ - (1 + \beta )D}&{\gamma DP{\Pi _{{c\theta _2}}} + \beta D}\\
{\gamma DP{\Pi _{{c\theta _1}}} + \beta D}&{ - (1 + \beta )D}
\end{array}} \right]\left[ {\begin{array}{*{20}{c}}
{{\theta _1}}\\
{{\theta _2}}
\end{array}} \right] + \left[ {\begin{array}{*{20}{c}}
{DR}\\
{DR}
\end{array}} \right]}}{c}\\
=& \mathop {\lim }\limits_{c \to \infty } \left[ {\begin{array}{*{20}{c}}
{ - (1 + \beta )D}&{\gamma DP{\Pi _{c{\theta _2}}} + \beta D}\\
{\gamma DP{\Pi _{c{\theta _1}}} + \beta D}&{ - (1 + \beta )D}
\end{array}} \right]\left[ {\begin{array}{*{20}{c}}
{{\theta _1}}\\
{{\theta _2}}
\end{array}} \right]\\
=& \left[ {\begin{array}{*{20}{c}}
{ - (1 + \beta )D}&{\gamma DP{\Pi _{{\theta _2}}} + \beta D}\\
{\gamma DP{\Pi _{{\theta _1}}} + \beta D}&{ - (1 + \beta )D}
\end{array}} \right]\left[ {\begin{array}{*{20}{c}}
{{\theta _1}}\\
{{\theta _2}}
\end{array}} \right]
\end{align*}
where the last equality is due to the homogeneity of the policy, ${\pi _{c\theta_i }}(s) = \arg {\max _{a \in A}}c\theta_i (s,a) = \arg {\max _{a \in A}}\theta_i (s,a)$ for $i\in \{1,2\}$, where $\theta_i (s,a)$ denotes the entry in the parameter vector $\theta_i$ corresponding to the state-action pair $(s,a)\in {\cal S}\times {\cal A}$.

\subsubsection{Step 2:}
Let us consider the system
\begin{align*}
\frac{d}{{dt}}\left[ {\begin{array}{*{20}{c}}
{{\theta _{1,t}}}\\
{{\theta _{2,t}}}
\end{array}} \right] = {f_\infty }\left( {\left[ {\begin{array}{*{20}{c}}
{{\theta _{1,t}}}\\
{{\theta _{2,t}}}
\end{array}} \right]} \right) = \left[ {\begin{array}{*{20}{c}}
{ - (1 + \beta )D}&{\gamma DP{\Pi _{{\theta _{2,t}}}} + \beta D}\\
{\gamma DP{\Pi _{{\theta _{1,t}}}} + \beta D}&{ - (1 + \beta )D}
\end{array}} \right]\left[ {\begin{array}{*{20}{c}}
{{\theta _{1,t}}}\\
{{\theta _{2,t}}}
\end{array}} \right].
\end{align*}
Its global asymptotic stability around the origin can be easily proved following similar lines as in the proof of the upper comparison system in~\cref{thm:appendix:AGT2-QL:stability:upper}.

\subsubsection{Step 3:}
According to~\cref{thm:appendix:AGT2-QL:stability:f}, the ODE model of SGT2-QL is globally asymptotically stable around $\left[ {\begin{array}{*{20}{c}}
{{\theta _1}}\\
{{\theta _2}}
\end{array}} \right] = \left[ {\begin{array}{*{20}{c}}
{{Q^*}}\\
{{Q^*}}
\end{array}} \right].$

\subsubsection{Step 4:}
Next, we prove the remaining parts. Recall that SGT2-QL is expressed as
\begin{align*}
{{\bar Q}_{k + 1}} = {{\bar Q}_k} + {\alpha _k}\{ f({{\bar Q}_k}) + {\varepsilon _{k + 1}}\}
\end{align*}
where ${{\bar Q}_k}: = \left[ {\begin{array}{*{20}{c}}
{Q_k^A}\\
{Q_k^B}
\end{array}} \right]$
and
\begin{align*}
{\varepsilon _{k + 1}}: =& \left[ {\begin{array}{*{20}{c}}
{ - (1 + \beta )({e_s} \otimes {e_a}){{({e_s} \otimes {e_a})}^T}}&{\gamma ({e_s} \otimes {e_a})e_{s'}^T{\Pi _{{Q^B}}} + \beta ({e_s} \otimes {e_a}){{({e_s} \otimes {e_a})}^T}}\\
{\gamma ({e_s} \otimes {e_a})e_{s'}^T{\Pi _{{Q^A}}} + \beta ({e_s} \otimes {e_a}){{({e_s} \otimes {e_a})}^T}}&{ - (1 + \beta )({e_s} \otimes {e_a}){{({e_s} \otimes {e_a})}^T}}
\end{array}} \right]\left[ {\begin{array}{*{20}{c}}
{{Q^A}}\\
{{Q^B}}
\end{array}} \right]\\
& + \left[ {\begin{array}{*{20}{c}}
{({e_s} \otimes {e_a}){{({e_s} \otimes {e_a})}^T}r(s,a,s')}\\
{({e_s} \otimes {e_a}){{({e_s} \otimes {e_a})}^T}r(s,a,s')}
\end{array}} \right] - \left[ {\begin{array}{*{20}{c}}
{ - (1 + \beta )D}&{\gamma DP{\Pi _{{Q^B}}} + \beta D}\\
{\gamma DP{\Pi _{{Q^A}}} + \beta D}&{ - (1 + \beta )D}
\end{array}} \right]\left[ {\begin{array}{*{20}{c}}
{{Q^A}}\\
{{Q^B}}
\end{array}} \right] + \left[ {\begin{array}{*{20}{c}}
R\\
R
\end{array}} \right]
\end{align*}
and
\begin{align*}
f\left( {\left[ {\begin{array}{*{20}{c}}
{{Q^A}}\\
{{Q^B}}
\end{array}} \right]} \right): = \left[ {\begin{array}{*{20}{c}}
{ - (1 + \beta )D}&{\gamma DP{\Pi _{{Q^B}}} + \beta D}\\
{\gamma DP{\Pi _{{Q^A}}} + \beta D}&{ - (1 + \beta )D}
\end{array}} \right]\left[ {\begin{array}{*{20}{c}}
{{Q^A}}\\
{{Q^B}}
\end{array}} \right] + \left[ {\begin{array}{*{20}{c}}
{DR}\\
{DR}
\end{array}} \right].
\end{align*}

To proceed, let us define the history
\[{\cal G}_k: = ({{\bar \varepsilon }_k},{{\bar \varepsilon }_{k - 1}}, \ldots ,{{\bar \varepsilon }_1},{{\bar Q}_k},{{\bar Q}_{k - 1}}, \ldots ,{{\bar Q}_0})\]
Moreover, let us define the corresponding process $(M_k)_{k=0}^\infty$ with $M_k:=\sum_{i=1}^k {\varepsilon_i}$. Then, we can prove that $(M_k)_{k=0}^\infty$ is Martingale. To do so, we can easily prove ${\mathbb E}[\varepsilon_{k+1}|{\cal G}_k]=0$. Using this identity, we have
\begin{align*}
{\mathbb E}[M_{k+1}|{\cal G}_k]=& {\mathbb E}\left[ \left. \sum_{i=1}^{k+1}{\varepsilon_i} \right|{\cal G}_k\right]={\mathbb E}[\varepsilon_{k+1}|{\cal G}_k]+{\mathbb E}\left[ \left. \sum_{i=1}^k {\varepsilon_i} \right|{\cal G}_k \right]\\
=&{\mathbb E}\left[\left.\sum_{i=1}^k{\varepsilon_i} \right|{\cal G}_k \right]=\sum_{i=1}^k {\varepsilon_i}=M_k.
\end{align*}
Therefore, $(M_k)_{k=0}^\infty$ is a Martingale sequence, and $\varepsilon_{k+1} = M_{k+1}-M_k$ is a Martingale difference. Moreover, it can be easily proved that the fourth condition of~\cref{assumption:1} is satisfied by algebraic calculations. Therefore, the fourth condition is met.

\section{Experiments on AGT2-QL and SGT2-QL}
In this section, we empirically demonstrate the convergence of the proposed AGT2-QL and SGT2-QL algorithms. 

First of all, we consider the simple MDP considered in the previous sections with ${\cal S}=\{1,2\}$, ${\cal A}=\{1,2\}$, $\gamma = 0.9$,
\begin{align*}
&P_1=\begin{bmatrix}
   0.2 & 0.8\\
   0.3 & 0.7 \\
\end{bmatrix},\quad P_2=\begin{bmatrix}
   0.5 & 0.5 \\
   0.7 & 0.3\\
\end{bmatrix},
\end{align*}
a behavior policy $\beta$ such that
\begin{align*}
&{\mathbb P}[a = 1|s = 1] = 0.2,\quad {\mathbb P}[a = 2|s = 1] = 0.8,\\
&{\mathbb P}[a = 1|s = 2] = 0.7,\quad {\mathbb P}[a = 2|s = 2] = 0.3,
\end{align*}
and the expected reward vectors
\begin{align*}
R_1  = \begin{bmatrix}
   3  \\
   1  \\
\end{bmatrix},\quad R_2  = \begin{bmatrix}
   2  \\
   1  \\
\end{bmatrix}.
\end{align*}

Simulated errors
$|Q_k^A(s,a) - {Q^*}(s,a)|,|Q_k^B(s,a) - {Q^*}(s,a)|,(s,a) \in {\cal S} \times {\cal A}$ of AGT2-QL are depicted in~\cref{fig:appendix:AGT2-QL:Qa} for the $Q^A_k$ part and in~\cref{fig:appendix:AGT2-QL:Qb} for the $Q^B_k$ part. The results are presented for different weight values $\beta \in \{0.01,0.05,0.1,0.2,0.5\}$ and the diminishing step-size $\alpha_k = \frac{80}{200 + k}$. The results demonstrate the convergence of AGT2-QL to the optimal $Q^*$ over time. Moreover, the results illustrate the convergence trends for varying weight $\beta$. In particular, the larger the $\beta$, the faster the convergence.
\begin{figure}[h!]
\centering\includegraphics[width=14cm,height=11cm]{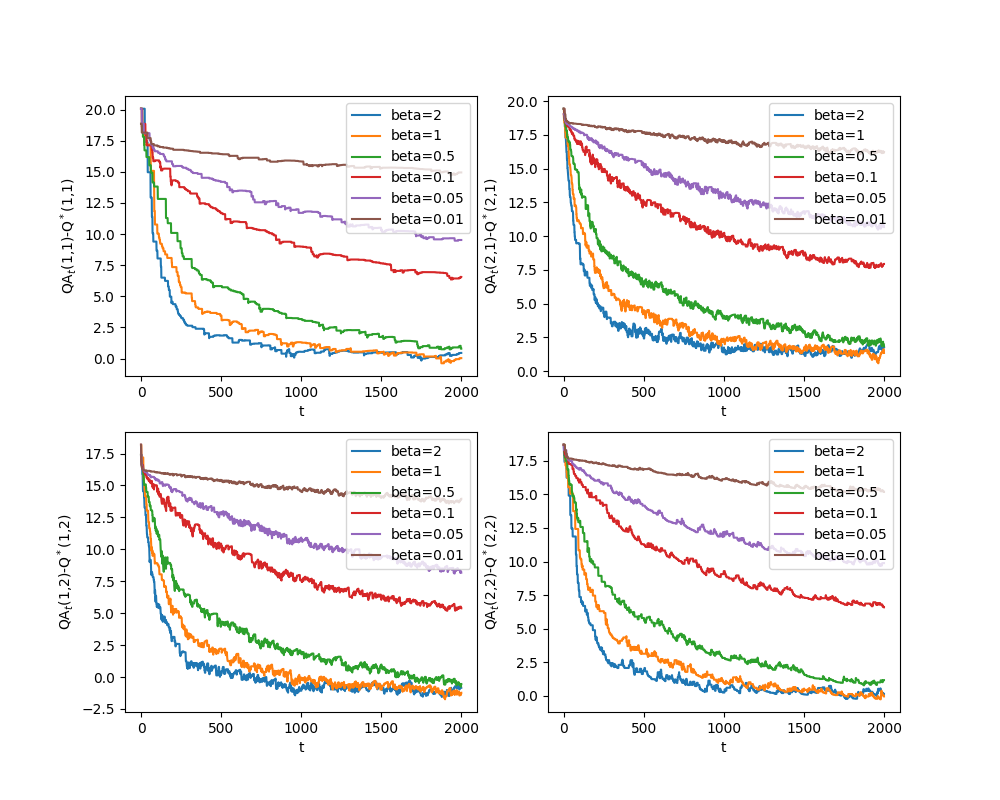}
\caption{Simulated errors of AGT2-QL ($Q^A_k$ part)}\label{fig:appendix:AGT2-QL:Qa}
\end{figure}
\begin{figure}[h!]
\centering\includegraphics[width=14cm,height=11cm]{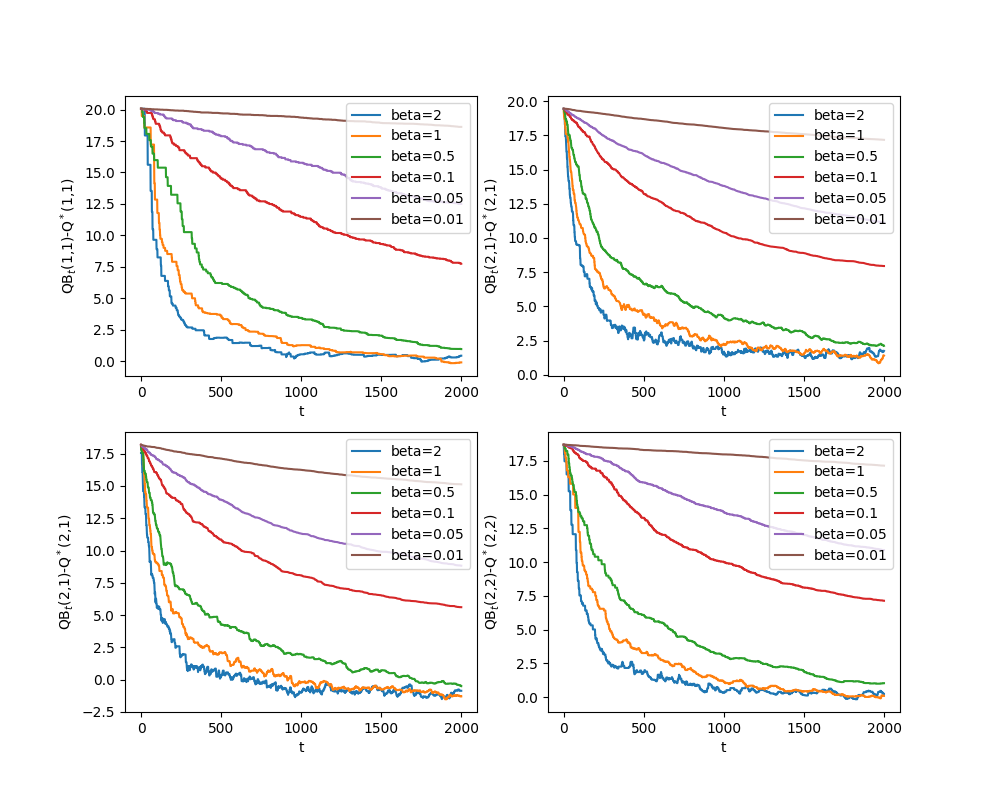}
\caption{Simulated errors of AGT2-QL ($Q^B_k$ part)}\label{fig:appendix:AGT2-QL:Qb}
\end{figure}

Similarly, for the same environment, simulated errors
$|Q_k^A(s,a) - {Q^*}(s,a)|,|Q_k^B(s,a) - {Q^*}(s,a)|,(s,a) \in S \times A$ of AGT2-QL are depicted in~\cref{fig:appendix:SGT2-QL:Qa} for the $Q^A_k$ part and in~\cref{fig:appendix:SGT2-QL:Qb} for the $Q^B_k$ part with different weight values $\beta \in \{0.01,0.05,0.1,0.2,0.5\}$. The results also demonstrate the convergence, but with different convergence trends for varying weight $\beta$. In particular, the error evolutions show that SGT2-QL generally converges faster than AGT2-QL, and the convergence speeds are less sensitive to $\beta$.
\begin{figure}[h!]
\centering\includegraphics[width=14cm,height=11cm]{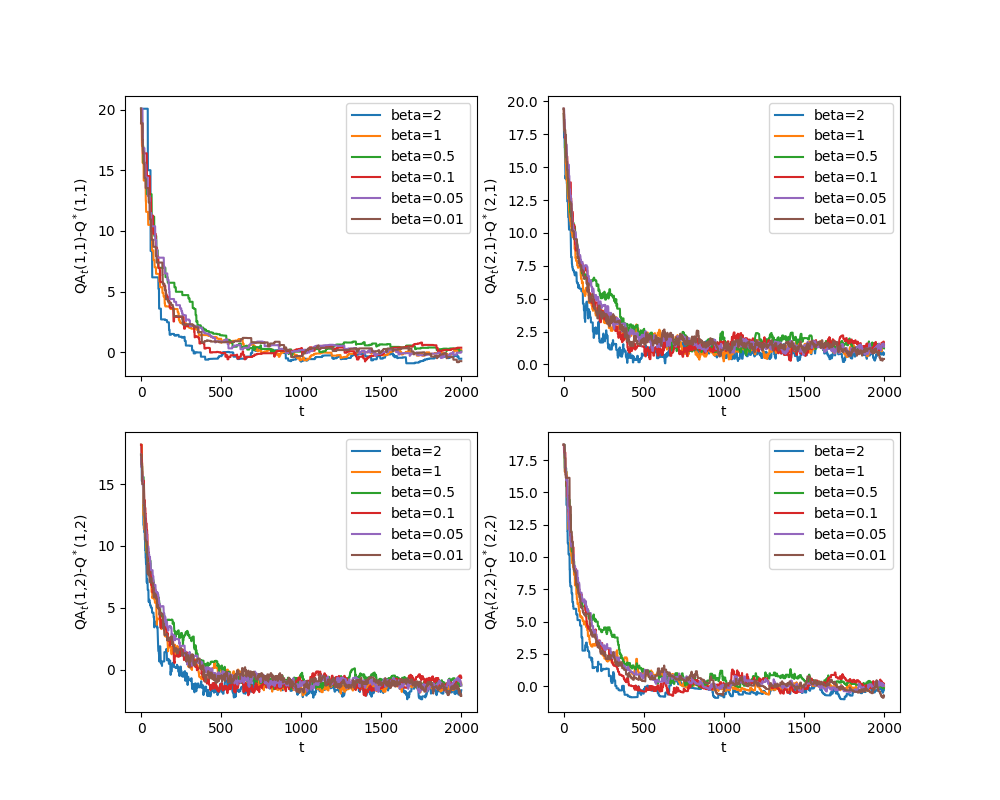}
\caption{Simulated errors of SGT2-QL ($Q^A_k$ part)}\label{fig:appendix:SGT2-QL:Qa}
\end{figure}
\begin{figure}[h!]
\centering\includegraphics[width=14cm,height=11cm]{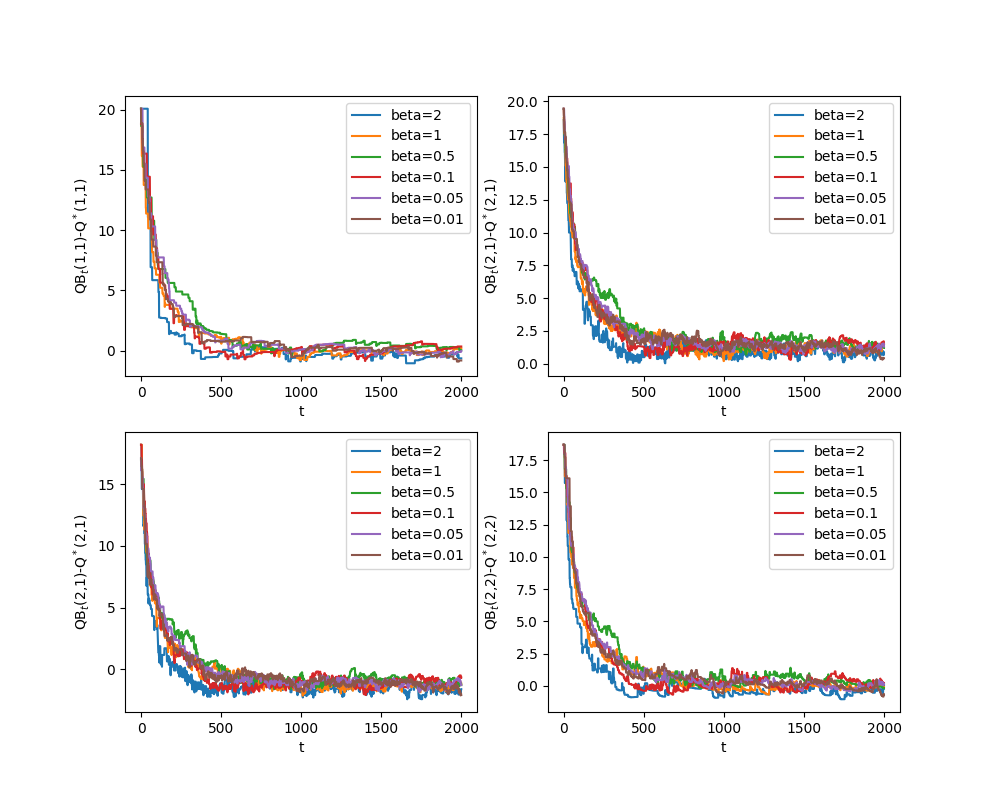}
\caption{Simulated errors of SGT2-QL ($Q^B_k$ part)}\label{fig:appendix:SGT2-QL:Qb}
\end{figure}

Additionally, we conduct experiments using grid world environments, including Taxi, FrozenLake, and CliffWalk, in OpenAI Gym. \cref{fig:appendix:comparison1} presents reward curves for the Taxi environment, which shows trends similar to previous results: AGT2-QL exhibits better learning performance for higher values of $\beta$, whereas SGT2-QL is not significantly affected by the choice of $\beta$.
\begin{figure}[h!]
\centering
\subfigure[Reward curves of AGT2-QL with different $\beta$]{\includegraphics[width=5cm,height=4cm]{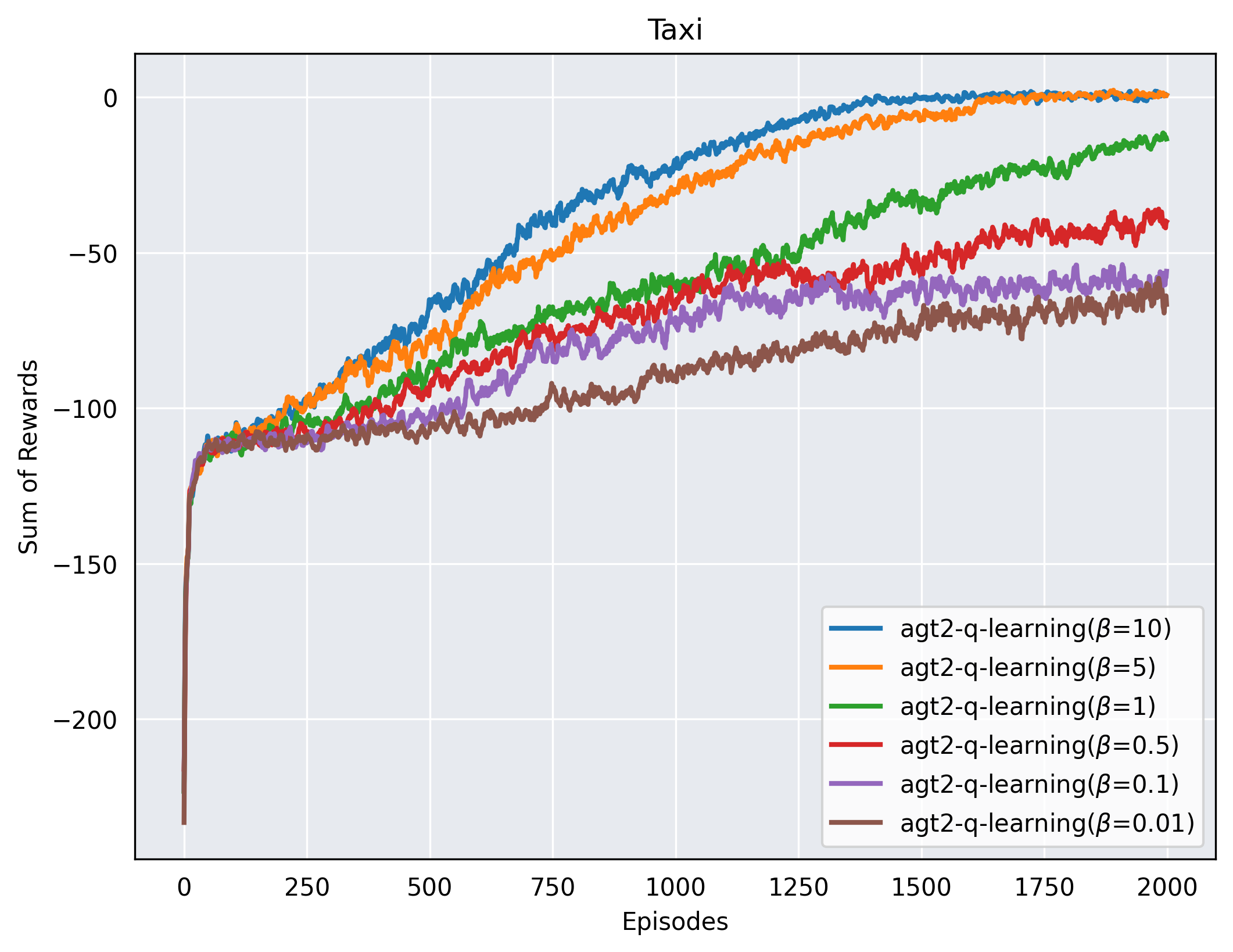}}
\subfigure[Reward curves of SGT2-QL with different $\beta$]{\includegraphics[width=5cm,height=4cm]{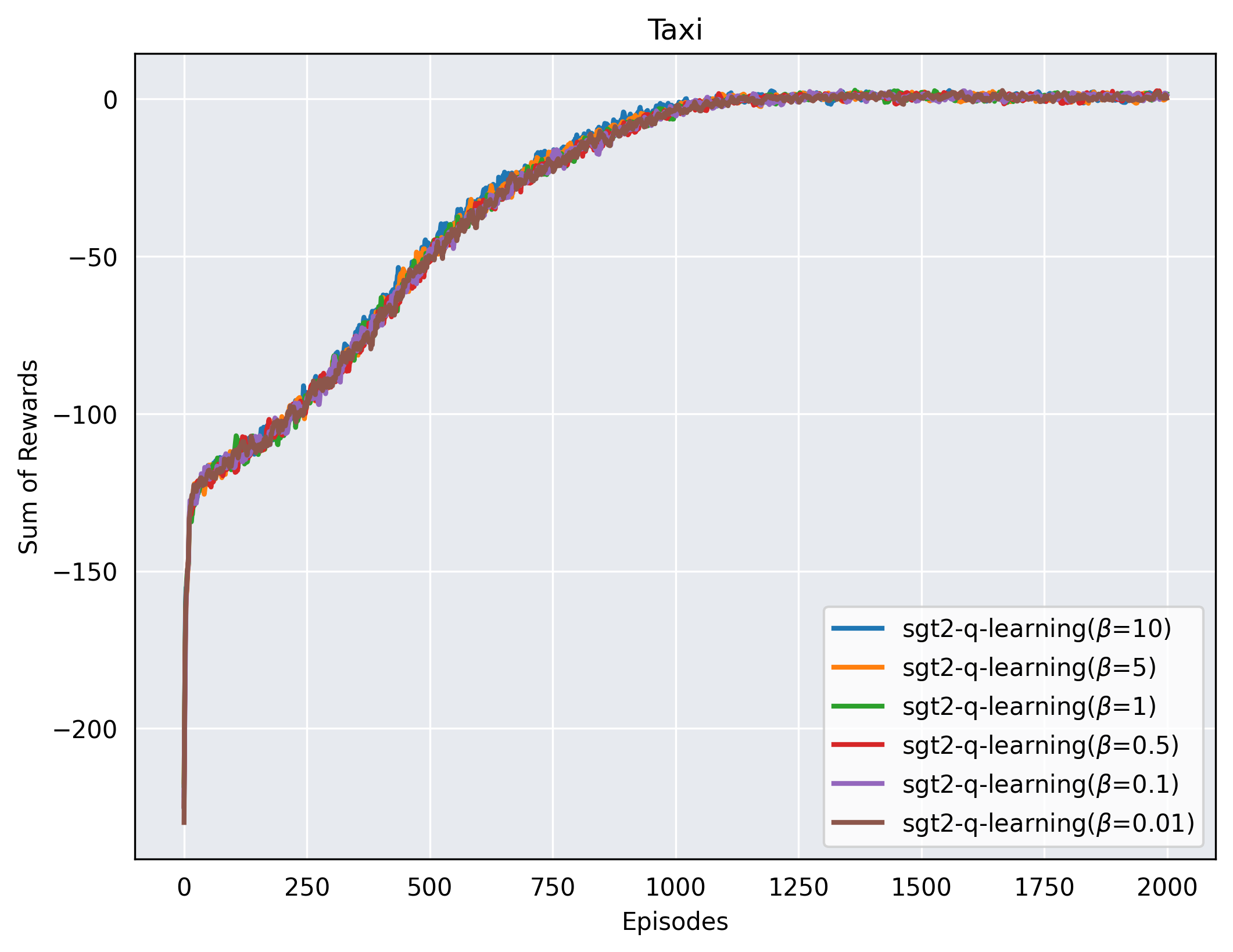}}\\
\subfigure[Reward curves of AGT2-QL, SGT2-QL, Q-learning, and double Q-learning]{\includegraphics[width=5cm,height=4cm]{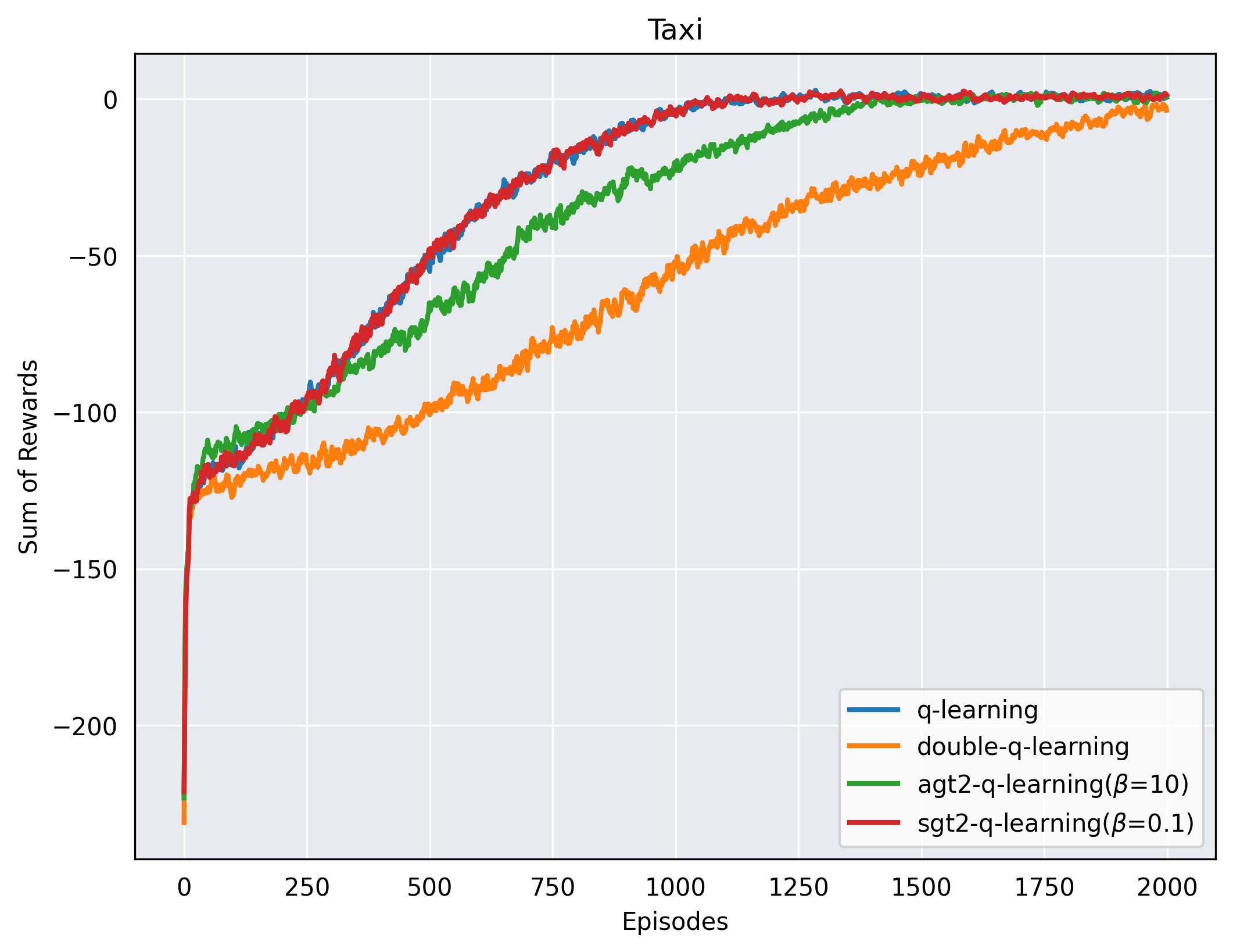}}
\caption{Comparison of simulated reward curves of AGT2-QL, SGT2-QL, Q-learning, double Q-learning in Taxi environment}\label{fig:appendix:comparison1}
\end{figure}

Similarly,\cref{fig:appendix:comparison2} presents reward curves for the FrozenLake environment, while\cref{fig:appendix:comparison3} shows reward curves for the CliffWalk environment. All results demonstrate that the proposed approaches, AGT2-QL and SGT2-QL, are valid and effectively learn the optimal policy. Moreover, the learning speeds of AGT2-QL and SGT2-QL are comparable to that of standard Q-learning.
\begin{figure}[h!]
\centering
\subfigure[Reward curves of AGT2-QL with different $\beta$]{\includegraphics[width=5cm,height=4cm]{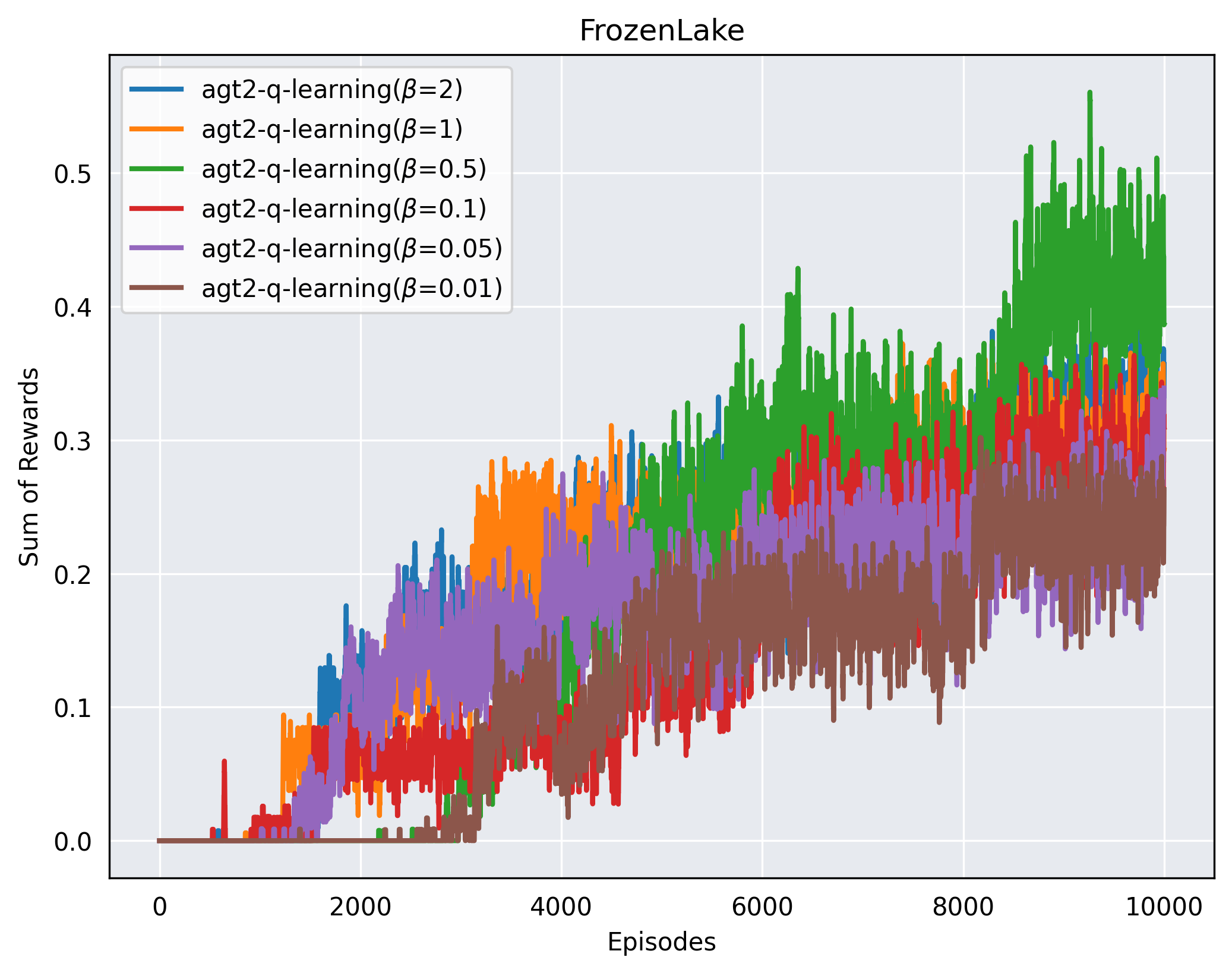}}
\subfigure[Reward curves of SGT2-QL with different $\beta$]{\includegraphics[width=5cm,height=4cm]{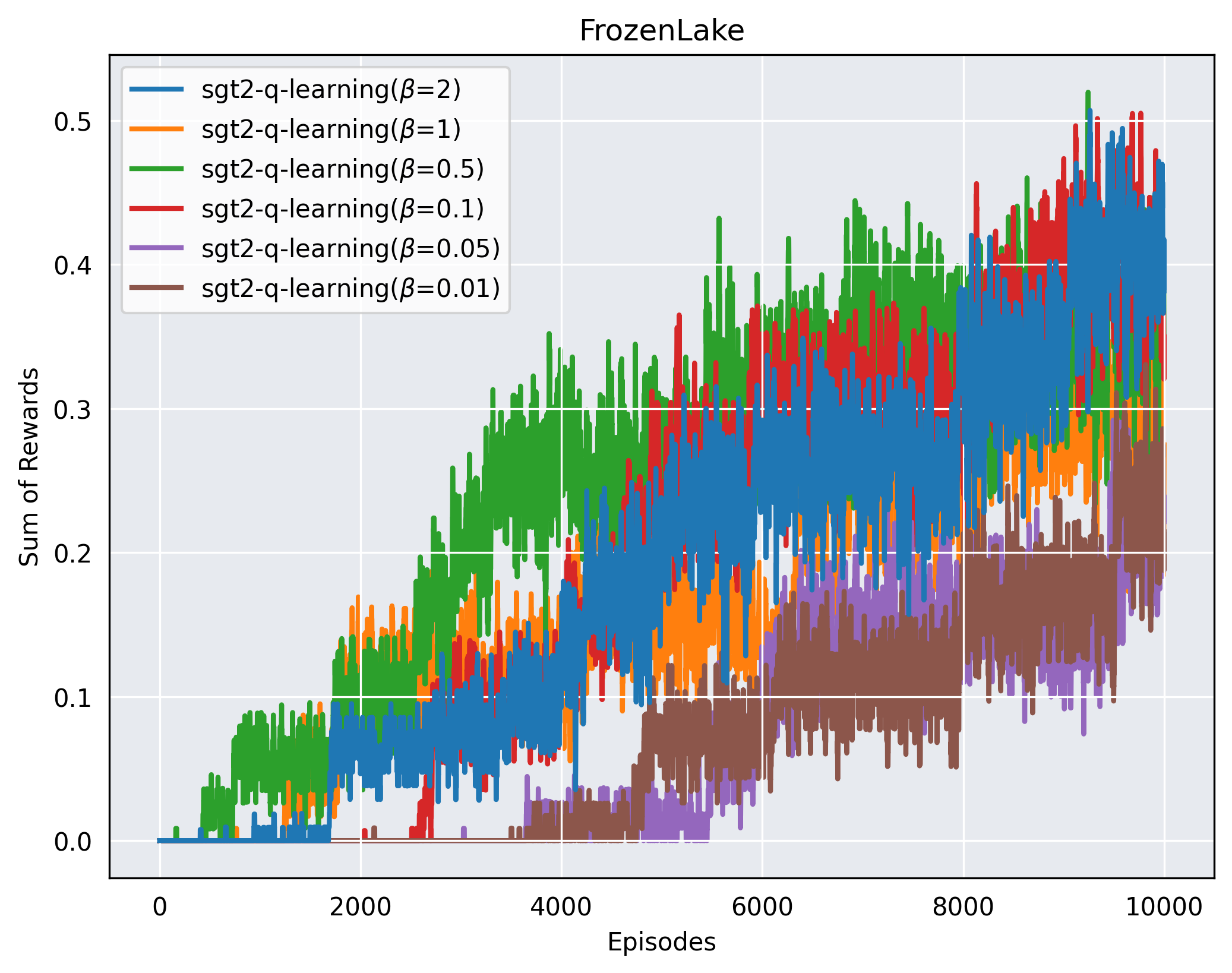}}\\
\subfigure[Reward curves of AGT2-QL, SGT2-QL, Q-learning, and double Q-learning]{\includegraphics[width=5cm,height=4cm]{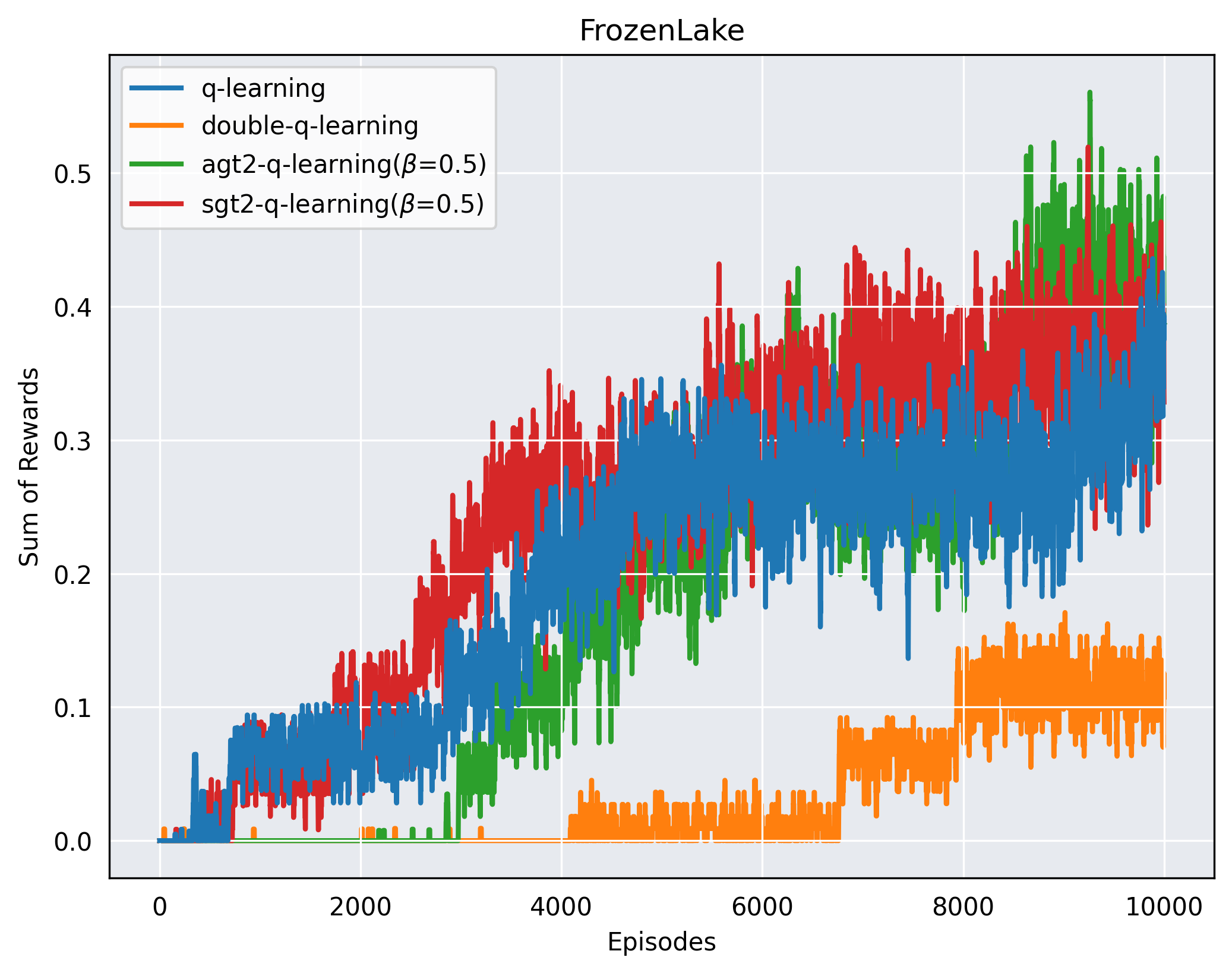}}
\caption{Comparison of reward curves of AGT2-QL, SGT2-QL, Q-learning, double Q-learning in Frozenlake environment}\label{fig:appendix:comparison2}
\end{figure}
\begin{figure}[h!]
\centering
\subfigure[]{\includegraphics[width=5cm,height=4cm]{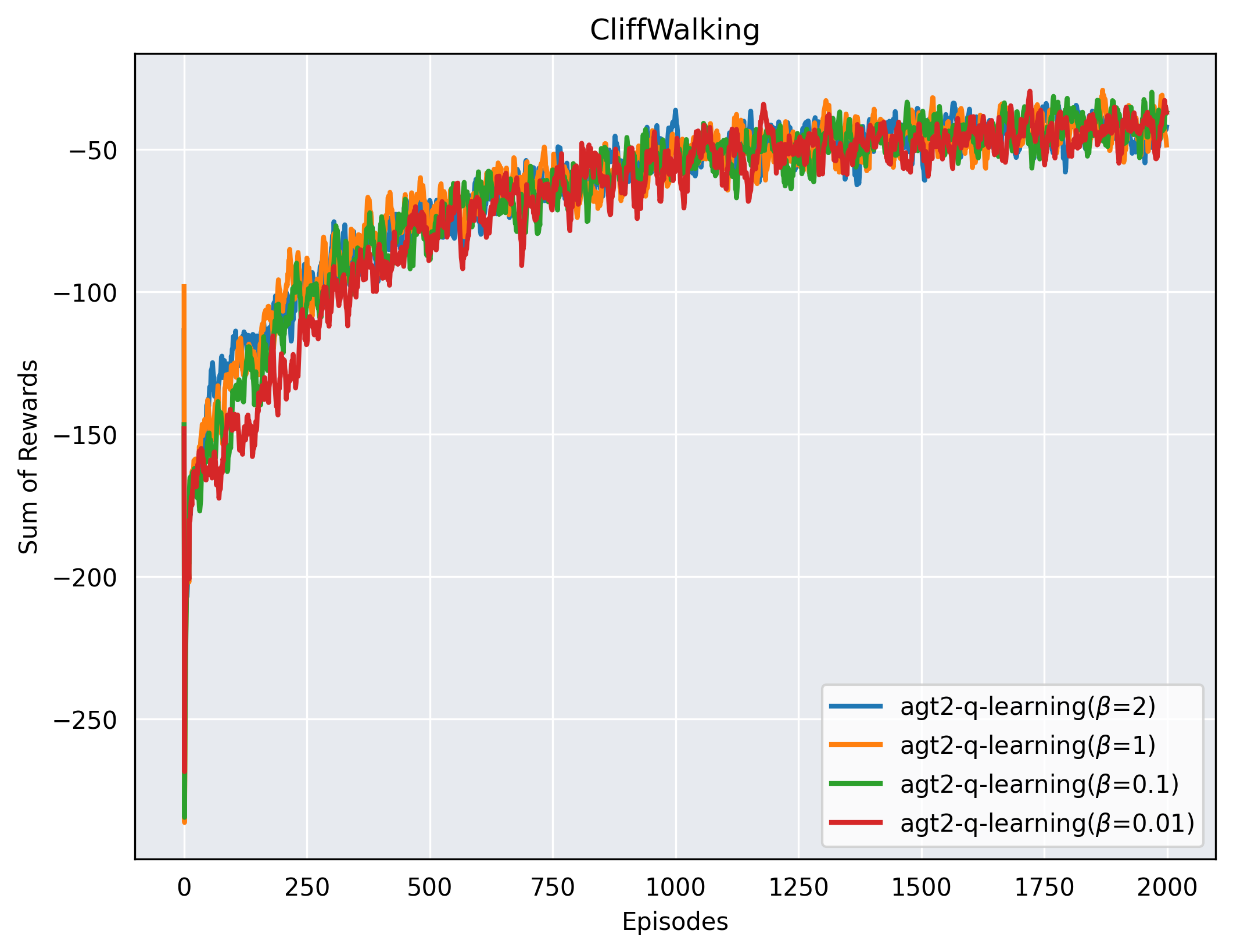}}
\subfigure[]{\includegraphics[width=5cm,height=4cm]{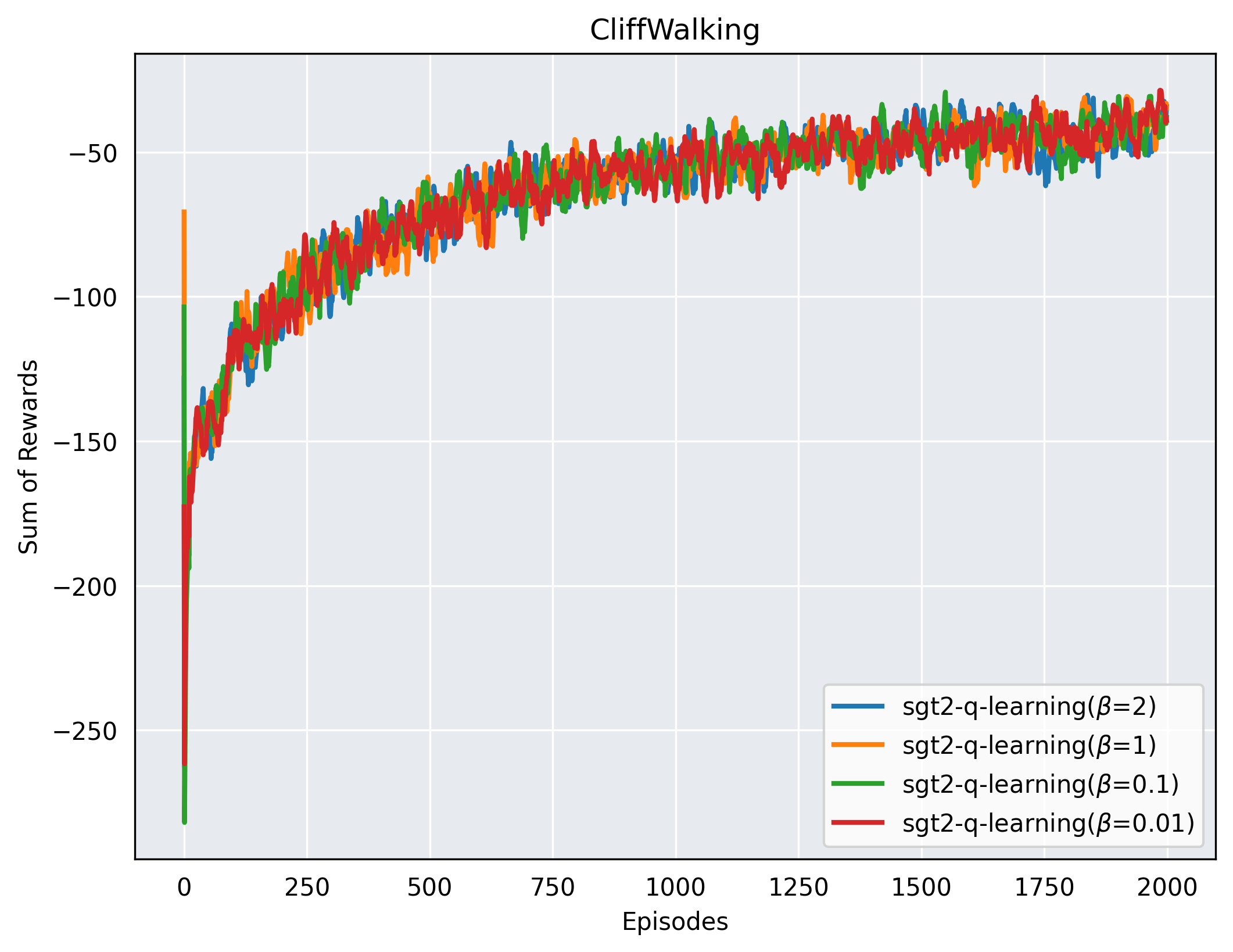}}\\
\subfigure[]{\includegraphics[width=5cm,height=4cm]{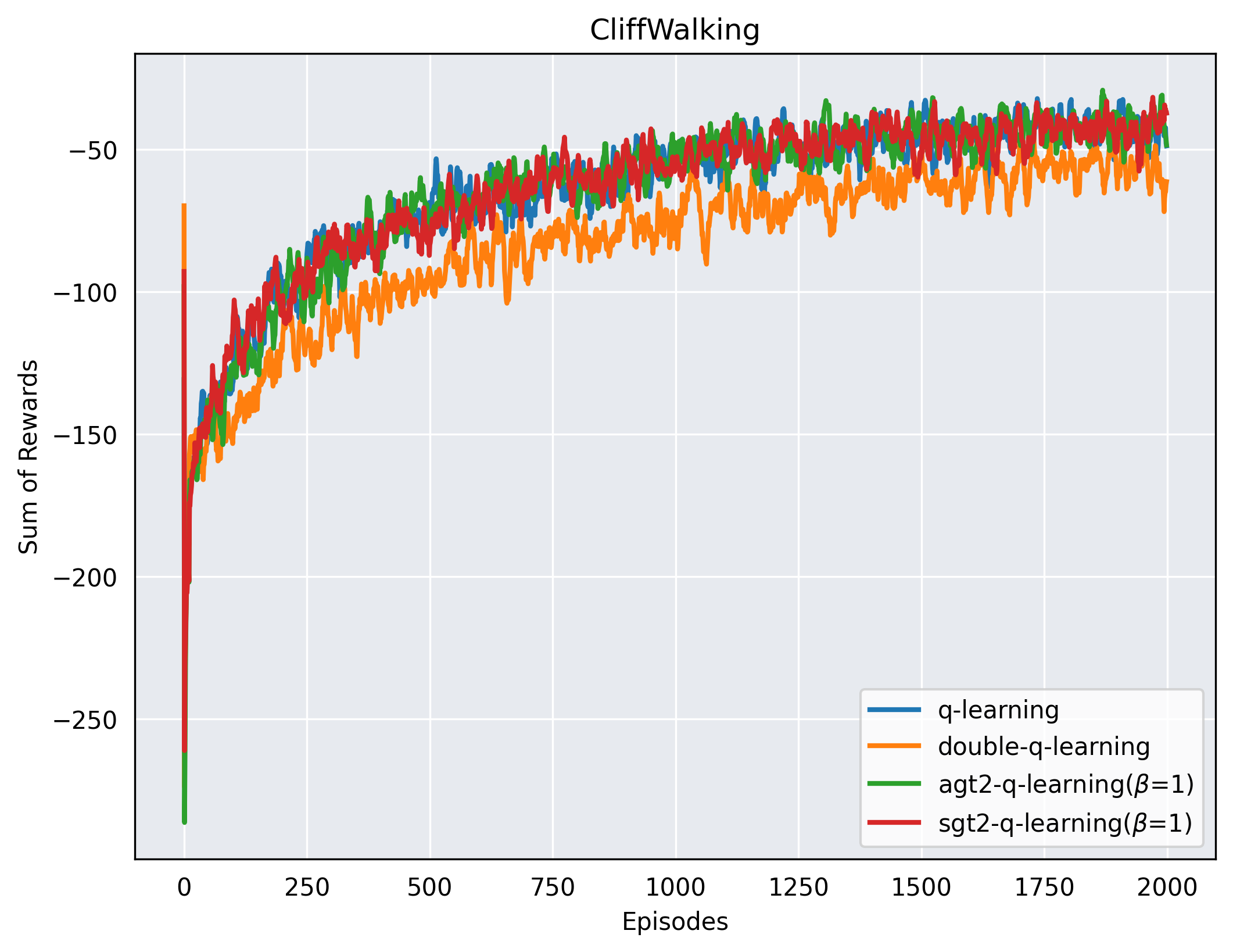}}
\caption{Comparison of simulated reward curves of AGT2-QL, SGT2-QL, Q-learning, double Q-learning in CliffWalk environment}\label{fig:appendix:comparison3}
\end{figure}

\clearpage
\section{Proof of~\cref{thm:AGT2-DQN:optimality}}
Let us assume that by minimizing the loss functions of AGT2-DQN, we can approximately minimize the following expected loss function
\begin{align*}
{L_1}({\theta _1}) =& {\mathbb E}_{(s,a) \sim U({\cal S} \times {\cal A}),s' \sim P( \cdot |s,a)}\left[ {{{\left( {r(s,a,s') + \gamma {{\max }_{a \in {\cal A}}}{Q_{{\theta _2}}}(s',a) - {Q_{{\theta _1}}}(s,a)} \right)}^2}} \right]\\
{L_2}({\theta _2}) =& {\mathbb E}_{(s,a) \sim U({\cal S} \times {\cal A}),s' \sim P( \cdot |s,a)}\left[ {\frac{\beta }{2}{{({Q_{{\theta _1}}}(s,a) - {Q_{{\theta _2}}}(s,a))}^2}} \right]
\end{align*}
where $U({\cal S} \times {\cal A})$ means the uniform distribution over the set ${\cal S}\times {\cal A}$. Let us suppose that the loss functions are minimized with the error
\[{L_1}({\theta _1}) \le \varepsilon ,\quad {L_2}({\theta _2}) \le \varepsilon \]
To proceed further, note that using the law of iterated expectations, the expected loss functions can be written by
\[{L_1}({\theta _1}) = {\mathbb E}_{(s,a) \sim U({\cal S} \times {\cal A})}\left[ {{\mathbb E}_{s' \sim P( \cdot |s,a)}\left[ {\left. {{{\left( {r(s,a,s') + \gamma {{\max }_{a \in {\cal A}}}{Q_{{\theta _2}}}(s',a) - {Q_{{\theta _1}}}(s,a)} \right)}^2}} \right|s,a} \right]} \right]\]
and
\[{L_2}({\theta _2}) = {\mathbb E}_{(s,a) \sim U({\cal S} \times {\cal A})}\left[ {{\mathbb E}_{s' \sim P( \cdot |s,a)}\left[ {\left. {\frac{\beta }{2}{{(Q_{\theta_1}(s,a) - Q_{\theta _2}(s,a))}^2}} \right|s,a} \right]} \right]\]

For convenience, let us define the Bellman operator
\[(TQ)(s,a): = R(s,a) + \gamma \sum\limits_{s' \in {\cal S}} {P(s'|s,a)} {\max _{a \in A}}Q(s',a)\]

Using Jensen's inequality, we can prove that
\begin{align*}
L_1(\theta _1) \ge& {\mathbb E}_{(s,a) \sim U({\cal S} \times {\cal A})}\left[ {{{\left( {{{\mathbb E}_{s' \sim P( \cdot |s,a)}}\left[ {r(s,a,s') + \gamma {{\max }_{a \in A}}Q_{\theta _2}(s',a) - Q_{\theta _1}(s,a)} \right]} \right)}^2}} \right]\\
=& {{\mathbb E}_{(s,a) \sim U({\cal S} \times {\cal A})}}\left[ {{{((T{Q_{{\theta _2}}})(s,a) - {Q_{{\theta _1}}}(s,a))}^2}} \right]\\
\ge &\frac{1}{{|{\cal S}||{\cal A}|}}{\left( {(TQ_{\theta _2})(s,a) - Q_{\theta _1}(s,a)} \right)^2}
\end{align*}
for all $(s,a)\in {\cal S}\times {\cal A}$, which implies
\[\left| {(TQ_{\theta _2})(s,a) - {Q_{\theta _1}}(s,a)} \right| \le \sqrt {\varepsilon |{\cal S}||{\cal A}|},\quad \forall (s,a)\in {\cal S}\times {\cal A} \]
Similarly, for the loss function $L_2(\theta _2)$, we have
\begin{align*}
L_2(\theta _2) = {{\mathbb E}_{(s,a) \sim U({\cal S} \times {\cal A})}}\left[ {\frac{\beta }{2}{{\left( {{Q_{\theta _2}}(s,a) - {Q_{\theta _1}}(s,a)} \right)}^2}} \right]\ge \frac{1}{{|{\cal S}||{\cal A}|}}\frac{\beta }{2}{\left( {{Q_{{\theta _2}}}(s,a) - {Q_{{\theta _1}}}(s,a)} \right)^2}
\end{align*}
for all $(s,a)\in {\cal S}\times {\cal A}$, which implies
\begin{align}
\left| {{Q_{{\theta _2}}}(s,a) - {Q_{{\theta _1}}}(s,a)} \right| \le \sqrt {\frac{{2\varepsilon }}{\beta }|{\cal S}||{\cal A}|},\quad (s,a)\in {\cal S}\times {\cal A}.\label{eq:appendix:Q1-Q2-bound}
\end{align}

Based on the above results, we will establish bounds on $Q_{\theta _1}$ and $Q_{\theta _2}$ in the sequel. 

\subsection{Bound on $Q_{{\theta _1}}$}

First of all, using the reverse triangle inequality leads to
\begin{align}
\left| {|{Q_{{\theta _1}}}(s,a) - {Q^*}(s,a)| - |(T{Q_{{\theta _2}}})(s,a) - (T{Q^*})(s,a)|} \right| \le \sqrt {\varepsilon |{\cal S}||{\cal A}|},\quad (s,a)\in {\cal S}\times {\cal A}. \label{eq:appendix:1}
\end{align}
To proceed, we consider the two cases:
\paragraph{Case 1:} Let us consider the case
\[|{Q_{{\theta _1}}}(s,a) - {Q^*}(s,a)| > |(T{Q_{{\theta _2}}})(s,a) - (T{Q^*})(s,a)|\]
Then, one can derive from~\eqref{eq:appendix:1}
\begin{align*}
&{|{Q_{{\theta _1}}}(s,a) - {Q^*}(s,a)| - |(T{Q_{{\theta _2}}})(s,a) - (T{Q^*})(s,a)|}\\
\ge& |{Q_{{\theta _1}}}(s,a) - {Q^*}(s,a)| - |(T{Q_{{\theta _2}}})(s,a) - (T{Q^*})(s,a)|\\
\ge& |{Q_{{\theta _1}}}(s,a) - {Q^*}(s,a)| - {\left\| {T{Q_{{\theta _2}}} - T{Q^*}} \right\|_\infty }\\
\ge& |{Q_{{\theta _1}}}(s,a) - {Q^*}(s,a)| - \gamma {\left\| {{Q_{{\theta _2}}} - {Q^*}} \right\|_\infty }\\
\ge& |{Q_{{\theta _1}}}(s,a) - {Q^*}(s,a)| - \gamma {\left\| {{Q_{{\theta _1}}} - {Q^*}} \right\|_\infty } - \gamma {\left\| {{Q_{{\theta _2}}} - {Q_{{\theta _1}}}} \right\|_\infty }\\
\ge& |{Q_{{\theta _1}}}(s,a) - {Q^*}(s,a)| - \gamma {\left\| {{Q_{{\theta _1}}} - {Q^*}} \right\|_\infty } - \gamma \sqrt {\frac{{2\varepsilon }}{\beta }|{\cal S}||{\cal A}|},
\end{align*}
where \eqref{eq:appendix:Q1-Q2-bound} is used in the last line. This leads to
\[|{Q_{{\theta _1}}}(s,a) - {Q^*}(s,a)| \le \gamma {\left\| Q_{\theta _1} - Q^* \right\|_\infty } + \sqrt {\varepsilon |{\cal S}||{\cal A}|}  + \gamma \sqrt {\frac{2\varepsilon}{\beta }|{\cal S}||{\cal A}|} \]
Taking the maximum on the left-hand side over $(s,a)\in {\cal S}\times {\cal A}$ yields
\[{\left\| Q_{\theta _1} - Q^* \right\|_\infty } \le \gamma {\left\| Q_{\theta _1} - Q^*\right\|_\infty } + \sqrt {\varepsilon |{\cal S}||{\cal A}|}  + \gamma \sqrt {\frac{2\varepsilon}{\beta }|{\cal S}||{\cal A}|} \]
Arranging terms leads to
\[{\left\| Q_{\theta_1} - Q^* \right\|_\infty } \le \frac{{\sqrt {\varepsilon |{\cal S}||{\cal A}|} }}{1 - \gamma} + \frac{\gamma }{1 - \gamma}\sqrt {\frac{2\varepsilon |{\cal S}||{\cal A}|}{\beta }} \]

\paragraph{Case 2:} Let us consider the case
\[|{Q_{{\theta _1}}}(s,a) - {Q^*}(s,a)| \le |(T{Q_{{\theta _2}}})(s,a) - (T{Q^*})(s,a)|,\]
which leads to
\begin{align*}
|{Q_{{\theta _1}}}(s,a) - {Q^*}(s,a)| \le& |(T{Q_{{\theta _2}}})(s,a) - (T{Q^*})(s,a)|\\
 \le& {\left\| {T{Q_{{\theta _2}}} - T{Q^*}} \right\|_\infty }\\
\le& \gamma {\left\| {{Q_{{\theta _2}}} - {Q^*}} \right\|_\infty }\\
\le& \gamma {\left\| {{Q_{{\theta _1}}} - {Q^*}} \right\|_\infty } + \gamma {\left\| {{Q_{{\theta _2}}} - {Q_{{\theta _1}}}} \right\|_\infty }\\
\le& \gamma {\left\| {{Q_{{\theta _1}}} - {Q^*}} \right\|_\infty } + \gamma \sqrt {\frac{{2\varepsilon }}{\beta }|{\cal S}||{\cal A}|}.
\end{align*}
where the last inequality uses~\eqref{eq:appendix:Q1-Q2-bound}. After simple manipulations, we have 
\[{\left\| {{Q_{{\theta _1}}} - {Q^*}} \right\|_\infty } \le \frac{\gamma }{{1 - \gamma }}\sqrt {\frac{{2\varepsilon }}{\beta }|{\cal S}||{\cal A}|} \]

Combining the two cases, one gets the following bound:
\[{\left\| {{Q_{{\theta _1}}} - {Q^*}} \right\|_\infty } \le \frac{{\sqrt {\varepsilon |{\cal S}||{\cal A}|} }}{{1 - \gamma }} + \frac{\gamma }{1 - \gamma}\sqrt {\frac{{2\varepsilon |{\cal S}||{\cal A}|}}{\beta }} \]

\subsection{Bound on $Q_{{\theta_2}}$}

Next, using the reverse triangle inequality, we also have
\begin{align}
\frac{{\sqrt {\varepsilon |S||A|} }}{{1 - \gamma }} \ge& |{Q_{{\theta _1}}}(s,a) - {Q^*}(s,a) + {Q^*}(s,a) - {Q_{{\theta _2}}}(s,a)|\nonumber\\
\ge& ||{Q_{{\theta _1}}}(s,a) - {Q^*}(s,a)| - |{Q_{{\theta _2}}}(s,a) - {Q^*}(s,a)||.\label{eq:appendix:5}
\end{align}

Now, we consider the two cases:
\paragraph{Case 1:} Let us first consider the following case:
\[|{Q_{{\theta _1}}}(s,a) - {Q^*}(s,a)| < |{Q_{{\theta _2}}}(s,a) - {Q^*}(s,a)|.\]

In this case, the inequality in~\eqref{eq:appendix:5} leads to
\[|{Q_{\theta _2}}(s,a) - {Q^*}(s,a)| \le \frac{{\sqrt {\varepsilon |{\cal S}||{\cal A}|} }}{{1 - \gamma }} + |Q_{\theta _1}(s,a) - {Q^*}(s,a)|\]
Applying the bound on $|Q_{\theta _1}(s,a) - {Q^*}(s,a)|$, we can derive
\[|{Q_{{\theta _2}}}(s,a) - {Q^*}(s,a)| \le \frac{{2\sqrt {\varepsilon |{\cal S}||{\cal A}|} }}{{1 - \gamma }} + \frac{\gamma }{1 - \gamma}\sqrt {\frac{2\varepsilon |{\cal S}||{\cal A}|}{\beta }} \]
Now, taking the maximum on the left-hand side over $(s,a)\in {\cal S}\times {\cal A}$ leads to
\[{\left\| {{Q_{{\theta _2}}} - {Q^*}} \right\|_\infty } \le \frac{{2\sqrt {\varepsilon |{\cal S}||{\cal A}|} }}{1 - \gamma} + \frac{\gamma }{{1 - \gamma }}\sqrt {\frac{2\varepsilon |{\cal S}||{\cal A}|}{\beta }}. \]

\paragraph{Case 2:} Next, let us consider the case
\[|{Q_{{\theta _1}}}(s,a) - {Q^*}(s,a)| \ge |{Q_{{\theta _2}}}(s,a) - {Q^*}(s,a)|\]
In this case, the inequality in~\eqref{eq:appendix:5} results in
\[|{Q_{{\theta _2}}}(s,a) - {Q^*}(s,a)| \le |{Q_{{\theta _1}}}(s,a) - {Q^*}(s,a)| \le \frac{{\sqrt {\varepsilon |{\cal S}||{\cal A}|} }}{{1 - \gamma }} + \frac{\gamma }{{1 - \gamma }}\sqrt {\frac{{2\varepsilon |{\cal S}||{\cal A}|}}{\beta }} \]
where the last inequality is due to the bound on $|{Q_{{\theta _1}}}(s,a) - {Q^*}(s,a)|$. Taking the maximum on the left-hand side over $(s,a)\in {\cal S}\times {\cal A}$ leads to
\[{\left\| Q_{\theta _2} - Q^* \right\|_\infty } \le \frac{{\sqrt {\varepsilon |{\cal S}||{\cal A}|} }}{{1 - \gamma }} + \frac{\gamma }{{1 - \gamma }}\sqrt {\frac{{2\varepsilon |{\cal S}||{\cal A}|}}{\beta }} \]

Combining the two cases, we have the desired conclusion.

\section{Proof of~\cref{thm:SGT2-DQN:optimality}}

Let us assume that by minimizing the loss functions of SGT2-DQN, we can approximately minimize the following expected loss function
\begin{align*}
{L_1}({\theta _1}) =& {\mathbb E}_{(s,a) \sim U({\cal S} \times {\cal A}),s' \sim P( \cdot |s,a)}\left[ {\left( {r(s,a,s') + \gamma \max_{a \in {\cal A}}Q_{\theta _2}(s',a) - Q_{\theta _1}(s,a)} \right)}^2 \right.\\
&\left. { + \frac{\beta }{2}{{({Q_{{\theta _2}}}(s,a) - {Q_{{\theta _1}}}(s,a))}^2}} \right]\\
{L_2}({\theta _2}) =& {{\mathbb E}_{(s,a) \sim U({\cal S} \times {\cal A}),s' \sim P( \cdot |s,a)}}\left[ {{{\left( r(s,a,s') + \gamma {{\max }_{a \in {\cal A}}}{Q_{\theta_1}}(s',a) - {Q_{\theta_2}}(s,a) \right)}^2}} \right.\\
&\left. { + \frac{\beta }{2}{{(Q_{\theta_1}(s,a) - Q_{\theta _2}(s,a))}^2}} \right]
\end{align*}
where $U({\cal S} \times {\cal A})$ means the uniform distribution over the set ${\cal S}\times {\cal A}$. Suppose that the loss functions are minimized with
\[L_1({\theta _1}) \le \varepsilon ,\quad L_2(\theta _2) \le \varepsilon .\]
For convenience, let us define the Bellman operator
\[(TQ)(s,a): = R(s,a) + \gamma \sum\limits_{s' \in {\cal S}} P(s'|s,a) \max _{a \in {\cal A}}Q(s',a)\]

Moreover, note that using the law of iterated expectations, the expected loss functions can be expressed as
\begin{align*}
L_1(\theta _1) =& {\mathbb E}_{(s,a) \sim U({\cal S} \times {\cal A})}\left[ {{\mathbb E}_{s' \sim P( \cdot |s,a)}}\left[ {\left( r(s,a,s') + \gamma {{\max }_{a \in A}}Q_{\theta_2}(s',a) - Q_{\theta _1}(s,a) \right)}^2 \right] \right]\\
&+ {{\mathbb E}_{(s,a) \sim U({\cal S} \times {\cal A})}}\left[ {\frac{\beta }{2}{(Q_{\theta_2}(s,a) - Q_{\theta _1}(s,a))}^2} \right],
\end{align*}
and
\begin{align*}
{L_2}(\theta_2) =& {{\mathbb E}_{(s,a) \sim U({\cal S} \times {\cal A})}}\left[ {\mathbb E}_{s' \sim P( \cdot |s,a)}\left[ {\left( r(s,a,s') + \gamma {{\max }_{a \in A}}Q_{\theta _1}(s',a) - Q_{\theta_2}(s,a) \right)}^2 \right] \right]\\
&+ {\mathbb E}_{(s,a) \sim U({\cal S} \times {\cal A})}\left[ {\frac{\beta }{2}{(Q_{\theta _1}(s,a) - {Q_{{\theta _2}}}(s,a))}^2} \right]
\end{align*}

Next, applying Jensen's inequality to $L_1(\theta _1)$, we can prove that
\begin{align*}
L_1({\theta _1}) \ge& {\mathbb E}_{(s,a) \sim U({\cal S} \times {\cal A})}\left[ {{{\left( {\mathbb E}_{s' \sim P( \cdot |s,a)}\left[ {r(s,a,s') + \gamma \max_{a \in {\cal A}}Q_{\theta _2}(s',a) - Q_{\theta _1}(s,a)} \right] \right)}^2}} \right]\\
& + {{\mathbb E}_{(s,a) \sim U({\cal S} \times {\cal A})}}\left[ {\frac{\beta }{2}{{({Q_{\theta _2}}(s,a) - {Q_{\theta _1}}(s,a))}^2}} \right]\\
=& {\mathbb E}_{(s,a) \sim U({\cal S} \times {\cal A})}\left[ ((T Q_{\theta _2})(s,a) - {Q_{{\theta _1}}}(s,a))^2 \right]\\
&+ {\mathbb E}_{(s,a) \sim U({\cal S} \times {\cal A})}\left[ {\frac{\beta }{2}{{({Q_{{\theta _2}}}(s,a) - {Q_{{\theta _1}}}(s,a))}^2}} \right]\\
\ge& \frac{1}{|{\cal S}||{\cal A}|}{\left( (TQ_{\theta _2})(s,a) - Q_{\theta _1}(s,a) \right)^2}\\
& + \frac{1}{|{\cal S}||{\cal A}|}\frac{\beta }{2}{\left( Q_{\theta _2}(s,a) - Q_{\theta _1}(s,a) \right)^2}
\end{align*}
for all $(s,a)\in {\cal S}\times {\cal A}$, which implies
\[\frac{1}{{|{\cal S}||{\cal A}|}}{\left( {(TQ_{\theta _2})(s,a) - Q_{\theta _1}(s,a)} \right)^2} \le \varepsilon ,\quad \frac{1}{{|{\cal S}||{\cal A}|}}\frac{\beta }{2}{\left( {{Q_{\theta _2}}(s,a) - Q_{\theta _1}(s,a)} \right)^2} \le \varepsilon, \]
and equivalently
\[\left| {(T{Q_{\theta _2}})(s,a) - {Q_{{\theta _1}}}(s,a)} \right| \le \sqrt {\varepsilon |{\cal S}||{\cal A}|} ,\quad |Q_{\theta _2}(s,a) - Q_{\theta _1}(s,a)| \le \sqrt {\frac{{2\varepsilon |{\cal S}||{\cal A}|}}{\beta }} \]
Next, applying the same procedures as in the proof for AGT2-DQN, we arrive at
\[{\left\| Q_{\theta _1} - Q^* \right\|_\infty } \le \frac{\sqrt {\varepsilon |{\cal S}||{\cal A}|} }{1 - \gamma} + \frac{\gamma }{{1 - \gamma }}\sqrt \frac{{2\varepsilon |{\cal S}||{\cal A}|}}{\beta }. \]
By symmetry, one can obtain the same bound for ${\left\| {{Q_{{\theta _2}}} - {Q^*}} \right\|_\infty }$. This completes the proof.

\clearpage
\section{Additional comparison results among DQN, AGT2-DQN, and SGT2-DQN}

\subsection{Cartpole environment}
\begin{figure}[h!]
\centering
\includegraphics[width=6cm,height=5cm]{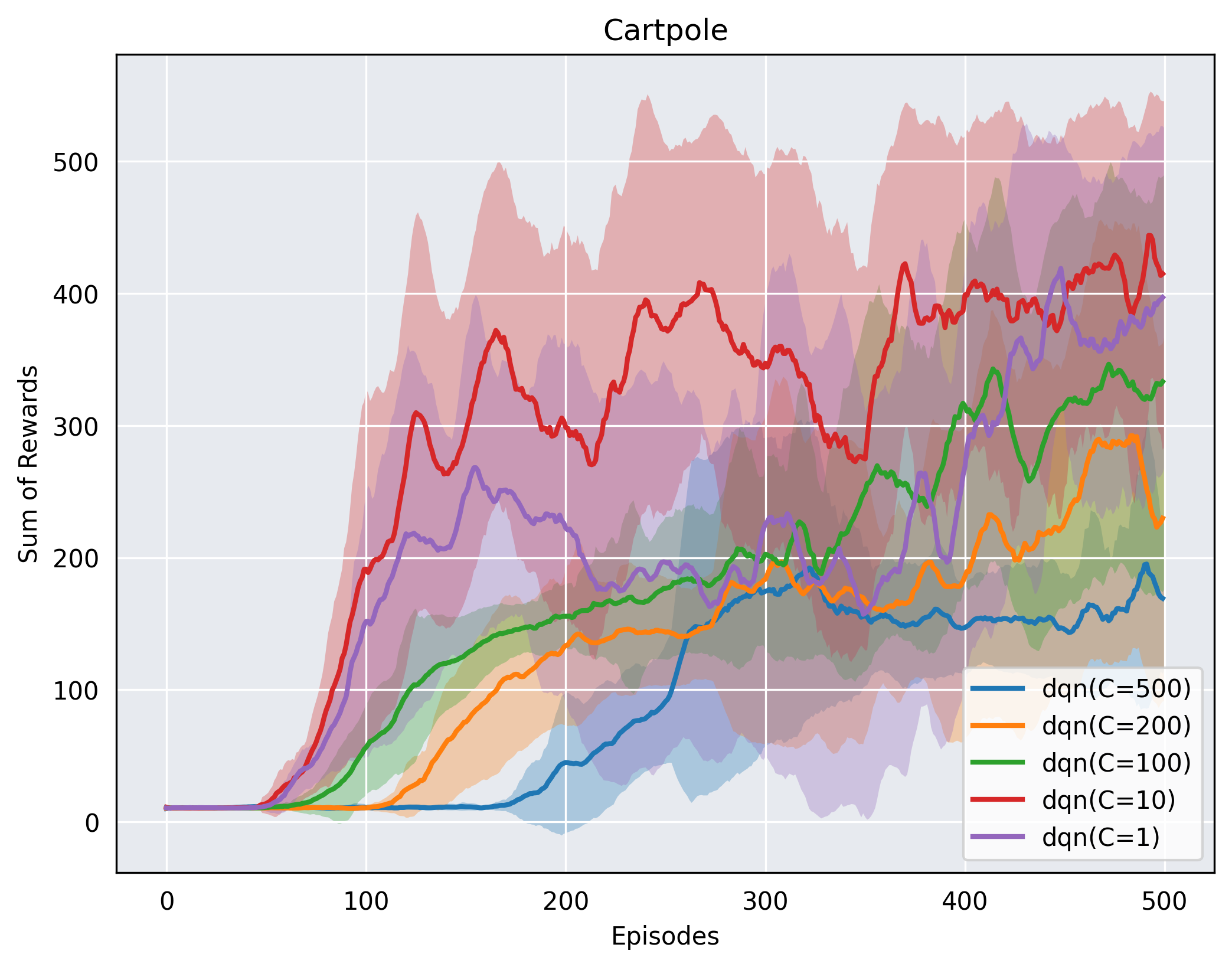}
\caption{Comparison of reward curves of DQN with different $C$ in Cartpole environment. The results show that $C=10$ gives the best learning performance among the other options.}\label{fig:appendix:cartpole:DQN-C}
\end{figure}

\cref{fig:appendix:cartpole:DQN-C} shows the learning curves of DQN for different $C\in \{1,10,100,200,500\}$ in Cartpole environment. The results show that $C=10$ gives the best learning performance among the other choices of $C$.

\begin{figure}[h!]
\centering
\subfigure[AGT2-DQN]{\includegraphics[width=6cm,height=5cm]{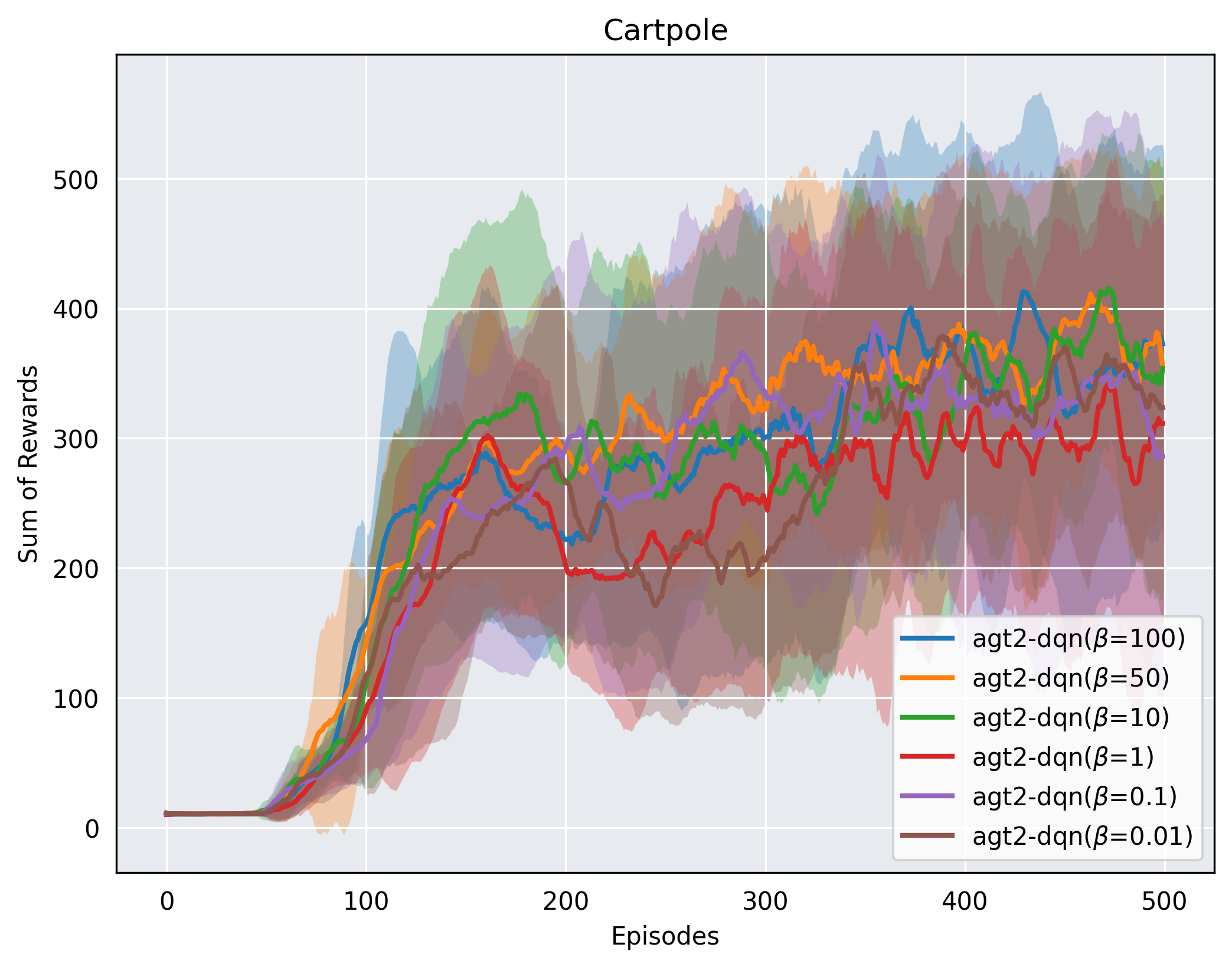}}
\subfigure[SGT2-DQN]{\includegraphics[width=6cm,height=5cm]{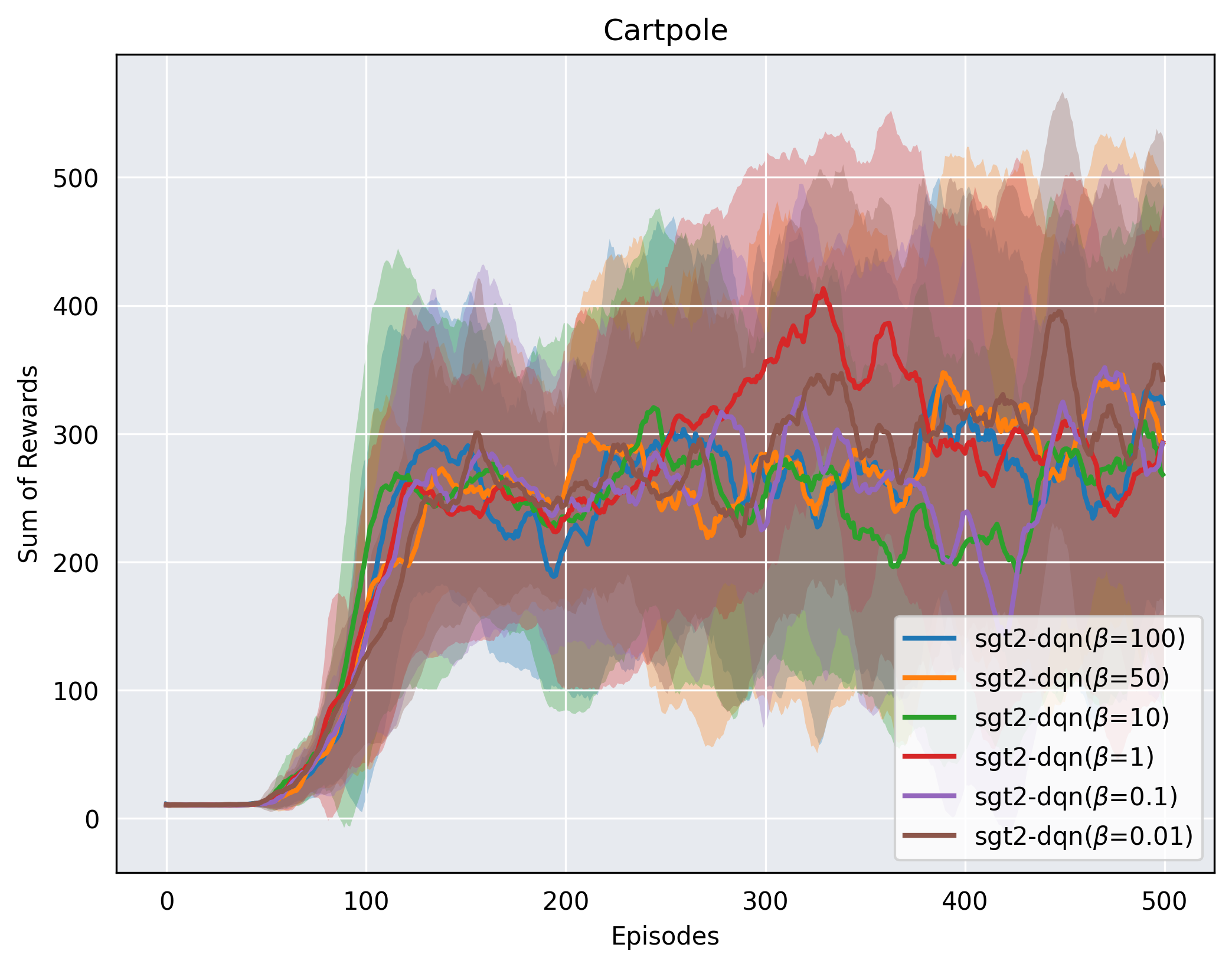}}
\caption{Comparison of reward curves of AGT2-DQN and SGT2-DQN with different $\beta$ in Cartpole environment. }\label{fig:appendix:cartpole:AGT2-DQN-beta}
\end{figure}

\cref{fig:appendix:cartpole:AGT2-DQN-beta} illustrates the learning curves of AGT2-DQN and SGT2-DQN for different $\beta\in \{0.01,0.1,1,10,50,100\}$in Cartpole environment. The results demonstrate that their learning efficiency is comparable to DQN but without requiring extensive tuning, as they appear to be less sensitive to $\beta$.

\begin{figure}[h!]
\centering
\includegraphics[width=6cm,height=5cm]{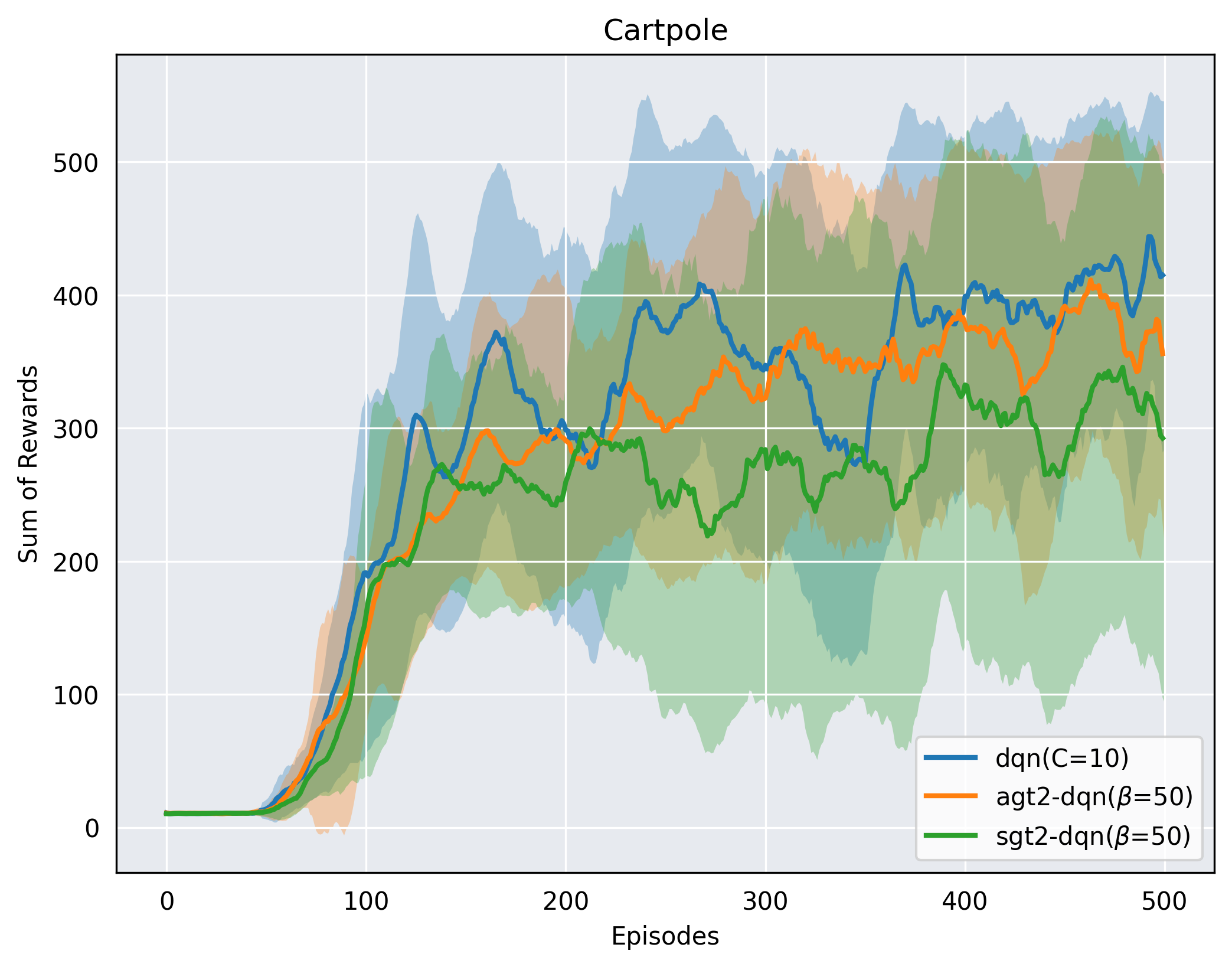}
\caption{Comparison of reward curves of DQN with $C=10$, AGT2-DQN with $\beta = 50$ and SGT2-DQN with $\beta = 50$ in Cartpole environment. }\label{fig:appendix:cartpole:DQN-AGT2-SGT2}
\end{figure}

\cref{fig:appendix:cartpole:DQN-AGT2-SGT2} illustrates the learning curves of DQN with $C=10$, AGT2-DQN and SGT2-DQN with $\beta = 50$ in Cartpole environment. The results show that while DQN exhibits slightly better learning efficiency than the proposed methods in this environment, while as mentioned before the latter require less efforts for hyperparameter tuning.

\clearpage
\subsection{Acrobot environment}

\begin{figure}[h!]
\centering
\includegraphics[width=6cm,height=5cm]{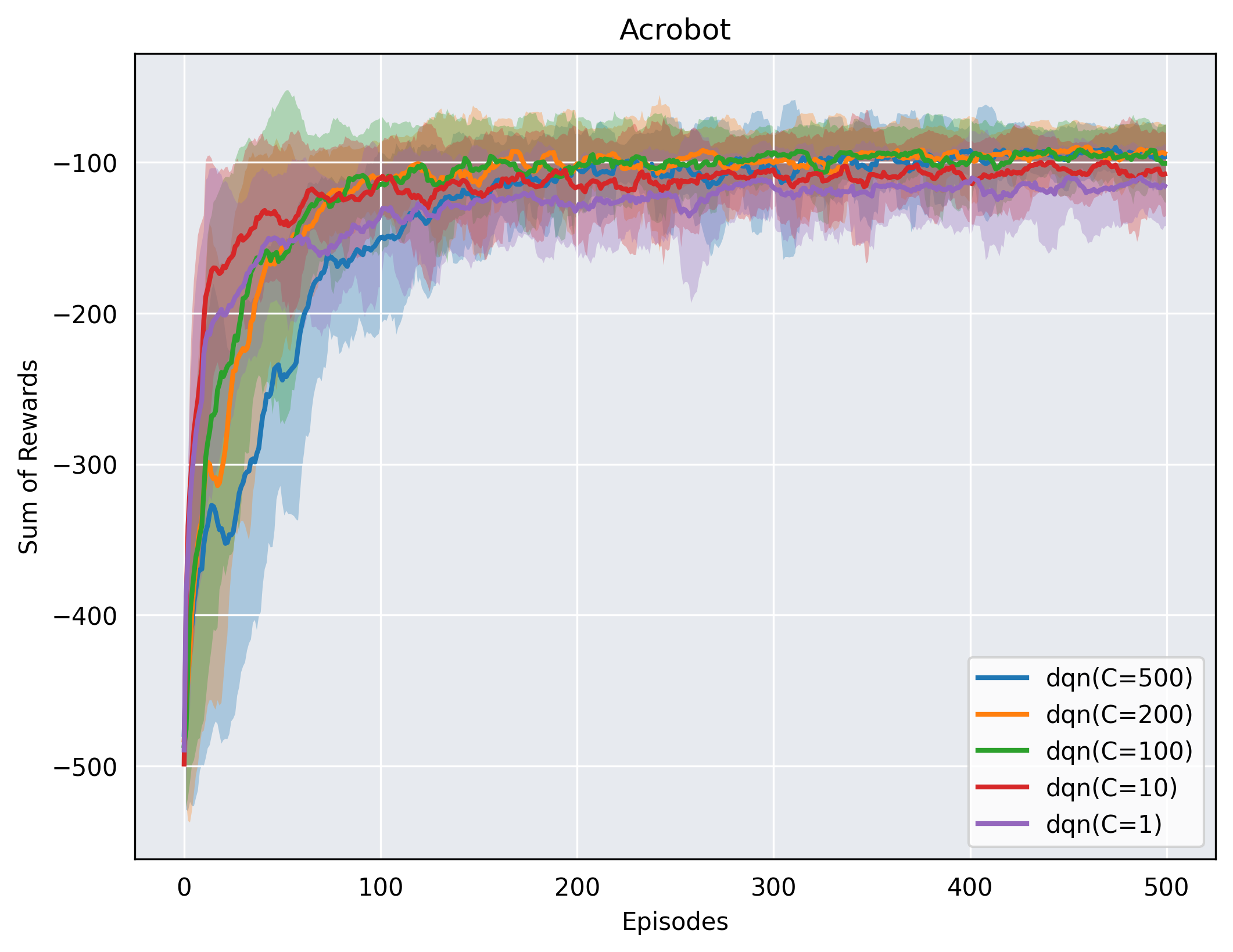}
\caption{Comparison of reward curves of DQN with different $C$ in Acrobot environment.}\label{fig:appendix:acrobot:DQN-C}
\end{figure}

\cref{fig:appendix:acrobot:DQN-C} shows the learning curves of DQN for different $C\in \{1,10,100,200,500\}$ in Acrobot environment.

\begin{figure}[h!]
\centering
\subfigure[AGT2-DQN]{\includegraphics[width=6cm,height=5cm]{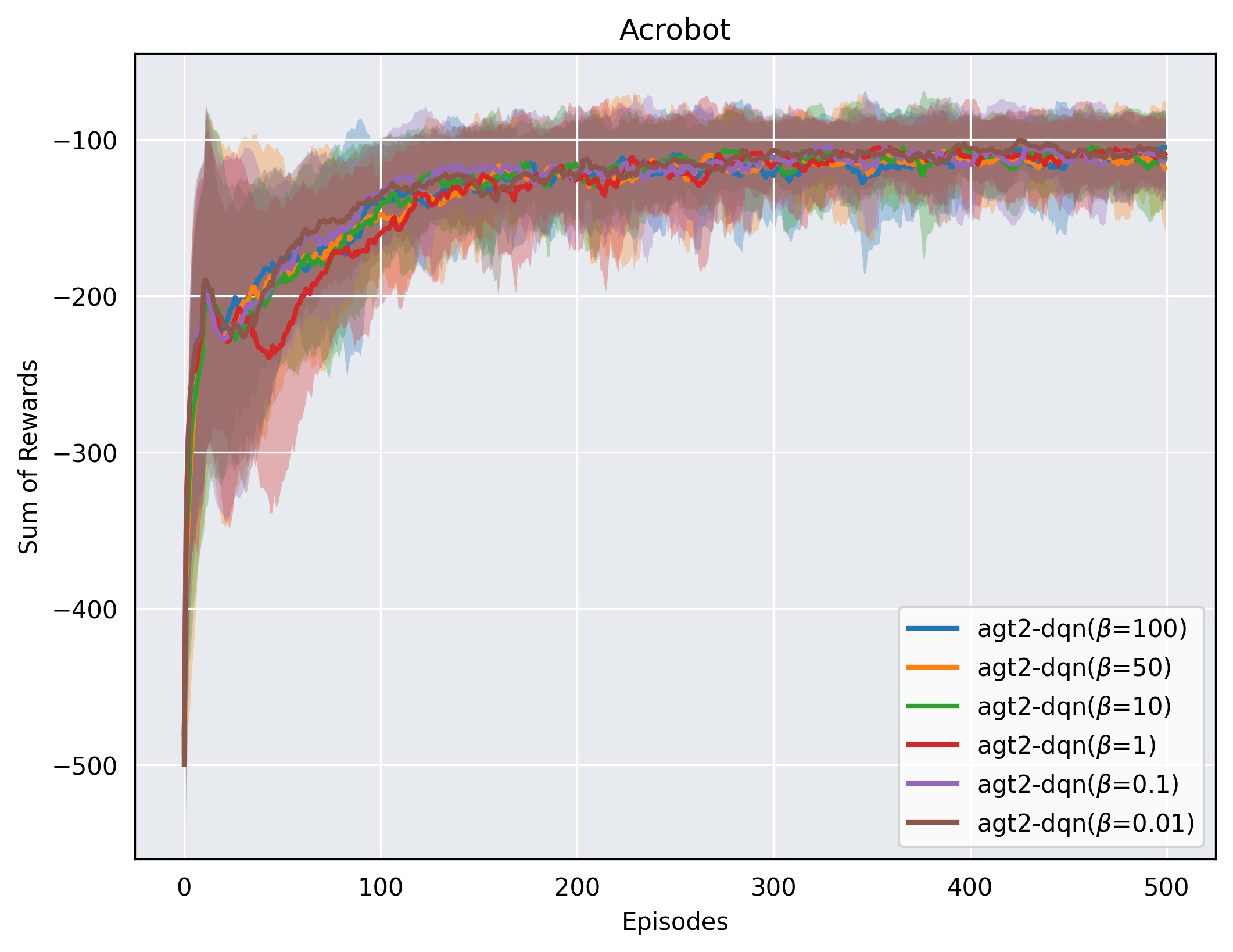}}
\subfigure[SGT2-DQN]{\includegraphics[width=6cm,height=5cm]{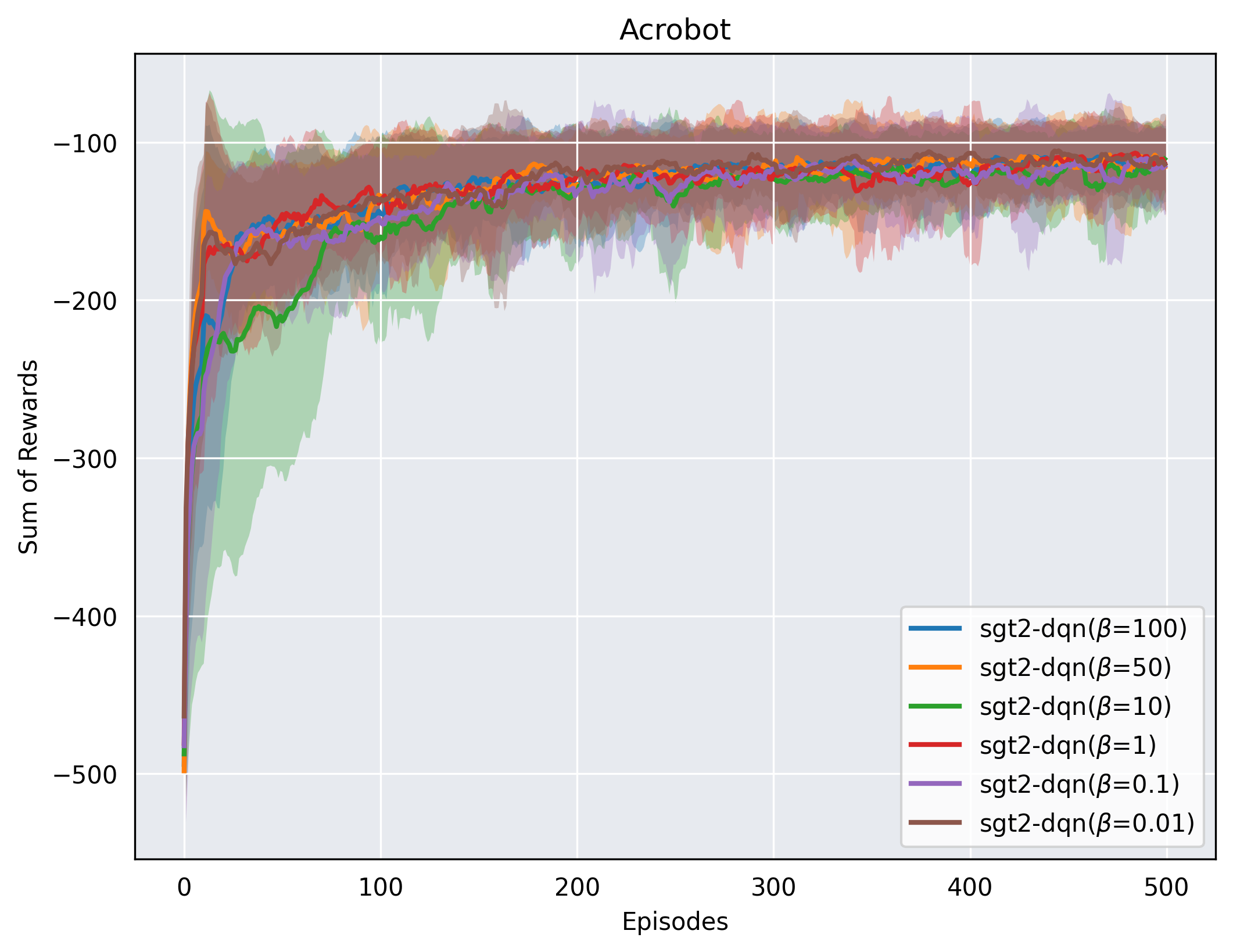}}
\caption{Comparison of reward curves of AGT2-DQN and SGT2-DQN with different $\beta$ in Acrobot environment. }\label{fig:appendix:acrobot:AGT2-DQN-beta}
\end{figure}

\cref{fig:appendix:acrobot:AGT2-DQN-beta} illustrates the learning curves of AGT2-DQN and SGT2-DQN for different $\beta\in \{0.01,0.1,1,10,50,100\}$. The results demonstrate that their learning efficiency is comparable to DQN and slightly less sensitive to the choice of the hyperparameter $\beta$.

\begin{figure}[h!]
\centering
\includegraphics[width=6cm,height=5cm]{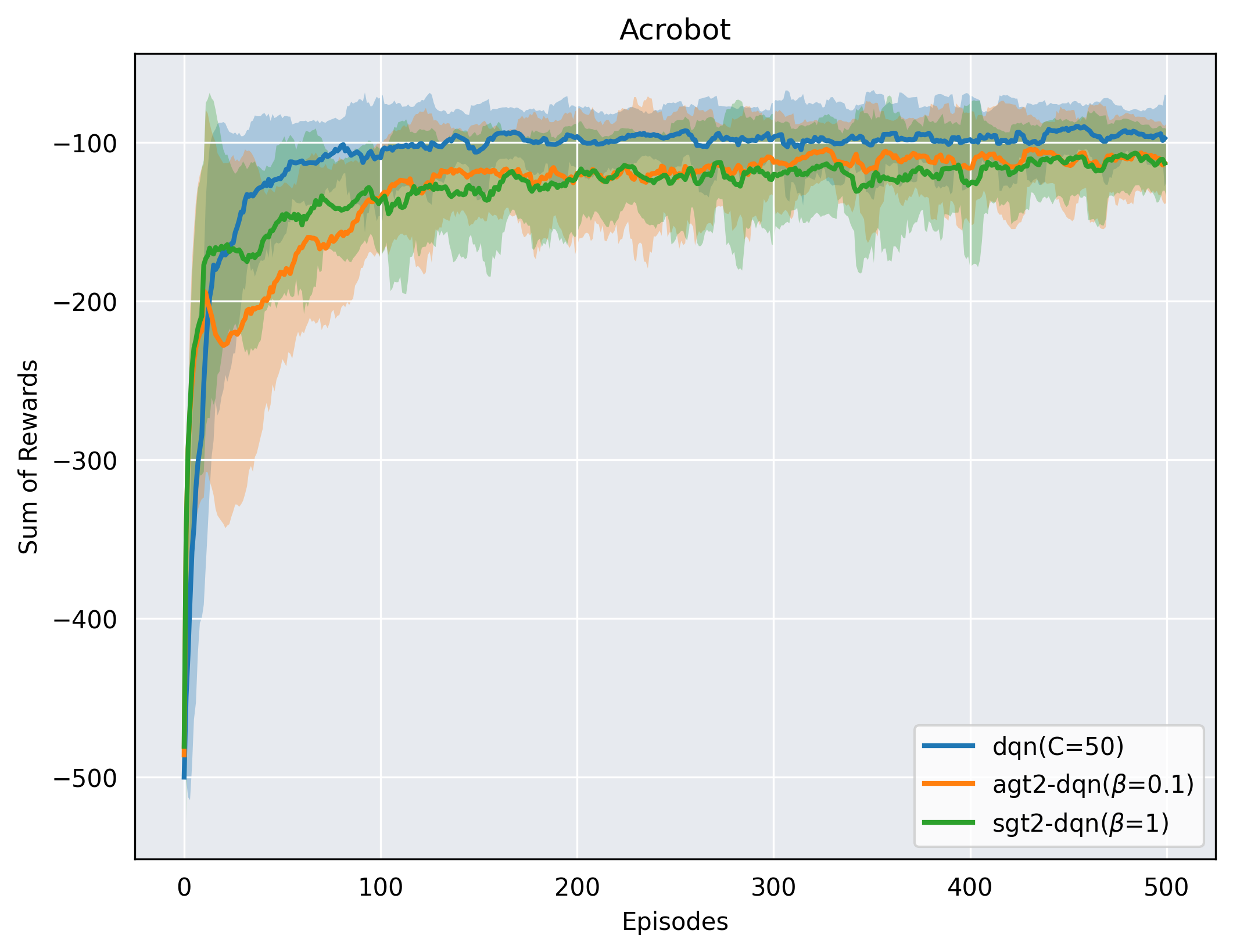}
\caption{Comparison of reward curves of DQN with $C=50$, AGT2-DQN with $\beta = 0.1$ and SGT2-DQN with $\beta = 1$ in Acrobot environment. }\label{fig:appendix:acrobot:DQN-AGT2-SGT2}
\end{figure}

\cref{fig:appendix:acrobot:DQN-AGT2-SGT2} illustrates the learning curves of DQN with $C=50$, AGT2-DQN with $\beta = 0.1$, and SGT2-DQN with $\beta = 1$, where each hyperparameter has been selected to approximately yield the best results. The results show that DQN exhibits slightly better learning efficiency than the proposed methods.

\clearpage
\subsection{Pendulum environment}

\begin{figure}[h!]
\centering
\includegraphics[width=6cm,height=5cm]{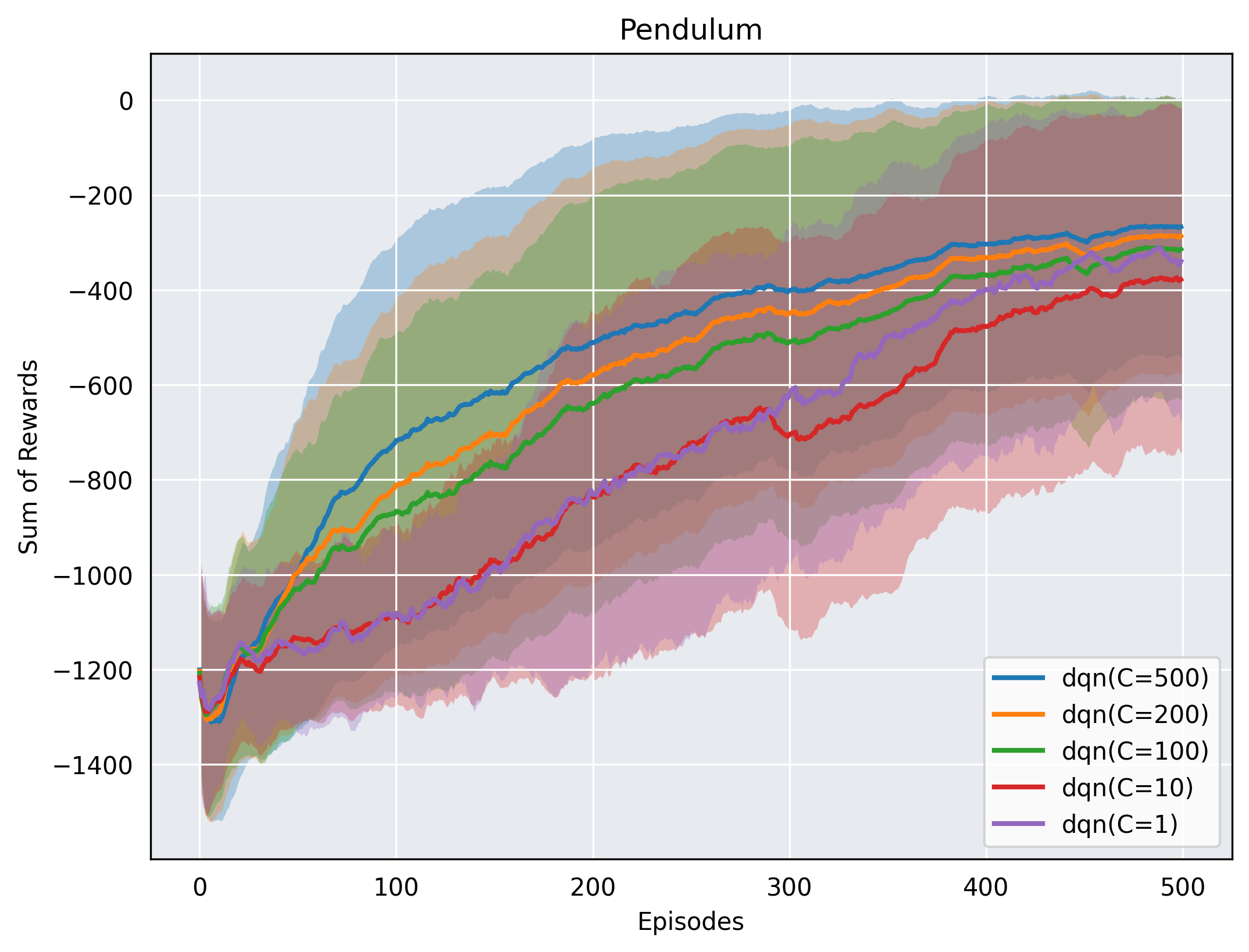}
\caption{Comparison of reward curves of DQN with different $C$ in Pendulum environment. }\label{fig:appendix:pendulum:DQN-C}
\end{figure}

\cref{fig:appendix:pendulum:DQN-C} shows the learning curves of DQN for different $C\in \{1,10,100,200,500\}$ in Pendulum environment, which are sensitive to the choice of the hyperparameter $C$.

\begin{figure}[h!]
\centering
\subfigure[AGT2-DQN]{\includegraphics[width=6cm,height=5cm]{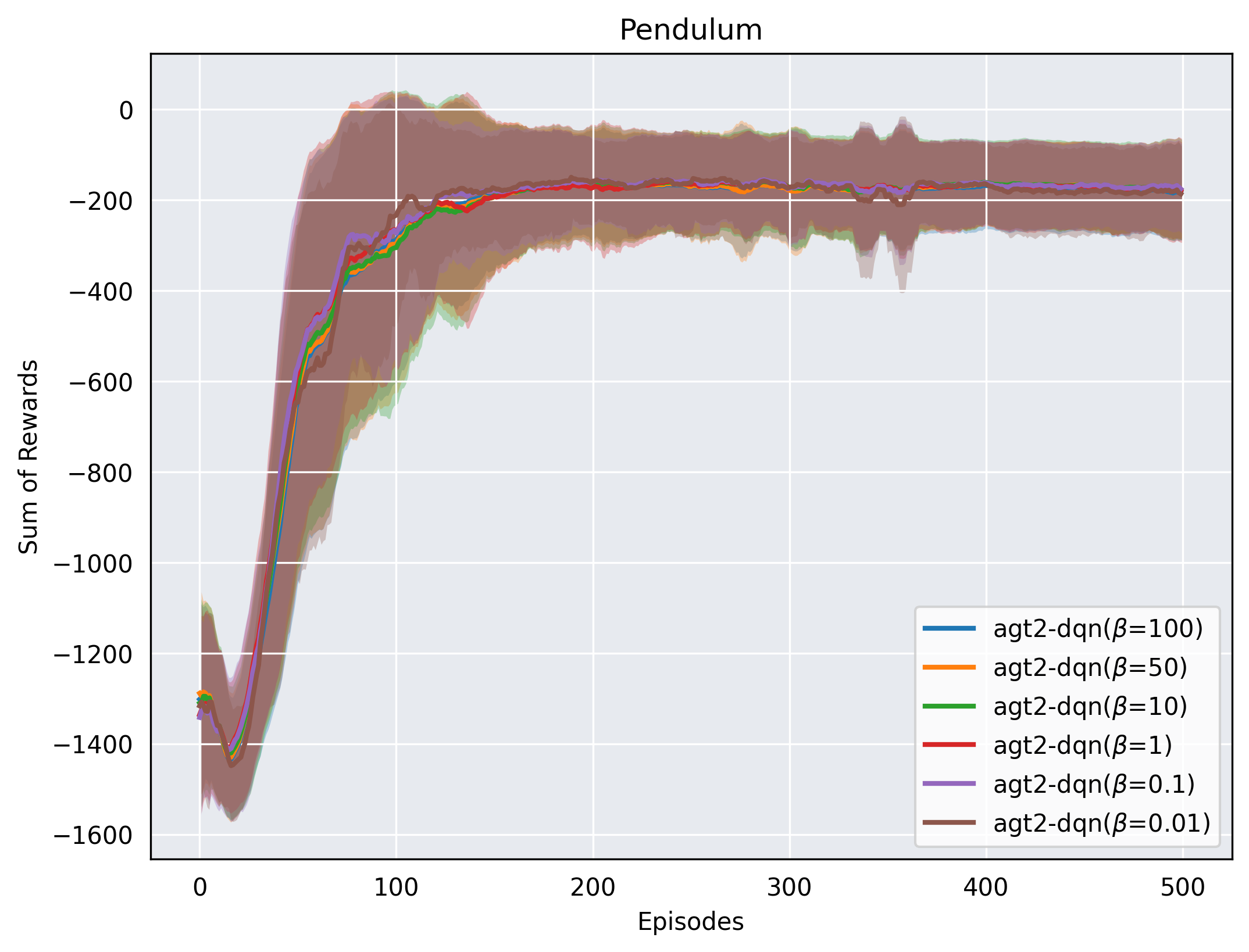}}
\subfigure[SGT2-DQN]{\includegraphics[width=6cm,height=5cm]{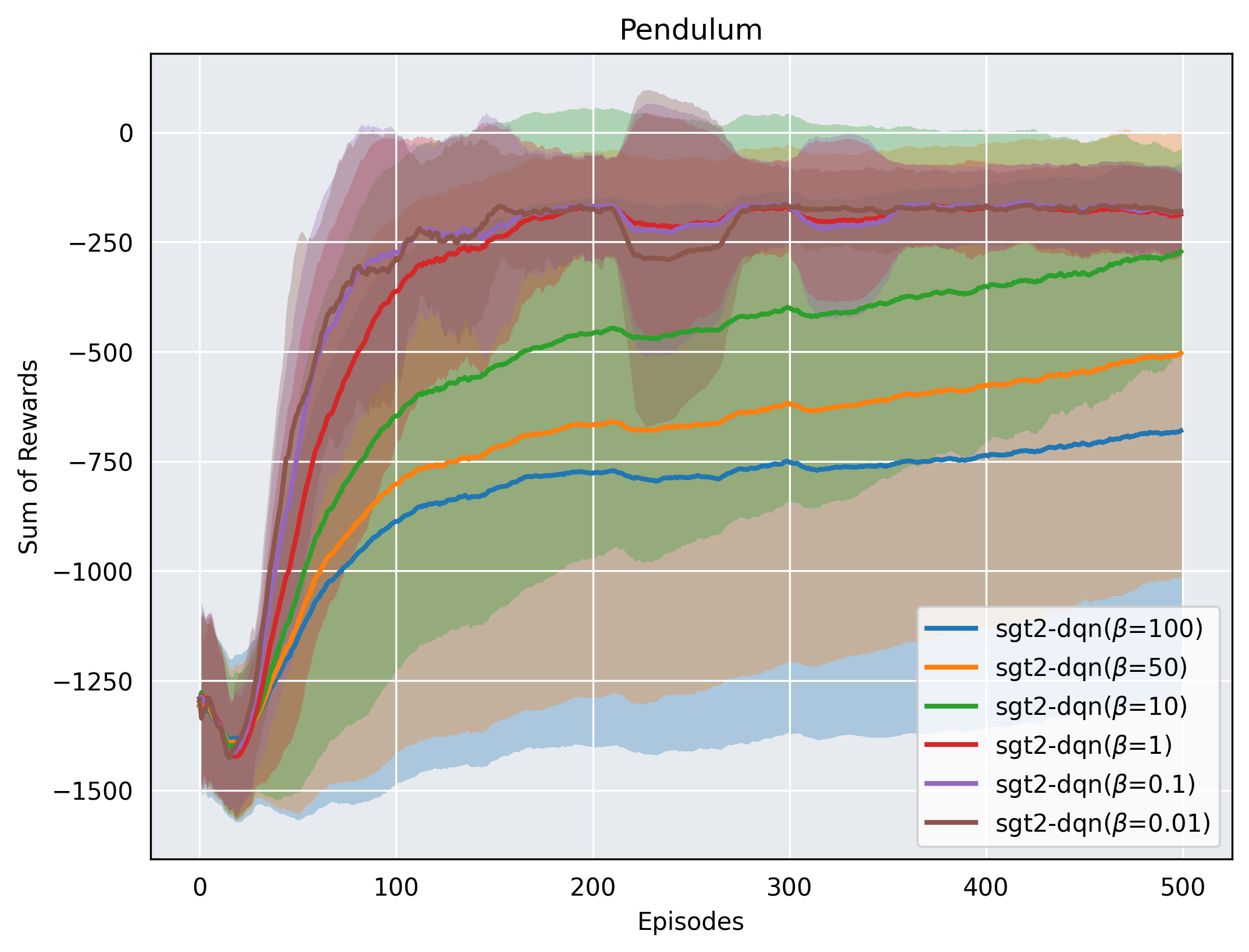}}
\caption{Comparison of reward curves of AGT2-DQN and SGT2-DQN with different $\beta$ in Pendulum environment. }\label{fig:appendix:pendulum:AGT2-DQN-beta}
\end{figure}

\cref{fig:appendix:pendulum:AGT2-DQN-beta} illustrates the learning curves of AGT2-DQN and SGT2-DQN for different $\beta\in \{0.01,0.1,1,10,50,100\}$. The results demonstrate that their learning efficiency is comparable to DQN. Moreover, compared to DQN, AGT2-DQN is less sensitive to the choice of the hyperparameter $\beta$, while SGT2-DQN is more  sensitive to the choice of the hyperparameter.

\begin{figure}[h!]
\centering
\includegraphics[width=6cm,height=5cm]{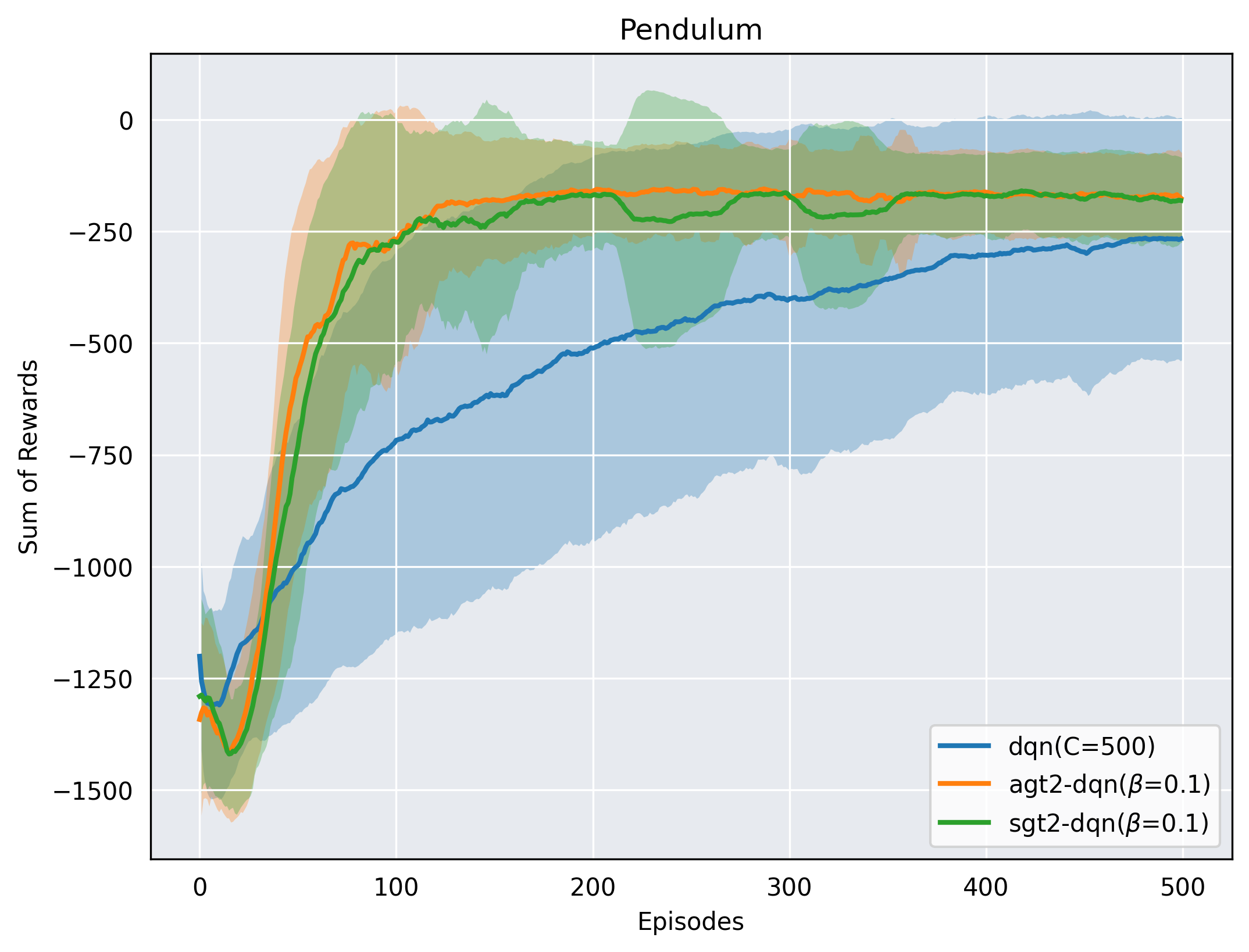}
\caption{Comparison of reward curves of DQN with $C=500$, AGT2-DQN with $\beta = 0.1$ and SGT2-DQN with $\beta = 0.1$ in Pendulum environment. }\label{fig:appendix:pendulum:DQN-AGT2-SGT2}
\end{figure}

\cref{fig:appendix:pendulum:DQN-AGT2-SGT2} illustrates the learning curves of DQN with $C=500$, AGT2-DQN with $\beta = 0.1$, and SGT2-DQN with $\beta = 0.1$, where each hyperparameter has been selected to approximately yield the best results. The results show that in this case, AGT2-DQN and SGT2-DQN exhibit better learning efficiency than DQN.

\clearpage
\subsection{Mountaincar environment}

\begin{figure}[h!]
\centering
\includegraphics[width=6cm,height=5cm]{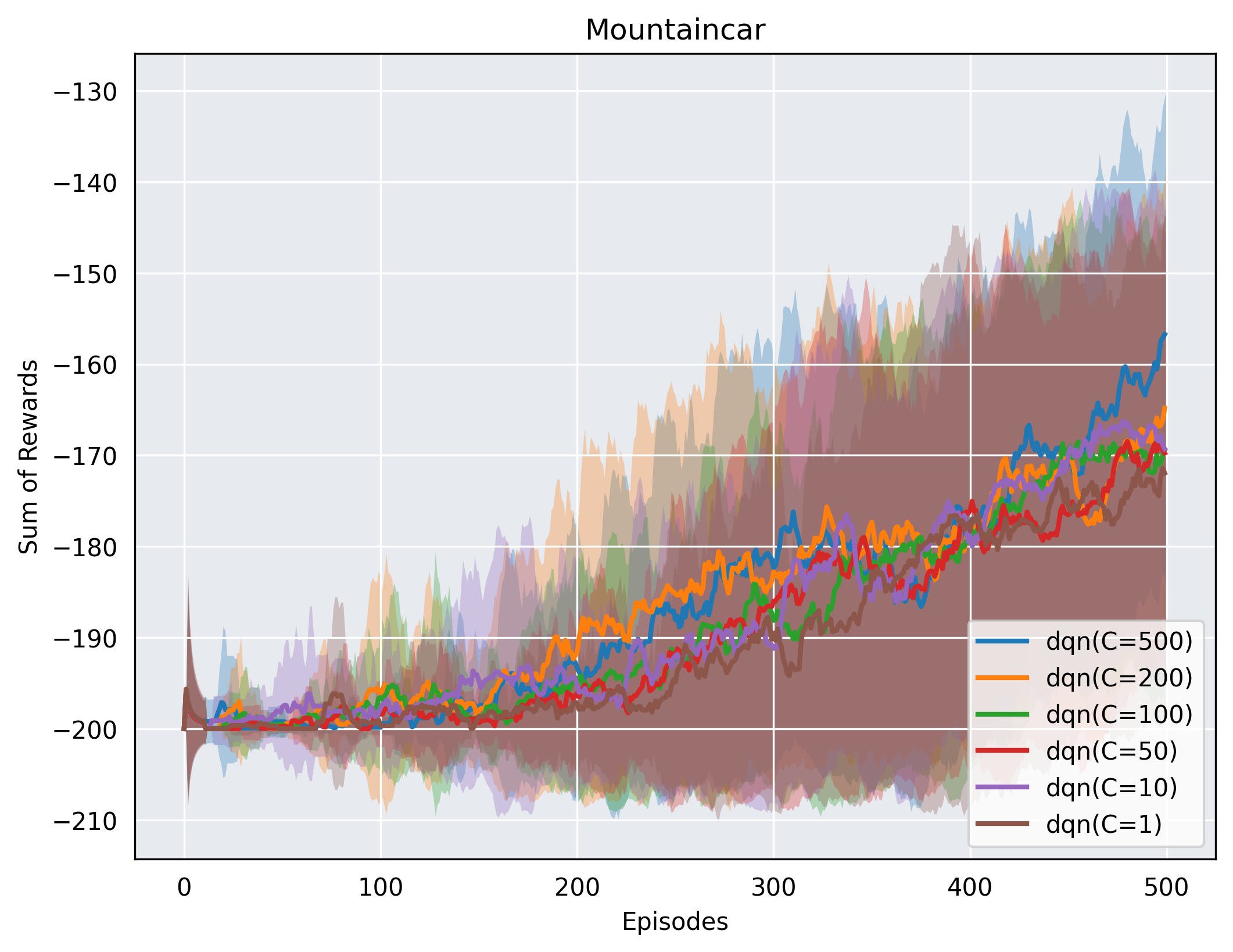}
\caption{Comparison of reward curves of DQN with different $C$ in Mountaincar environment.}\label{fig:appendix:mountaincar:DQN-C}
\end{figure}

\cref{fig:appendix:mountaincar:DQN-C} shows the learning curves of DQN for different $C\in \{1,10,100,200,500\}$ in Pendulum environment.

\begin{figure}[h!]
\centering
\subfigure[AGT2-DQN]{\includegraphics[width=6cm,height=5cm]{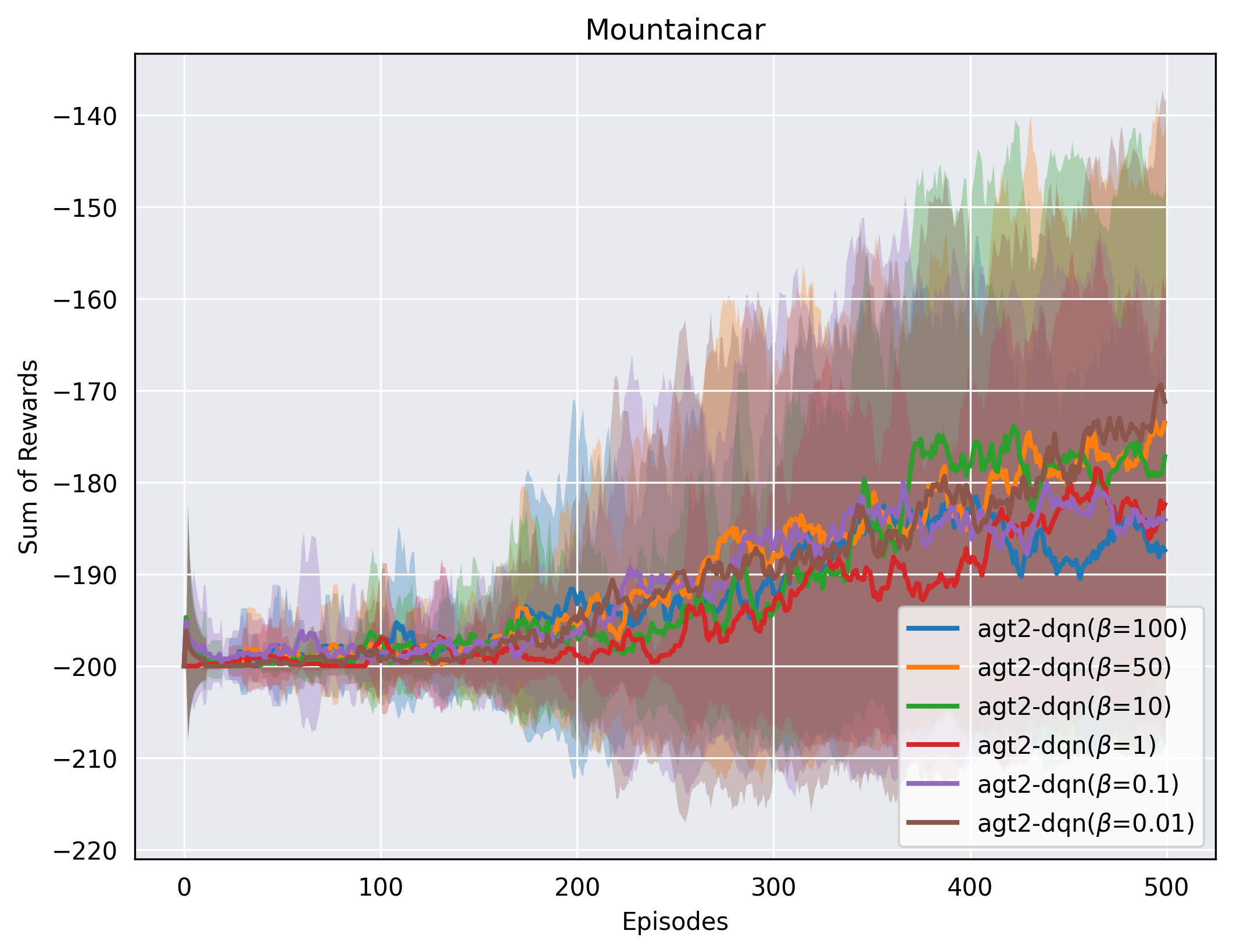}}
\subfigure[SGT2-DQN]{\includegraphics[width=6cm,height=5cm]{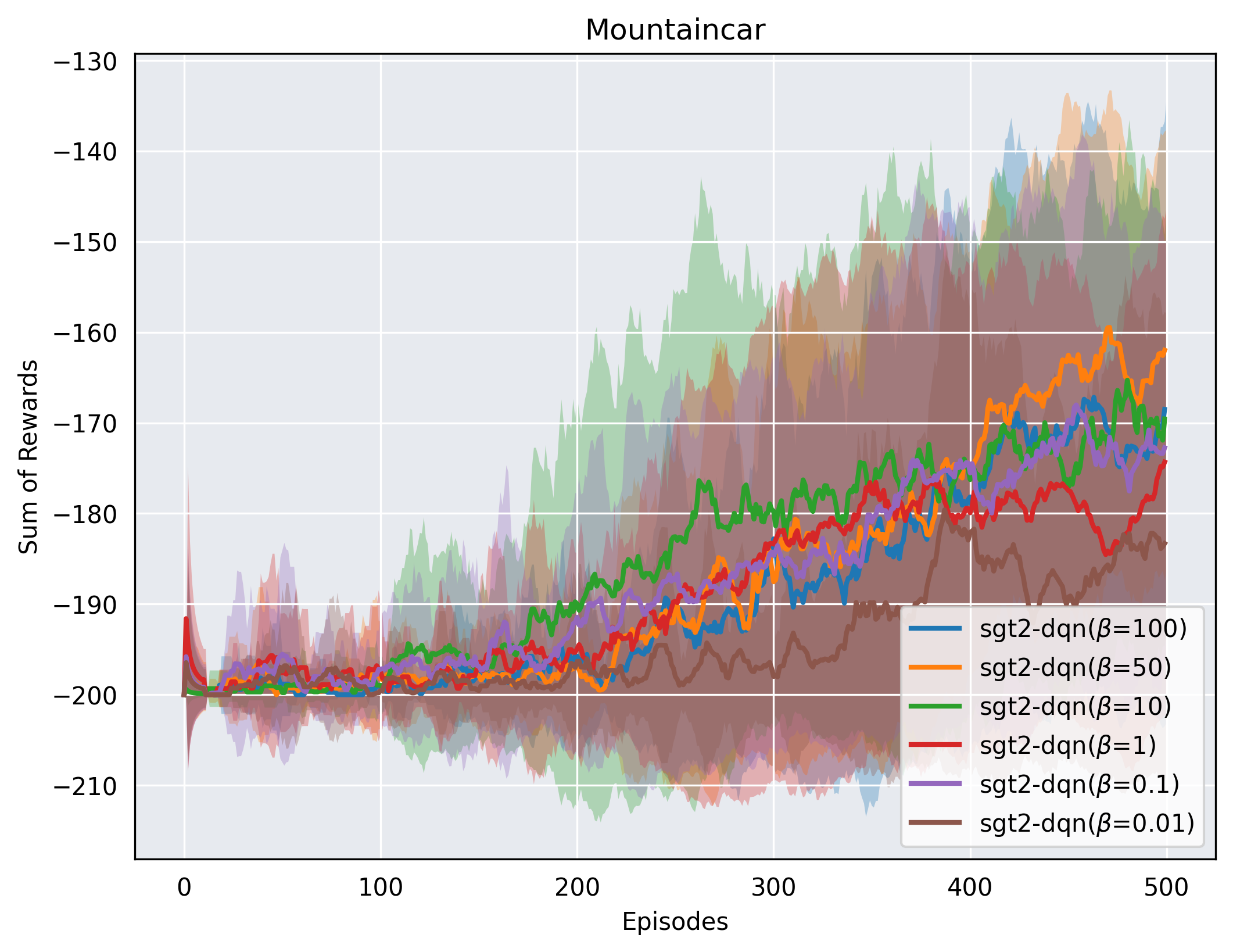}}
\caption{Comparison of reward curves of AGT2-DQN and SGT2-DQN with different $\beta$ in Mountaincar environment. }\label{fig:appendix:mountaincar:AGT2-DQN-beta}
\end{figure}

\cref{fig:appendix:mountaincar:AGT2-DQN-beta} illustrates the learning curves of AGT2-DQN and SGT2-DQN for different $\beta\in \{0.01,0.1,1,10,50,100\}$.

\begin{figure}[h!]
\centering
\includegraphics[width=6cm,height=5cm]{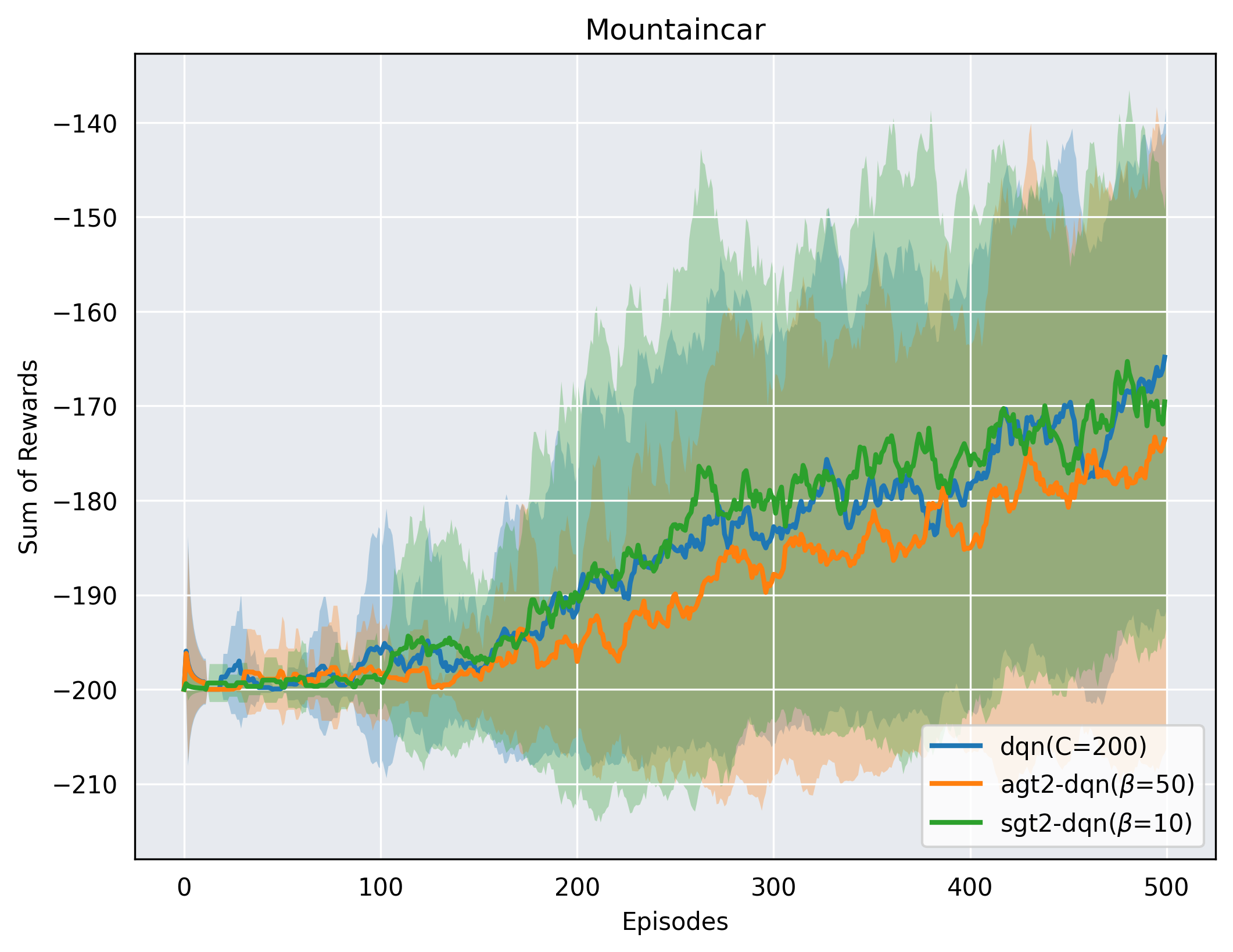}
\caption{Comparison of reward curves of DQN with $C=200$, AGT2-DQN with $\beta = 50$ and SGT2-DQN with $\beta = 10$ in Mountaincar environment. }\label{fig:appendix:mountaincar:DQN-AGT2-SGT2}
\end{figure}

\cref{fig:appendix:mountaincar:DQN-AGT2-SGT2} illustrates the learning curves of DQN with $C=200$, AGT2-DQN with $\beta = 50$, and SGT2-DQN with $\beta = 10$, where each hyperparameter has been selected to approximately yield the best results. 

Overall, the results reveal that DQN, AGT2-DQN, and SGT2-DQN provide similar learning efficiencies and sensitivity to the hyperparameters.

\clearpage
\subsection{Lunarlander environment}

\begin{figure}[h!]
\centering
\includegraphics[width=6cm,height=5cm]{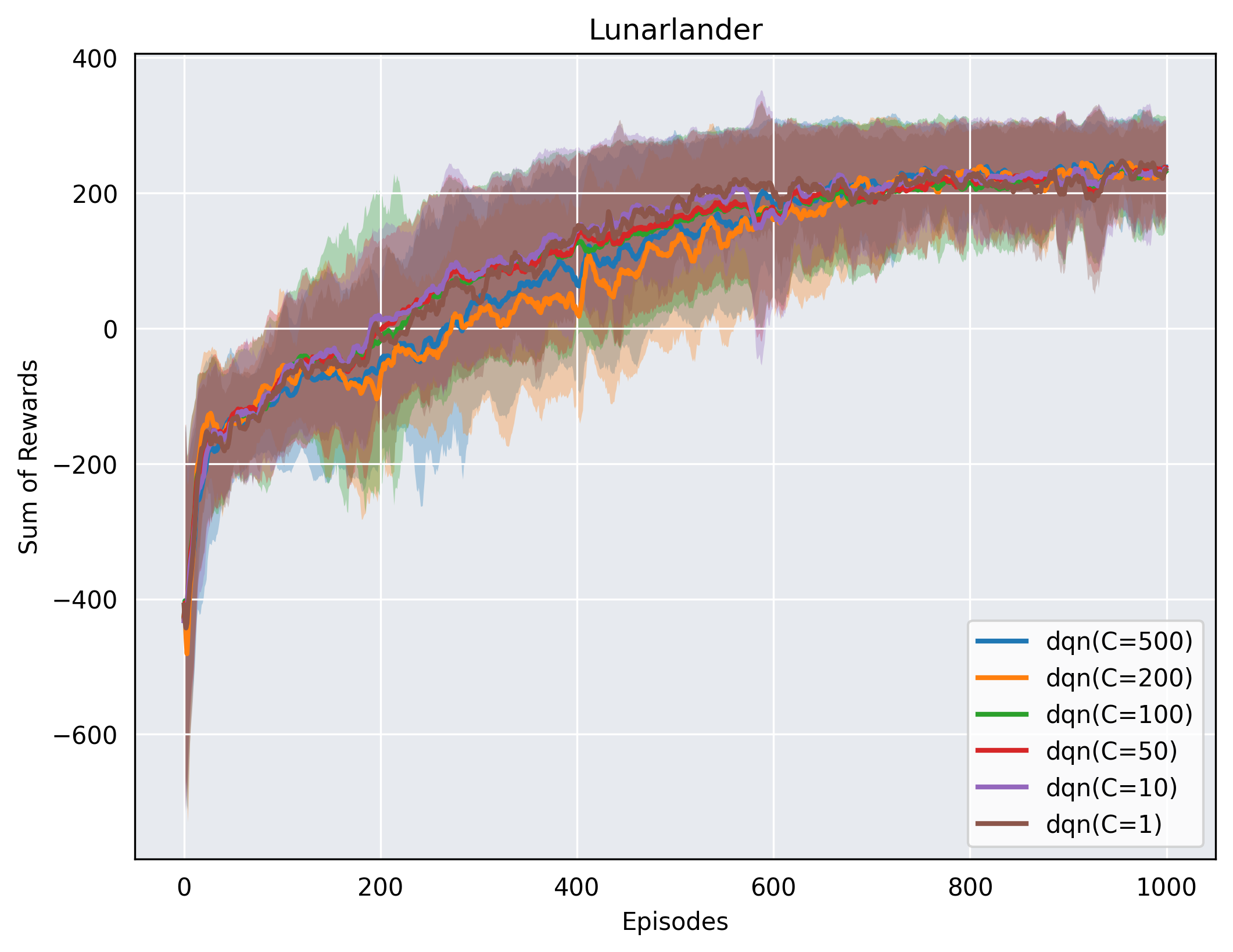}
\caption{Comparison of reward curves of DQN with different $C$ in Lunarlander environment.}\label{fig:appendix:lunarlander:DQN-C}
\end{figure}

\cref{fig:appendix:lunarlander:DQN-C} shows the learning curves of DQN for different $C\in \{1,10,100,200,500\}$ in Lunarlander environment.

\begin{figure}[h!]
\centering
\subfigure[AGT2-DQN]{\includegraphics[width=6cm,height=5cm]{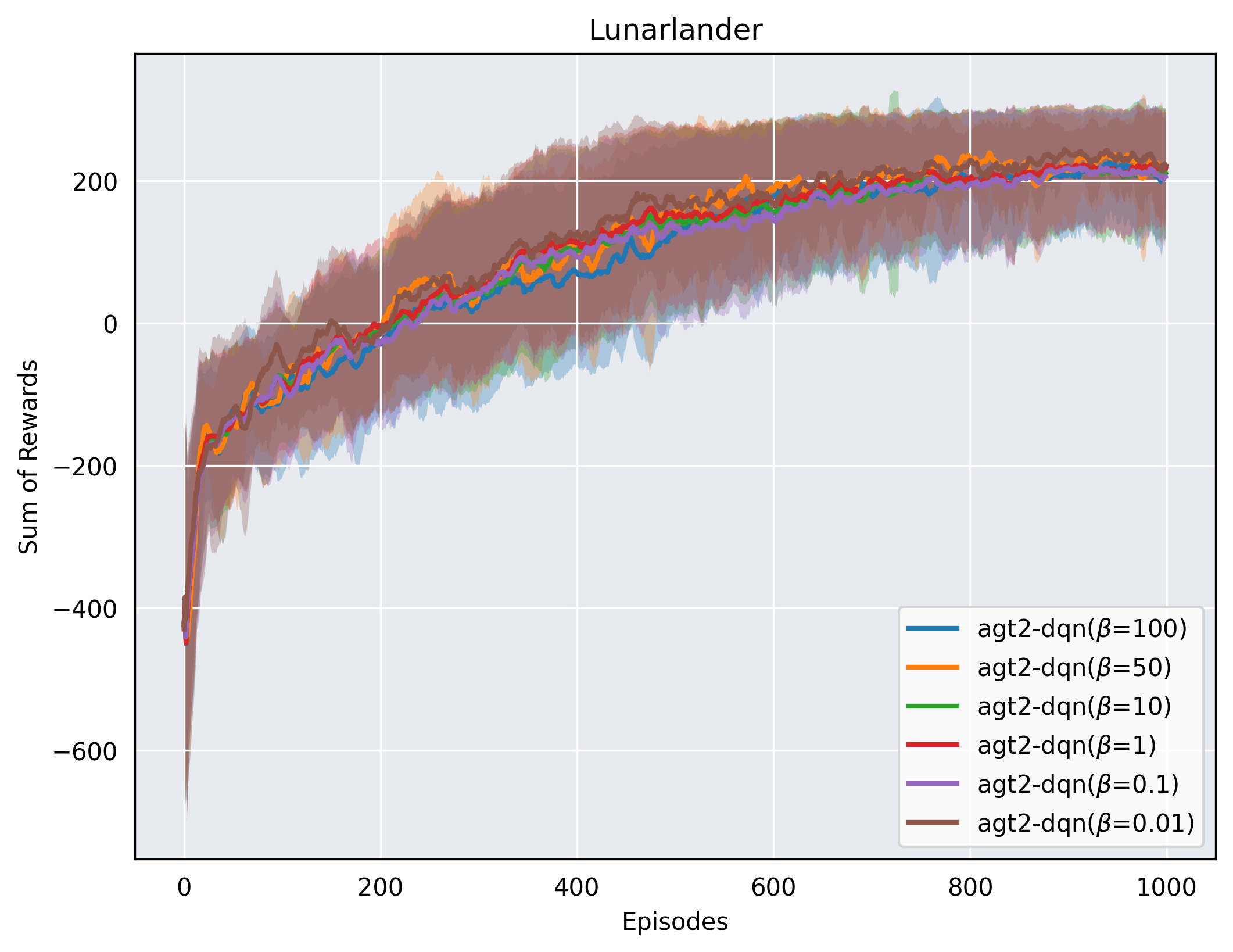}}
\subfigure[SGT2-DQN]{\includegraphics[width=6cm,height=5cm]{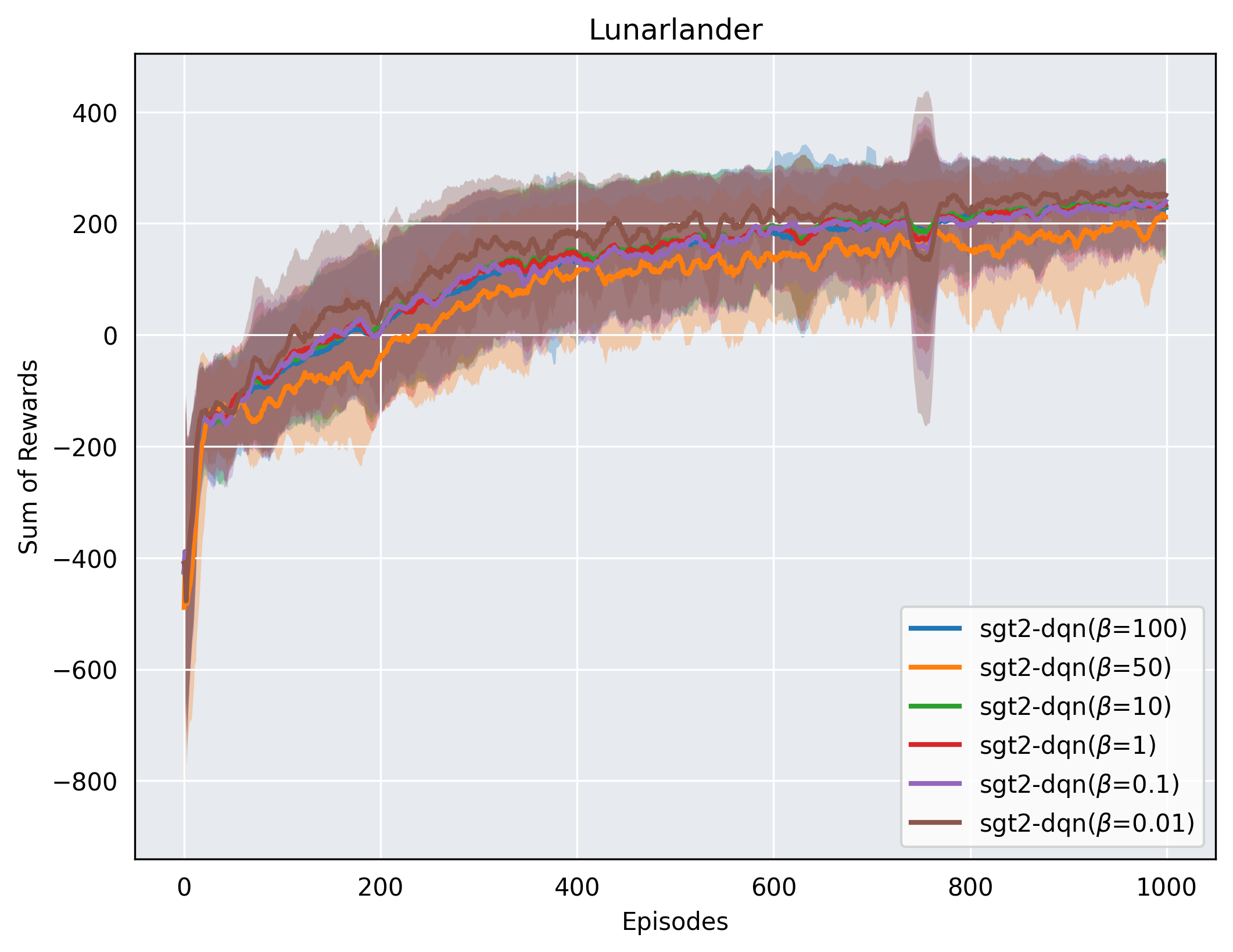}}
\caption{Comparison of reward curves of AGT2-DQN and SGT2-DQN with different $\beta$ in Lunarlander environment. }\label{fig:appendix:lunarlander:AGT2-DQN-beta}
\end{figure}

\cref{fig:appendix:lunarlander:AGT2-DQN-beta} illustrates the learning curves of AGT2-DQN and SGT2-DQN for different $\beta\in \{0.01,0.1,1,10,50,100\}$.

\begin{figure}[h!]
\centering
\includegraphics[width=6cm,height=5cm]{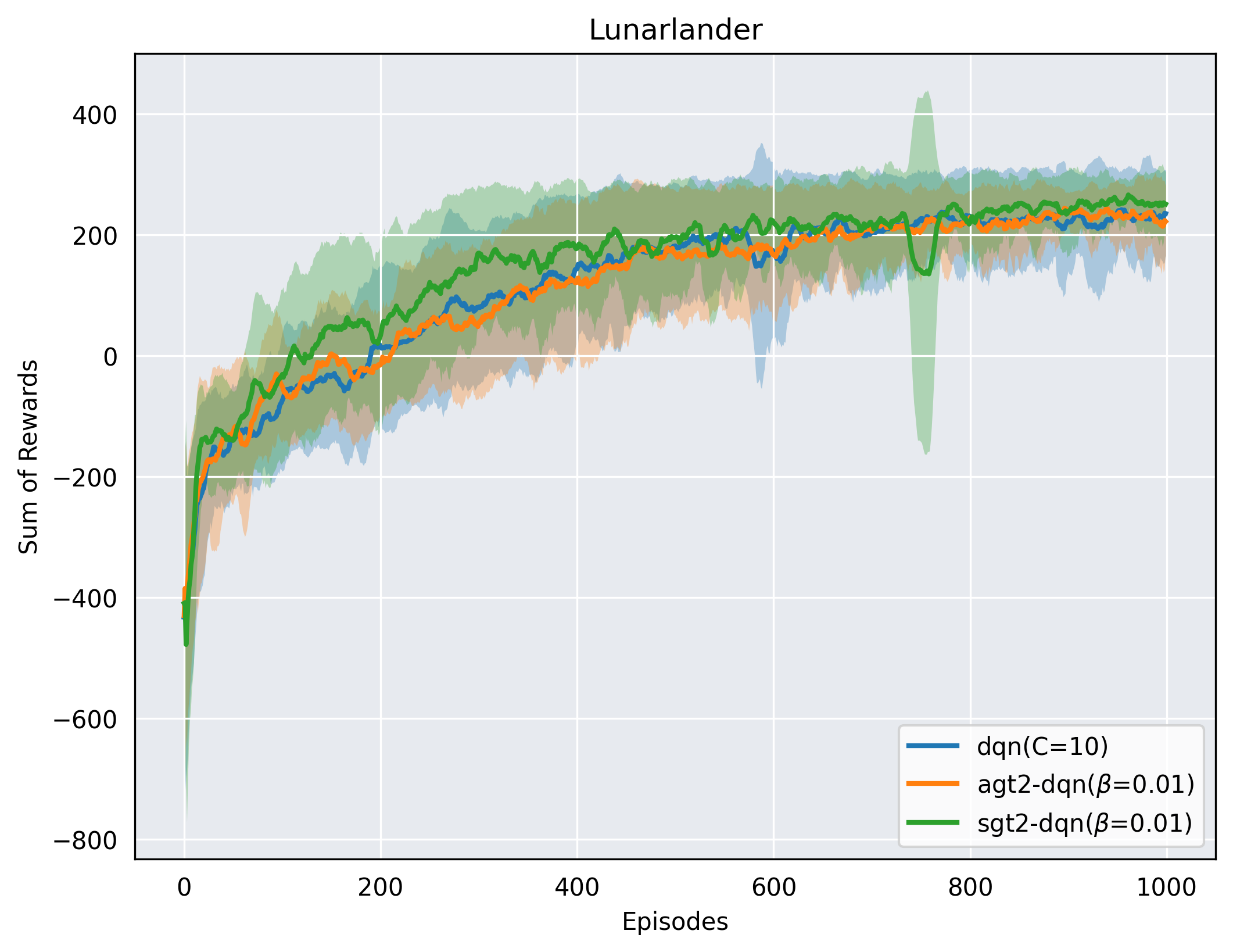}
\caption{Comparison of reward curves of DQN with $C=10$, AGT2-DQN with $\beta = 0.01$ and SGT2-DQN with $\beta = 0.01$in Lunarlander environment. }\label{fig:appendix:lunarlander:DQN-AGT2-SGT2}
\end{figure}

\cref{fig:appendix:lunarlander:DQN-AGT2-SGT2} illustrates the learning curves of DQN with $C=10$, AGT2-DQN with $\beta = 0.01$, and SGT2-DQN with $\beta = 0.01$, where each hyperparameter has been selected to approximately yield the best results. 

Overall, the results indicate that DQN, AGT2-DQN, and SGT2-DQN exhibit similar learning efficiencies and sensitivity to hyperparameters.

\end{document}